\documentclass[]{elsarticle}
\makeatletter
\def\ps@pprintTitle{%
 \let\@oddhead\@empty
 \let\@evenhead\@empty
 \def\@oddfoot{\hfill \textit{Published in Artificial Intelligence,
Volume 299,
2021,
103523,
ISSN 0004-3702 }}%
 \let\@evenfoot\@oddfoot}
\makeatother
\usepackage{lineno}
\usepackage[colorlinks=true, linkcolor=red, citecolor=red, urlcolor=blue]{hyperref}

\usepackage[small,center,hang]{subfigure} 
\usepackage[boxruled, vlined]{algorithm2e}
\SetKwComment{tcp}{$\triangleright~$}{}

\modulolinenumbers[50]
\journal{Artificial Intelligence}

\usepackage{enumitem}

\newenvironment{LIST} 
   {\begin{itemize}[parsep=1ex,itemsep=0mm,topsep=1ex,leftmargin=1.5em]}
   {\end{itemize}}

\usepackage{xspace}
\newcommand{\la}{\langle}
\newcommand{\ra}{\rangle}
\newcommand{\RAE}{\textsf{RAE}\xspace}

\newcommand{\srpe}{\textsf{SeRPE}\xspace}

\newcommand{\APE}{\textsf{RAE}\xspace}

\newcommand{\PLAN}{\textsf{UPOM}\xspace}
\newcommand{\UPOM}{\textsf{UPOM}\xspace}

\newcommand{\Progress}{\textsf{Progress}\xspace}
\newcommand{\Retry}{\textsf{Retry}\xspace}
\newcommand{\tried}{\textit{tried}\xspace}

\newcommand{\step}{\textit{i}\xspace}

\newcommand{\failed}{\textsf{failed}\xspace}

\newcommand{\R}{\mathbb{R}}

\newcommand{\sv}[1]{{\normalfont\textsf{#1}}} 
\newcommand{\set}[1]{\textit{#1}}

\newcommand{\range}{\textrm{Range}}

\newcommand{\argmax}{\textrm{argmax}}

\newcommand{\ot}{\leftarrow}

\usepackage{amsmath,amsthm,amssymb}
\newtheorem{theorem}{Theorem}

\theoremstyle{definition}   
\newtheorem{example}{Example}

\usepackage{aliascnt}
\newcommand{\theoremlike}[2]{
	\newaliascnt{#1}{theorem}
	\newtheorem{#1}[#1]{#2}
	\aliascntresetthe{#1}
}

\theoremstyle{definition}   

{\ifx&#1&\begin{definitionbody}\pushQED{\qed}\else\begin{definitionbody}[\text{#1}]\pushQED{\qed}\fi}%
{\popQED\end{definitionbody}}

\theoremlike{definitionbody}{Definition}

\newcommand{\nullset}{\varnothing}

\newcommand{\M}{\mathcal{M}\xspace}
\newcommand{\A}{\mathcal{A}\xspace}
\newcommand{\Tasks}{\mathcal{T}\xspace} 

\newlength{\tabsize}
\newcommand{\T}{\hspace*{\tabsize}}

\usepackage[dvipsnames]{xcolor}

\usepackage{newfloat}

\newenvironment{pcode}{%
	\setcounter{pline}{0}%
	\begin{tabular}[b]{@{\kern .8em}r@{~}l@{}l@{\,}}
		&&\\[-5.5ex] 
	}{%
	\end{tabular}%
}

\newcommand{\phead}[1]{
	\\ \multicolumn{3}{@{}l}{#1}}

\newcounter{pline}

\newcommand{\1}{\\&}

\newcommand{\pkey}[1]{\\#1:&}

\usepackage{placeins} 

\usepackage[normalem]{ulem}

\newcommand{\failure}{\textit{\textsf{retrial-failure}}\xspace}

\makeatletter
\def\oldfootnotesize{\@setsize\oldfootnotesize{9pt}\viiipt\@viiipt} 
\makeatother

\newcommand{\stack}{\ensuremath{\sigma}\xspace}

\newcommand{\Next}{\textsf{Next}\xspace}

\usepackage{tabularx}

\newcommand{\CR}{{\small \sv{Fetch}}\xspace}
\newcommand{\SD}{{\small \sv{Nav}}\xspace}
\newcommand{\SR}{{\small \sv{S\&R}}\xspace}
\newcommand{\EE}{{\small \sv{Explore}}\xspace}
\newcommand{\OF}{{\small \sv{Deliver}}\xspace}

\usepackage{listings}

\definecolor{codegreen}{rgb}{0,0.6,0}
\definecolor{codegray}{rgb}{0.5,0.5,0.5}
\definecolor{codepurple}{rgb}{0.58,0,0.82}
\definecolor{backcolour}{rgb}{1,1,1}

\lstdefinestyle{mystyle}{
    backgroundcolor=\color{backcolour},   
    commentstyle=\color{codegreen},
    keywordstyle=\color{magenta},
    numberstyle=\tiny\color{codegray},
    stringstyle=\color{codepurple},
    basicstyle=\ttfamily\footnotesize,
    breakatwhitespace=false,         
    breaklines=true,                 
    captionpos=b,                    
    keepspaces=true,                 
    numbers=left,                    
    numbersep=5pt,                  
    showspaces=false,                
    showstringspaces=false,
    showtabs=false,                  
    tabsize=2
}

\begin{document}

\begin{frontmatter}

\title{
Deliberative Acting, Online Planning and Learning with Hierarchical Operational Models
}

\author[umd]{Sunandita Patra\corref{mycorrespondingauthor}}
\ead{patras@umd.edu}
\cortext[mycorrespondingauthor]{Corresponding author}

\author[umd]{James Mason}
\ead{james.mason@jpl.nasa.gov}

\author[laas]{Malik Ghallab}
\ead{malik@laas.fr}

\author[umd]{Dana Nau}
\ead{nau@umd.edu}

\author[fbk]{Paolo Traverso}
\ead{traverso@fbk.eu}

\address[umd]{Dept.\ of Computer Science and Inst.\ for Systems Research, University of Maryland, College Park, Maryland, USA}
\address[laas]{LAAS-CNRS, Toulouse, France}
\address[fbk]{FBK-ICT, Povo-Trento, Italy}


\begin{abstract}
{\sloppy
In AI research, synthesizing a plan of action has typically used \textit{descriptive models} of the actions that abstractly specify \textit{what} might happen as a result of an action, and are tailored for efficiently computing state transitions. However, executing the planned actions has needed \textit{operational} models,
in which rich computational control structures and closed-loop online decision-making are used
to specify \textit{how} to perform an action in a nondeterministic  execution context, react to events and adapt to an unfolding situation. 
\emph{Deliberative actors}, which integrate acting and planning, have typically needed to use both of these models together---which causes problems when attempting to develop the different models, verify their consistency, and smoothly interleave acting and planning.

As an alternative, we define and implement an integrated acting and planning system in which both planning and acting use the same operational models. These rely on hierarchical task-oriented \textit{refinement methods} offering rich control structures.
The acting component, called Reactive Acting Engine (\RAE), is inspired by the well-known PRS system. At each decision step, \RAE can get advice from a planner for a near-optimal choice with respect to an utility function. The anytime planner uses a UCT-like  Monte Carlo Tree Search procedure, called \UPOM, 
whose rollouts are simulations of the actor's operational models. We also present learning strategies for use with \RAE and \UPOM that acquire, from online acting experiences and/or simulated planning results, a mapping from decision contexts to method instances as well as a heuristic function to guide \PLAN.
We demonstrate the asymptotic convergence of \PLAN towards  optimal methods  in static domains, and show experimentally that \UPOM and the learning strategies significantly improve the acting efficiency and robustness.}

\end{abstract}


\end{frontmatter}


\begin{footnotesize}
\end{footnotesize}

\sloppy  

\section{Introduction}
\label{sec:intro}

Consider a system required to autonomously perform a diversity of tasks in varying dynamic environments. Such a system, referred to as an ``\textit{actor}'' (following \cite{ghallab2014actors}), needs to be reactive and to act in a purposeful deliberative way. This requirements are usually addressed by combining a reactive approach and a plan-based approach using, respectively, \textit{operational models} and \textit{descriptive models} of actions.

Descriptive models specify \textit{what} might happen as a result of an action.
They are tailored to efficiently search a huge state space by representing state transitions with abstract preconditions and effects. This representation, inherited from the early STRIPS system \cite{fikes1971strips}, refined into SAS$^+$ \cite{Backstrom:1995tc,jonsson1998state}, the PDDL  languages \cite{mcdermott20001998, fox2003pddl2.1, fox2006modeling, haslum2019PDDL} and their nondeterministic and probabilistic variants, e.g.,  PPDDL \cite{Younes:04} and RDDL \cite{Sanner:10}, is used by most planning algorithms. 

Operational models  specify \textit{how} to perform an action. They are designed to take into account an elaborate context about ongoing activities, react to events and adapt to an unfolding situation. Using several computational paradigms with rich control structures (e.g., procedures, rules, automata, Petri nets), they allow for  closed-loop online decision-making. They have been designed into a diversity of languages, such as PRS, RAPS, TDL, Plexil, Golex, or RMPL  (see the survey \cite[Section 4]{Ingrand:2014ue}).

The combination of descriptive and operational models raises several problems.
First, it is difficult to take into account with two separate models the highly interconnected reasoning required between planning and deliberative acting. Second, the mapping between descriptive and operational models is very complex. A guarantee of the  consistency of this mapping is required in safety-critical applications, such as self-driving cars \cite{henaff2019model}, collaborative robots working directly with humans \cite{veloso2015cobots}, or
virtual coaching systems to help patients with chronic diseases \cite{ramchandani2019virtual}. However, to verify the consistency between the two different models is difficult (e.g., see the work on formal verification of operational models such as PRS-like procedures, 
using model checking and theorem proving \cite{deSilva:2018vd,cimatti03}). Finally, modeling is always a costly bottleneck; reducing the corresponding effort is beneficial in most applications.

Therefore, it is desirable  to have a single representation for both acting and planning. If such a representation were solely descriptive, it wouldn't provide sufficient functionality. Instead, the planner needs to be able to reason directly with the actor's operational models.

This paper describes an integrated planning and acting system in which both planning and acting use the actor's operational models.\footnote{Prior results about this approach have been presented in \cite{ghallab2016automated, patra2018ape,patra2019acting,patra2020integrating}. The last paragraph of 
\autoref{sec:rw} describes what the current paper adds to that work.
}
The acting component, called \textit{Refinement Acting Engine} (\RAE), is inspired by the well-known PRS system \cite{Ingrand:1996uj}.
\RAE uses a hierarchical task-oriented operational representation
in which an expressive, general-purpose language offers rich programming control structures for online decision-making. A collection of {\em hierarchical refinement methods} describes alternative ways to handle \emph{tasks} and react to \emph{events}. A method can be any complex algorithm, including 
\emph{subtasks}, which need to be refined recursively, and primitive \textit{actions}, which query and change the world \textit{nondeterministically}.
We assume that methods are manually specified (approaches for learning method bodies are discussed in \autoref{sec:discussion}).

Rather than behaving purely reactively, \RAE interacts with a
planner. To choose how best to refine tasks, the planner uses a Monte Carlo Tree Search procedure, called \PLAN, which assesses the utility of possible alternatives and finds an approximately optimal one.
Two utility functions are proposed reflecting the acting  \textit{efficiency} (reciprocal of the cost) and \textit{robustness} (success ratio).
Planning is performed using the constructs and steps as of the operational model, except that methods and actions are executed in a simulated world rather than the real one. When a refinement method contains an action, \PLAN takes samples of its possible outcomes, using either a domain-dependent generative simulator, when available, or a probability distribution of its effects.

\PLAN is used by \RAE as a progressive deepening, receding-horizon anytime planner. Its scalability requires heuristics. However, operational models lead to quite complex  search spaces not easily amenable to the usual techniques for domain-independent heuristics. 
Fortunately, this issue can be addressed with a learning approach to acquire a mapping from decision contexts 
to method instances; this mapping provides the base case of the anytime algorithm. Learning can also be used to acquire a heuristic function to prune deep Monte Carlo rollouts. We use  an off-the-shelf  learning library with appropriate adaptation for our experiments. Our contribution is not on the learning techniques {\em per se}, but on the integration of learning, planning, and acting. 
The learning algorithms do not provide the operational models needed by the planner, but they do several other useful things. They speed up the online planning search allowing for an anytime procedure. Both the planner and the actor can find better solutions, thereby improving the actor's performance.
The human domain author can write refinement methods without needing to specify a preference ordering in which the planner or actor should try instances of those methods.

Following a discussion of the state of the art in \autoref{sec:rw}, \autoref{sec:operational} describes the actor's architecture and the hierarchical operational model representation.
In Sections \ref{sec:rae}, \ref{sec:plan}, and \ref{sec:integration}, respectively, we present the acting component \RAE, the planning component \PLAN, and, the learning procedures for \RAE and \PLAN.
We provide an experimental evaluation of the approach in \autoref{sec:implementation}, followed by a discussion and a conclusion. The planner's asymptotic convergence to optimal choices is detailed in \ref{app:mapping}. The operational model for a search and rescue domain is described in \ref{app:rescue}. A table of notation is presented in \ref{app:notation}. The code for the algorithms and test domains is available online.

\section{Related Work}
\label{sec:rw}

To our knowledge, no previous approach has proposed the integration of planning, acting and learning directly with operational models.
In this section, we first discuss the relations with systems for acting, including those approaches that provide some (limited) deliberation mechanism. 
We then discuss the main differences of our approach with systems for on-line acting and planning, such as RMPL and Behaviour Trees. 
We continue the section by comparing our work with HTNs  
and with different approaches based on MDP and Monte Carlo search. 
We discuss the relation with work on integrating planning and execution in robotics,
the work on approaches based on temporal logics (including situation calculus and BDIs), 
and the approach to planning by reinforcement learning. 
We conclude the section with a relation with our prior work on the topic of  this paper.

\paragraph{(Deliberative) Acting} 
\label{sec:rw-acting}

Our acting algorithm and operational models are based on the \textit{Refinement Acting Engine}, \RAE algorithm \cite[Chapter 3]{ghallab2016automated}, which in turn 
is inspired from PRS \cite{Ingrand:1996uj}. 
If \RAE needs to choose among several eligible refinement method instances for a given task or event, it makes the choice without trying to plan ahead.
This approach
has been extended with some planning capabilities in
PropicePlan \cite{Despouys:1999va} and  \srpe \cite{ghallab2016automated}. 
Unlike our approach,
those systems model actions as classical planning operators; they both require the action models and the refinement methods to satisfy classical planning assumptions of deterministic, fully observable  and static environments, which are not acceptable assumptions for most acting systems.
Moreover, these works do not perform any kind of learning.

Various acting approaches similar to PRS and \APE have been proposed, e.g., \cite{Firby:1987tq,Simmons:1992fy,Simmons:1998wf,Beetz:1994tc,Muscettola:1998wka,myers1999cpef}. Some of these have refinement capabilities and hierarchical models, e.g., \cite{Verma:2005tl,Wang:1991ie,Bohren:2011bh}. While such systems offer expressive acting environments, e.g., with real time handling primitives, none of them provides the ability to plan with the operational models used for acting, and thus cannot integrate  acting and planning as we do. Most of these systems do not reason about alternative refinements, and do not perform any kind of learning.

\paragraph{Online Acting and Planning} 
\label{sec:rw-online}

Online planning and acting is addressed in many approaches, e.g., \cite{musliner2008evolution,goldman2016hybrid,goldman2009semantics}, but their notion of ``online'' is different from ours. For example, in \cite{musliner2008evolution}, the old plan is executed repeatedly in a loop while the planner synthesizes a new plan, which isn't installed until planning has been finished. In \PLAN, hierarchical task refinement is simulated to do the planning, and can be interrupted anytime when \RAE needs to act.

The Reactive Model-based Programming Language (RMPL) \cite{Ingham:2001uga} is a comprehensive CSP-based approach for temporal planning and acting, which combines a system model with a control model. The system model specifies nominal as well as failure state transitions with  hierarchical constraints. The control model uses standard reactive programming constructs. RMPL programs are transformed into an extension of Simple Temporal Networks with symbolic constraints and decision nodes \cite{Williams:2001wt,Conrad:2009tu}. Planning consists in finding a path in the network that meets the constraints. RMPL has been extended with error recovery, temporal flexibility, and conditional execution based on the state of the world \cite{Effinger:2010um}. Probabilistic RMPL is introduced in \cite{Santana:2014vw,Levine:2014wo} with the notions of weak and strong consistency, as well as uncertainty for contingent decisions taken by the environment or another agent. The acting system adapts the execution to observations and predictions based on the plan. RMPL and subsequent developments have been illustrated with a service robot which observes and assists a human.  Our approach does not handle time; it focuses instead on hierarchical decomposition with Monte Carlo rollout and sampling.

Behavior trees (BT) \cite{colledanchise17behaviour,colledanchise17how,
Colledanchise:2015um} can also  respond reactively to contingent events that were not predicted. 
In \cite{colledanchise17behaviour,colledanchise17how}, BT are synthesized by planning. In \cite{Colledanchise:2015um} BT are generated by genetic programming.
Building the tree refines the acting process by mapping the descriptive action model onto an operational model. We integrate acting, planning, and learning directly in an operational model with the control constructs of a programming language.  
Moreover, we learn how to select refinement methods and method instances in a natural and practical way to specify different ways of accomplishing a task.

\paragraph{Hierarchical Task Networks} 
\label{sec:rw-HTN}

Our methods are significantly different from those used in HTNs \cite{nau1999shop}: 
to allow for the operational models needed for acting,
we use rich control constructs rather than simple sequences of primitives.
Learning HTN methods has also been investigated. HTN-MAKER \cite{hogg2008HTN} learns methods given a set of actions, a set of solutions to classical planning problems, and a collection of annotated tasks. This is extended for nondeterministic domains in \cite{hogg2009learning}.  \cite{hogg2010learning} integrates HTN with Reinforcement Learning (RL), and estimates the expected values of the learned methods by performing Monte Carlo updates.   
At this stage, we do not learn the methods but only how to chose the appropriate one.

\paragraph{MDP and Monte Carlo Tree Search} 
\label{sec:rw-MDP}

A wide literature on MDP planning and Monte Carlo Tree Search refers to simulated execution, e.g.,  \cite{feldman2013monte-carlo,feldman2014monte,james2017analysis} and sampling outcomes of action models e.g.,  RFF  \cite{Teichteil:2008vq}, FF-replan \cite{yoon2007ff}, or hindsight optimization \cite{yoon2008probabilistic}. In particular, our \PLAN procedure is an adaptation of the popular UCT algorithm \cite{kocsis2006bandit}, which has been used for various games and MDP planners, e.g., in PROST for RDDL domains \cite{Keller:2012vu}.
The main conceptual and practical difference with our work is that these approaches use descriptive models, i.e., abstract actions on finite MDPs. 
Although most of the papers refer to online  planning, they plan using descriptive models rather than operational models. There is no integration of acting and planning, hence no concerns about the planner's descriptive models versus the actor's operational models.  
Some works deal with hierarchically structured MDPs (see, e.g., \cite{parr98hierarchical,barry2013hierarchical,hauskrech2013hierarchical}), 
or hierarchical extensions of Hidden Markov Models for plan recognition (see, e.g., \cite{duong2009efficient}).
However, most of the approaches based on MDPs do not deal with hierarchical models and none of them is based on refinement methods.

\paragraph{Planning and Execution in Robotics} 
\label{sec:rw-robotics}

There has been a lot of work in robotics to integrate planning and execution. They propose various techniques and strategies to handle the inconsistency issues that arise when execution and planning are done with different models. \cite{lallement2014hatp} shows how HTN planning can be used in robotics. \cite{garrett2018ffrob} and \cite{garrett2018stripstream} 
integrates task and motion planning for robotics. 
The approach of \cite{Morisset:2008gy} addresses a problem similar to ours but specific to robot navigation. Several methods for performing a navigation task and its subtasks are available, each with strong and weak points depending on the context. The problem of choosing a best method instance for starting or pursuing a task in a given context is formulated as receding-horizon planning in an MDP for which a model-explicit RL technique is proposed. Our approach is not limited to navigation tasks; it allows for richer  hierarchical refinement models and is combined with a powerful Monte-Carlo tree search technique. 

The Hierarchical Planning in the Now (HPN) of \cite{kaelbling_hierarchical_2011} is designed for integrating task and motion planning and acting in robotics. Task planning in HPN relies on a goal regression hierarchized according to the level of fluents in an operator preconditions. The regression is pursued until the preconditions of the considered action (at some hierarchical level) are met by current world state, at which point acting starts. Geometric reasoning is performed at the planning level (i) to test ground fluents through procedural attachement (for truth, entailment, contradiction), and (ii) to focus the search on a few suggested branches corresponding to geometric bindings of relevant operators using heuristics called geometric suggesters. It is also performed at the acting level to plan feasible motions for the primitives to be executed. HPN is correct but not complete; however when primitive actions are reversible, interleaved planning and acting is complete. HPN has been extended in a comprehensive system for handling geometric uncertainty  \cite{kaelbling_integrated_2013}. 

The integration of task and motion planning  problem is also addressed in \cite{wolfe_combined_2010-1}, which uses an HTN approach. Motion primitives are assessed with a specific solver through sampling for cost and feasibility. An algorithm called SAHTN extends the usual HTN search with a bookkeeping mechanism to cache previously computed motions. In comparison to this work as well as to HPN, our approach does not integrate specific constructs for motion planning. However, it is more generic regarding the integration of planning and acting.

\paragraph{Logic Based Approaches} 
\label{sec:rw-logic}

Approaches based on temporal logics 
(see e.g., \cite{Doherty:2009jba}) 
specify acting and planning knowledge through high-level descriptive models and not through operational models like  in \APE.
Moreover, these approaches integrate acting and planning without exploiting the hierarchical refinement approach described here. 
Models based on GOLOG (see, e.g., \cite{Hahnel:1998tl,Classen:2012tn,Ferrein:2008gl})
share some similarities with \APE operational models, since they extend situation calculus with control constructs, procedure invocation and nondeterministic choice.
In principle, GOLOG can use forward search to make a decision at each nondeterministic choice. This resembles our approach where \PLAN is used to choose a method to be executed among the available ones. 
However, no work based on GOLOG provides an effective and practical planning method such as \PLAN, which is based on UCT-like Monte Carlo Tree Search and simulations over different scenarios. 
It is not clear to us how GOLOG could be extended in such a direction.

There are many commonalities between \RAE and architectures based on BDI (Belief-Desire-Intention) models \cite{meneguzzi2015planning, de2020bdi, sardina2006hierarchical, de2018operational}. Both approaches rely on a reactive system, but with differences regarding their primitives as well as their methods or plan-rules. 
Several BDI systems rely on a descriptive model, e.g.,  specified in PDDL, 
whereas \RAE can handle any type of skill (e.g., physics-based action simulators) with nondeterministic effects. Because of the latter, there are also differences about how they can do planning. 
Nondeterministic approaches are required in \RAE for the selection of methods, and since we are relying on a reactive approach, we do not synthesize a policy per se, but only a receding-horizon best choice.

\paragraph{Reinforcement Learning and Learning Domain Models} 
\label{sec:rw-RL}

Our approach shares some similarities with the work on planning by reinforcement learning (RL)
\cite{kaelbling1996reinforcement,sutton1998reinforcement,Geffner:2013to,leonetti2016synthesis,garmelo2016towards}, since we learn by acting in a (simulated) environment. However, most of the works on RL learn policies that map states to actions to be executed, and  learning is performed in a descriptive model. 
We learn how to select refinement method instances in an operational model that allows for programming control constructs.
This main difference holds also with works on hierarchical reinforcement learning, see, e.g., \cite{yang2018peorl}, \cite{parr1997reinforcement}, \cite{ryan2002using}. 
Works on user-guided learning, see e.g., 
\cite{martinez2017relational}, \cite{martinez2017relational2}, use model based RL to learn relational models, and the learner is integrated in a robot for planning with exogenous events.
Even if relational models are then mapped to execution platforms, the main difference with our work still holds: Learning is performed in a descriptive model.
\cite{jevtic2018robot} uses RL for user-guided learning directly  in the specific case  of robot motion primitives.

Several approaches have been investigated for learning planning-domain models.  In probabilistic planning, for example \cite{ross2011bayesian}, or \cite{katt2017learning}, learn a POMDP domain model through interactions with the environment, in order to plan by reinforcement learning or by sampling methods. 
In these cases, no integration with operational models and hierarchical refinements is provided.

\paragraph{Relation with our prior work} 
\label{sec:rw-ours}

Here is how the current paper relates to our prior work on this topic. A pseudocode version of \RAE first appeared in \cite{ghallab2016automated}. An implementation of \RAE, and three successively better planners for use with it, were described in \cite{patra2018ape,patra2019acting,patra2020integrating}. The current paper is based on \cite{patra2020integrating}, with the following additional contributions. We provide complete formal specifications and explanations of the actor \RAE and planner \UPOM. We present a learning strategy to learn values of uninstantiated method parameters, with experimental evaluation. We have an additional experimental domain, called \OF. We propose a new performance metric, called Retry Ratio, and evaluate it on our five experimental domains. We perform experiments with success-ratio (or probability of success) as the utility function optimized by \UPOM, We compare success ratio with efficiency. We perform experiments with varying the parameters, number of rollouts and maximum rollout length, of \UPOM.
We provide a proof of convergence of \UPOM to a plan with optimal expected utility.

\section{Architecture and Representation}
\label{sec:operational}

Our approach integrates reactive and deliberative capabilities on the basis of hierarchical operational models. It focuses on a reactive perspective, extended with planning capabilities. This section presents the global architecture of the actor and details the ingredient of the representation.

\subsection{Architecture}

The popular three-layer architectures for autonomous deliberative actors usually combine \textit{(i)} a platform layer with sensory-motor modules, \textit{(ii)} a reactive control layer, and \textit{(iii)} a deliberative planning layer \cite{kortenkamp08}. As motivated earlier, our approach merges the last two layers within a reactive-centered perspective.

\begin{figure}[!htb]
  \centering 
\subfigure[]{
    \includegraphics[width=.45\textwidth]{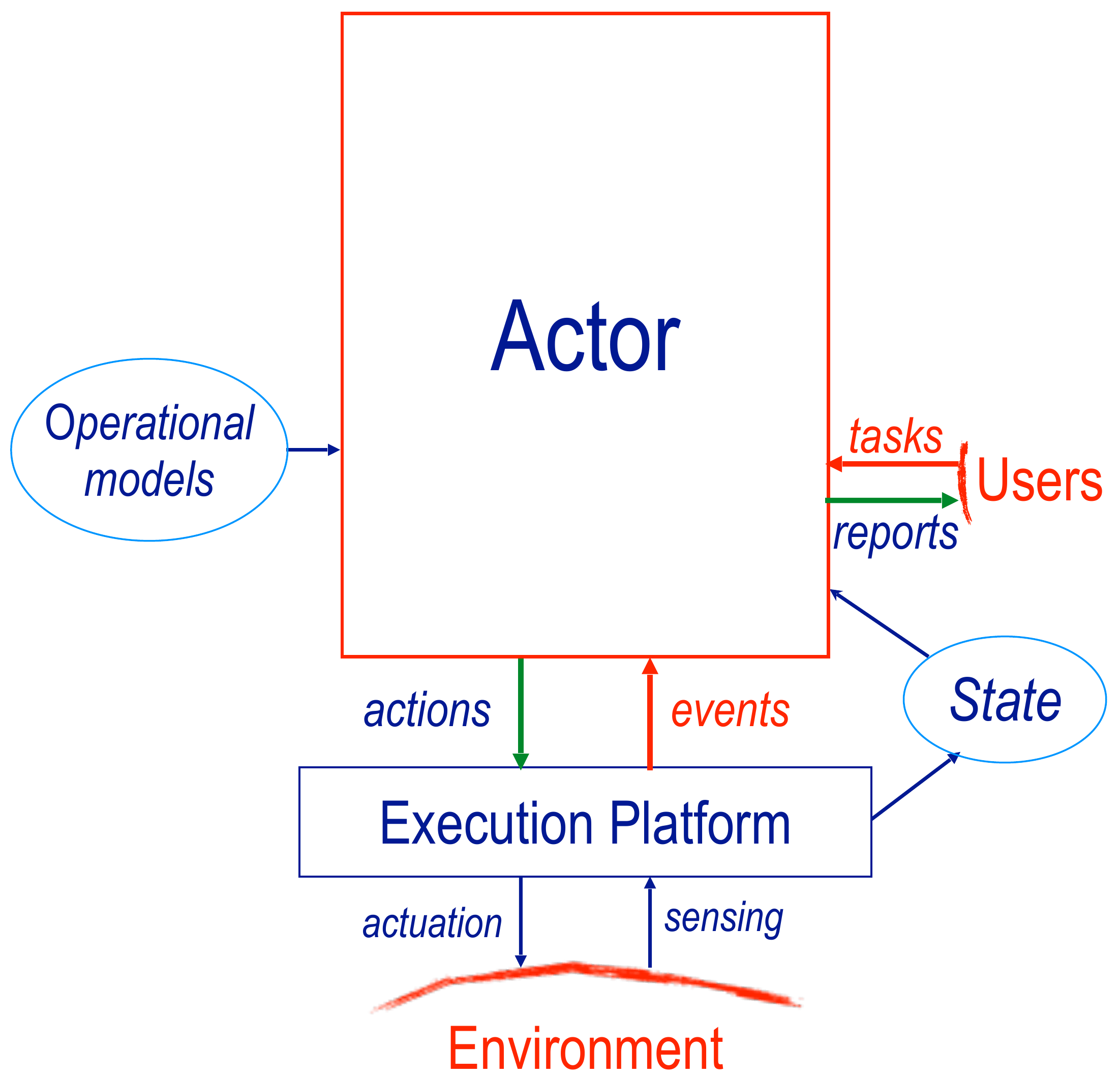}
    \label{fig:actor}} 
\subfigure[]{
    \includegraphics[width=.49\textwidth]{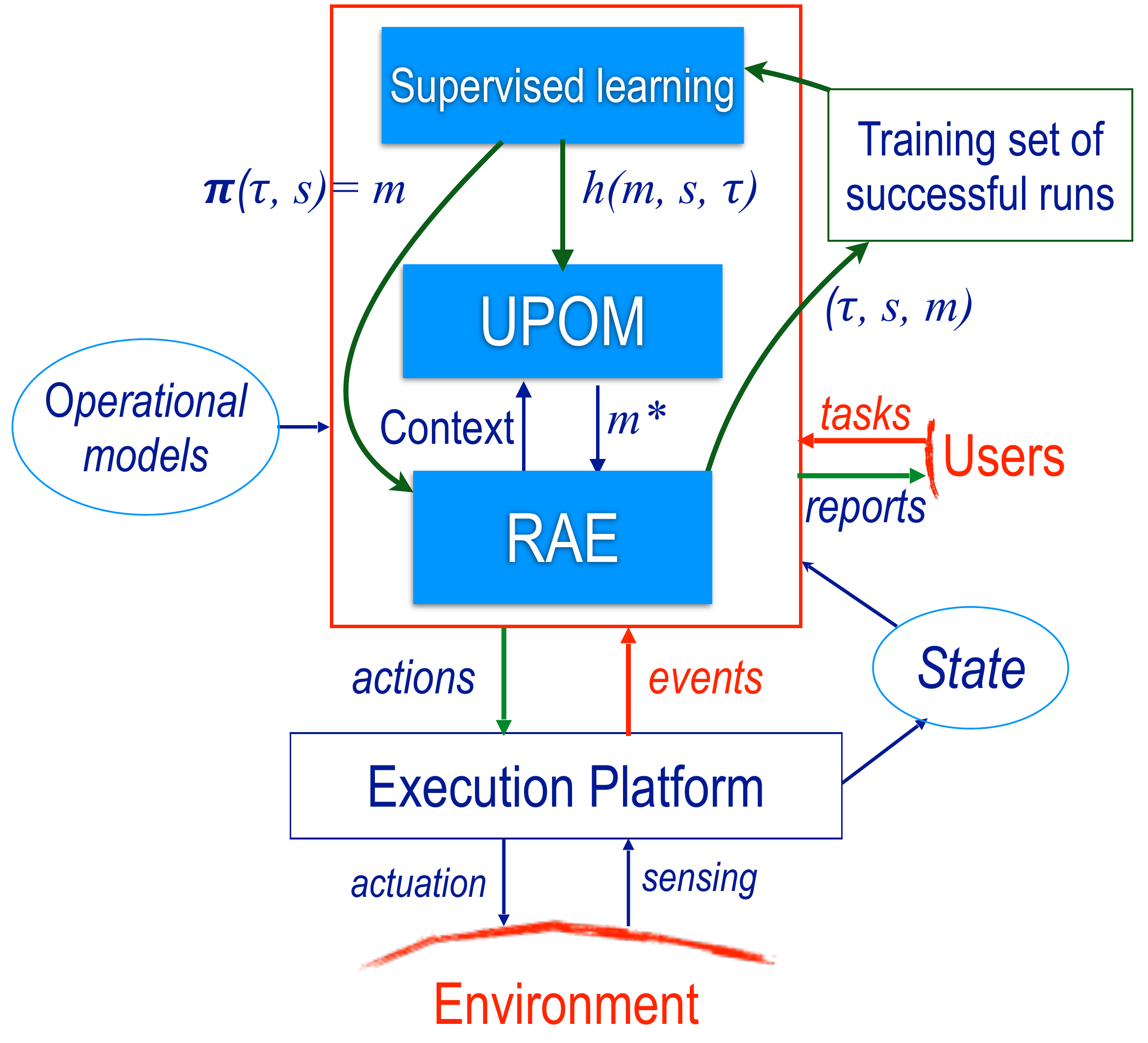}
    \label{fig:archi}}
  \caption{(a) Architecture of an actor reacting to events and tasks through an execution platform; (b) Integration of refinement acting, planning and learning.}
  \label{fig:}
\end{figure}

The central component of the architecture (labelled ``Actor'' in Figure \ref{fig:actor}) interacts with the environment for sensing and actuation through an execution platform, from which it receives events and world state updates. It also interacts with users  getting tasks to be performed and reporting on their achievement.  

\RAE
is the driving system within the actor. It reacts to tasks and events through hierarchical refinements specified by a library of operational models. At each decision step, \RAE uses the planner \UPOM to make the appropriate choice. \UPOM performs a look-ahead by simulating available options in the current context. Supervised learning is used to speed-up \UPOM with a heuristic, avoiding very deep and costly Monte Carlo rollouts; it also provides a base policy for an anytime strategy when the actor has no time for planning (see Figure \ref{fig:archi}, the integration of planning, learning and refinement acting is detailed in subsequent sections).

\subsection{Hierarchical Operational Models}

We rely  on a formalism described in \cite[Chapter 3]{ghallab2016automated}, which has been designed for acting and reacting in a dynamic environment. It provides a hierarchical representation of tasks through alternative refinement methods and primitive actions. 
Let us detail its main ingredients (see table of notation in \ref{app:notation}).

\paragraph{State variables}

We rely on a representation with parameterized state variables. These are a finite collection of mappings from typed sets of objects of the domain into some range. 
For example,  to describe  the kinematic configuration of a two-arm robot $r$, its location with respect to a floor reference frame, and  the status of a door $d$, one might use state variables \sv{configuration}$(r)\in\R^{17}$, \sv{loc}$(r)\in\R^2$, \sv{door-status}$(d)\in$ \{\sv{closed, open, cracked}\}.  Let $X$ be a finite set of  such state variables; variable $x \in X$\label{def:X} takes values from the set $\range(x)$. 
For each $x\in X$, to provide a convenient notation for handling partial knowledge, $\range(x)$ is extended to include the special symbol \sv{unknown}, which is the default value if $x$
has not been set or updated to another value.

\phantomsection \label{def:xi}A state is a total assignment  of values to state variables. The world state $\xi$ is updated through observation by the execution platform, reflecting the dynamics of the external world. This update is continual for some state variables, e.g., \sv{configuration}$(r)$, with the usual propriosensing in robotics. Other variables may be updated to \sv{unknown}, unless they are in the range of a sensor or known to keep a value set by the actor. In the following algorithms, an update step, denoted as ``observe current state $\xi$'', precedes every reactive decision of the actor.

\paragraph{State abstraction}
For the purpose of the planning lookahead, $\xi$ is simplified into an abstract state $s\in S$ which evolves by reasoning. The state $s$, which gets updated from the observed state $\xi$ each time the actor calls the planner (see \autoref{sec:plan}), is a domain-dependent abstraction of $\xi$ meeting the following conditions:
\begin{LIST}
\item The set of state variables of the abstract state $s$ is a subset of $X$, i.e., some state variables in $\xi$ may be ignored in $s$.
\item For a variable $x$ in $s$, Range($x$) can be a discretization of its range in $\Xi$. Continuing with the preceding robotics example, we might have $\range(\sv{configuration}(r)) = \{\sv{packed},\sv{holding},\sv{carrying},\sv{manipulating}\}$, and $\range(\sv{loc}(r)) =$ a finite set of locations.
Note that the formal proof of the asymptotic convergence of \UPOM  assumes $S$ to be finite. However, this assumption is not a practical requirement of the planning algorithm.\footnote{We report (\autoref{sec:eval}) on a test domain with  continuous state variables in $S$.}
\end{LIST}
Hence, a world state $\xi$ is deterministically mapped to a single $s$=Abstract($\xi$). An abstract state $s$ may correspond to a subset of world states; but there is never a need to map back an abstract state. Indeed, the reasoning by simulation on $s$=Abstract($\xi$) informs the actor about an appropriate choice of a method at some point. But the abstract states $s', s'', \ldots$, to which this reasoning led are not needed by the actor whose actions always rely on the current observed state.

\paragraph{Other variables and relations}
The state variables in $X$ are managed in the acting and planning algorithms as global variables. However, since methods embody programs, it is  convenient to define \textit{local variables}, which are generally derived from other variables.
 For example, one might use $\sv{stable}(o,\textit{pose}) \in \{ \top, \bot \}$, to mean that object $o$ in some particular \textit{pose} is stable; this property results from some geometric and dynamic computation. Local variables are updated by assignment statements inside methods. An assignment statement is of the form $x\ot \set{expr}$, where \set{expr} may be either a ground value in $\range(x)$, or a computational expression that returns a ground value in $\range(x)$. Such an expression may include, for example, calls to specialized software packages.

It is convenient to define the unvarying properties of a domain through a set of  \textit{rigid relations} (as opposed to fluents, in our case parametrized state variables). For example, the adjacency of two locations or the color objects (if relevant) could be described with rigid relations.

\paragraph{Tasks}
A task is a label naming an activity to be
performed. It has the form \sv{task-name}(\textit{args}), where
\sv{task-name} designates the task considered, arguments
\textit{args} is an ordered list of objects and values.  Tasks specified by  a user  
are called {\em root} tasks, to distinguish them from the subtasks in which they are refined.
	
\paragraph{Events}
 An event designates an occurrence of some type detected
by the execution platform; it corresponds to an \textit{exogenous} change in the environment to which the actor may have to react, e.g., the
activation of an emergency signal. It has the form
\sv{event-name}(\textit{args}).

\paragraph{Actions}
 An action is a primitive function with instantiated parameters that can be executed by the execution platform through sensory motor commands. It has \textit{nondeterministic} effects. For the purpose of planning, we do not represent actions with formal templates, as usually done with descriptive models. Instead, we assume there is a generative nondeterministic sampling simulator, denoted \sv{Sample}. A call to \sv{Sample}$(a,s)$ returns a state $s'$ randomly drawn among the possible states resulting from the execution of $a$ in $s$. \sv{Sample} can be implemented simply through a probability distribution of the effects of $a$ (see \autoref{sec:plan}).

When the actor triggers an action $a$ for some task or event, it waits until $a$ terminates or fails before pursuing that task or event. To follow its execution progress, when action $a$ is triggered, there	 is a state variable, denoted $\sv{execution-status}(a)
\in \{\sv{running, done, failed}\}$, which expresses the fact that the execution of $a$ is going on, has terminated or failed. A terminated action returns a value of some type, which can be used to branch over various followup of the activity.

\paragraph{Refinement Methods}
A refinement method is a triple of the form $(\set{task}, \set{precondition},
\set{body})$ or $(\set{event}, \set{precondition},
\set{body})$.  The first field, either a task or an event, is its \set{role}; it tells
what the method is about.  When the \set{precondition} holds in the
current state, the method is \textit{applicable} for addressing the task or event in
its role by running a program given in the method's \set{body}. This
program refines the task or event into a sequence of subtasks,
actions, and assignments. It may use recursions and iteration loops, but its sequence of steps is assumed to be finite.\footnote{
One way to enforce such a restriction would be as follows. For each iteration loop, one could require it to have a loop counter that will terminate it after a finite number of iterations. For recursions, one could use a {\em level mapping} 
(e.g., see \cite{erol1995complexity,hitzler2005uniform})
that assigns to each task $t$ a positive integer $\ell(t)$, and require that for every method $m$ whose task is $t$ and every task $t'$ that appears in the body of $m$, $\ell(t') < \ell(t)$. 
However, in most problem domains it is straightforward to write a set of methods that don't necessarily satisfy this property but still don't produce infinite recursion.
}

Refinement methods are specified as parameterized templates with a name and list of arguments $\textit{method-name}(\textit{arg}_1,\ldots,\textit{arg}_k)$.
An instance of a method is given by the substitution of its arguments
by constants that are the values of state variables.  

A method instance
is applicable for a task if its role
matches a current task or event, and its preconditions are satisfied by
the current values of the state variables.  
A method may have several applicable instances for
a current state, task, and event. An applicable instance of a
method, if executed, addresses a task or an event by refining it, in a context dependent manner, into subtasks, actions, and possibly state updates, as specified in its body.  

The body of a method is a sequence of lines with the usual programming control structure (if-then-else, while loops, etc.), and  tests on the values of state variables.  A \emph{simple} test has the form $(x \circ v)$, where $\circ \in \{=,
\neq, <, >\}$. A \emph{compound} test is a negation, conjunction, or
disjunction of simple or compound tests.  Tests are evaluated with
respect to the current state $\xi$.  In tests, the symbol
\sv{unknown} is not treated in any special way; it is just one of the
state variable's possible values.

The following example of a simplified
 search-and-rescue domain illustrates the representation.

\begin{example}
	\label{ex:ee1}
Consider a set $R$ of robots performing search and rescue operations in a  partially mapped area. The robots have to find people needing help in some area and leave them a package of supplies (medication, food, water, etc.). This domain is specified with state variables such as 
$\sv{robotType}(r) \in $ \{UAV, UGV\}, $r \in R$, a finite set of robot names; 
$\sv{hasSupply}(r) \in  \{\top, \bot\}$;
\sv{loc}$(r) \in L$, a finite set of locations. A rigid relation $\sv{adjacent} \subseteq L^2$ gives the topology of the domain.

\sloppy 
These robots can use actions such as $\textsc{Detect}(r, \text{camera}, \text{class})$ (which detects if an object of some \textit{class} appears in images acquired by \textit{camera} of $r$), $\textsc{TriggerAlarm}(r,l)$, $\textsc{DropSupply}(r,l)$, $\textsc{LoadSupply}(r,l)$, $\textsc{Takeoff}(r, l)$, $\textsc{Land}(r, l)$, $\textsc{MoveTo}(r,l)$, and $\textsc{FlyTo}(r,l)$. They can address tasks such as: $\sv{search}(r,\textit{area})$ (which makes a UAV $r$ survey in sequence the locations in area), $\sv{survey}(l)$, $\sv{navigate}(r, l)$, $\sv{rescue}(r, l)$, and $\sv{getSupplies}(r)$. 
 
Here is a refinement method for the $\sv{survey}$ task:

\begin{quote}
	{\rm
		\begin{pcode}
			\phead{\sv{m1-survey}$(l, r)$}
			\pkey{task}{\sv{survey}$(l)$}
			\pkey{pre}{\sv{robotType}$(r) =$ \textit{UAV} and \sv{loc}$(r) = l$ and \sv{status}$(r)$ = \textit{free}}
			\pkey{body} 
			 for all $l'$ in neighbouring areas of $l$ do:
			\1 \T \sv{moveTo}$(r, l')$
			\1 \T for \textit{cam} in \sv{cameras}($r$):
			\1 \T \T if \textsc{DetectPerson}($r$, \textit{cam}) $ = \top$ then:
			\1 \T \T \T  if \sv{hasSupply}($r$) then \sv{rescue}($r,l'$)
			\1 \T \T \T else \textsc{TriggerAlarm}$(r,l')$
		\end{pcode}
	}

\end{quote}

This method specifies that in the location $l$ the UAV $r$ detects if a person appears in the images from its camera. In that case, it proceeds to a rescue task if it has supplies; if it does not it triggers an alarm event. This event is processed (by some other methods) by finding the closest robot not involved in a current rescue and assigning to it a rescue task for that location.

\begin{quote}
\begin{pcode}
	\phead{\sv{m1-GetSupplies}$(r)$}
	\pkey{task}{\sv{GetSupplies}$(r)$}
	\pkey{pre}{\sv{robotType}$(r) =$ \textit{UGV}}
	\pkey{body} \sv{moveTo}$(r, $\sv{loc}$(BASE))$
	\1 \T \textsc{ReplenishSupplies}$(r)$
	\bigskip
	\phead{\sv{m2-GetSupplies}$(r)$}
	\pkey{task}{\sv{GetSupplies}$(r)$}
	\pkey{pre}{\sv{robotType}$(r) =$ \textit{UGV}}
	\pkey{body}
	$r_2 \gets \textrm{argmin}_{r'} \{\textrm{Distance}(r, r') \mid $ \sv{hasMedicine}$(r') = \textsc{True}\}$
	
	\1 if $r_2$ = None then \textsc{Fail}
	\1 else:
	\1 \T \sv{moveTo}$(r, loc(r_2))$
	\1 \T \textsc{Transfer}$(r_2, r)$
\popQED{\qed}
\end{pcode}
\end{quote}
\end{example}

\paragraph{Specification of an acting domain} \label{def:Sigma} We model an acting domain $\Sigma$ as a tuple $\Sigma=(\Xi, \mathcal{T, M, A})$ where:
\begin{LIST}
\item $\Xi$ is the set of world states the actor may be in.
\item \phantomsection \label{def:tau}$\mathcal{T}$ is the set of tasks and events the actor may have to deal with.
\item \phantomsection \label{def:m}$\mathcal{M}$ is the set of methods for handling tasks or events in $\mathcal{T}$, 
 \label{def:applicable} $\sv{Applicable}(\xi,\tau)$ is the set of method instances applicable to $\tau$ in  state $\xi$.
\item  \phantomsection \label{def:a}$\mathcal{A}$ is the set of actions the actor may perform. 
 \label{def:gamma}We let $\gamma(\xi,a)$ be the set of states that may be reached after performing action $a$ in state $\xi$.
\end{LIST}
We assume that $\Xi$, $\Tasks$, $\M$, and $\A$ are finite.

The deliberative acting problem can be stated informally as follows: given $\Sigma$ and a task or event $\tau \in \mathcal{T}$, what is the ``best'' 
method instance $m \in \mathcal{M}$ to perform $\tau$ in a current state $\xi$. In Example~\ref{ex:ee1}, for a task \sv{getSupplies}$(r)$, the choice is between \sv{m1-GetSupplies}$(r)$ and \sv{m2-GetSupplies}$(r)$. 
Strictly speaking, the actor does not require a plan, i.e., an organized set of actions or a policy. It requires a selection procedure which designates for each task or subtask at hand the ``best'' method instance for pursuing the activity in the current context.

The next section describes a reactive actor which relies on a predefined preference order of methods in $\sv{Applicable}(\xi,\tau)$. Such an order is often natural when specifying the set of possible methods for a task. In \autoref{sec:plan} we detail a more informed receding-horizon look-ahead mechanism using an approximately optimal 
planning algorithm which provides the needed selection procedure.

\newcommand{\current}{i}
\newcommand{\ldot}{.}
\setlength{\algomargin}{1.5em} 
\newcommand{\actingStack}{\textit{Stack}\xspace}
\newcommand{\planningStack}{R_p}
\newcommand{\uctRollouts}{n_{ro}}

\section{Acting with \textsf{RAE}} 
\label{sec:rae}

\RAE
 is adapted from \cite[Chapter 3]{ghallab2016automated}. 
It maintains an \textit{Agenda} consisting of a set of \textit{refinement stacks}, one for each root task or event that needs to be addressed.
A refinement stack  \stack is a  LIFO list of tuples of the form $(\tau,m,i,\tried)$ where $\tau$ is an identifier for the task or event; $m$ is a method instance to refine $\tau$ (set to \textit{nil} if no method instance has been chosen yet); $i$ is a pointer to a line in the body of $m$, initialized to 1 (first line in the body); and \tried is a set of refinement method instances already tried for
$\tau$ that failed to accomplish it. A stack \stack is handled with the usual \sv{push}, \sv{pop} and \sv{top} functions.

\RestyleAlgo{boxed}

\begin{algorithm}[!h]  \label{alg:rae} \DontPrintSemicolon
\uline{$\RAE$:}\;
		$\set{Agenda} \gets$ empty list\;
		\While{True}
		{\nl\For{each new task or event $\tau$ to be addressed }
		{\label{forloop1} 
\nl			observe current state $\xi$      \label{obs1}\;
\nl			$m \gets \sv{Select}( \xi, \tau, \la (\tau, nil, 1, \emptyset)\ra,d_{max}, \uctRollouts)$ \label{plan1}\;
\nl				\lIf{$m = \emptyset$}
				{output($\tau$, ``failed'') } \label{failure0}
				\lElse
				{
					$\set{Agenda} \gets \set{Agenda}\cup \{\langle (\tau,m,1,\emptyset) \rangle\}$   }
			}
\nl 			\For {each $\stack \in \textit{Agenda}$}{\label{forloop2}
			observe current state $\xi$     \;
				$\stack \gets \Progress(\stack, \xi)$ \;
\nl				\If{$\stack=\emptyset$ }
				{$\set{Agenda}\gets \set{Agenda} \setminus \stack$ \;
				output($\tau$, ``succeeded'')}
\nl				\ElseIf{\stack=\failure}
				{$\set{Agenda}\gets \set{Agenda} \setminus \stack$ \;
				output($\tau$, ``failed'')} \label{failure1}
			}}
\caption{Refinement Acting Engine \RAE}
\end{algorithm}

When \APE addresses a task $\tau$, it must choose a method instance $m$ for $\tau$. 
This is performed by function \sv{Select} (lines \ref{plan1} of \APE, \ref{plan2} of \Progress, and \ref{plan3} of \Retry). \textsf{Select} takes five arguments: the current state $\xi$, task $\tau$, and stack \stack, and two control parameters $d_{max}, \uctRollouts$ which are needed only for planning.  In purely reactive mode (without planning), \sv{Select} returns the first applicable method instance, according to a pre-defined ordering, which has not already been tried (\tried is given in \stack).
Note that this choice is with respect to the current world state $\xi$. Lines \ref{obs1},\ref{obs2},\ref{obs3} in  \APE, \Progress and \Retry respectively,  specify to get an update of the world state from the execution platform.
If  \sv{Applicable}$(\xi,\tau) \subseteq \tried$, then \sv{Select} returns $\emptyset$, i.e., there is no applicable method instances for $\tau$ in $\xi$ that has not already been tried, meaning a failure to address $\tau$.

The first inner loop of \APE (line \ref{forloop1}) reads each new root task or event $\tau$ to be addressed and adds to the \textit{Agenda} its refinement stack, initialized to $\langle(\tau,m,1,\emptyset)\rangle$, $m$ being the method instance returned by \sv{Select},  if there is one. The root task $\tau$  for this stack will remain at the bottom of  \stack until solved; the subtasks in which $\tau$ refines will be pushed onto  \stack along with the refinement. The second loop of \RAE progresses by one step in the topmost method instance of each stack in the $\textit{Agenda}$.

\begin{algorithm}[!ht] \label{alg:progress} 	\DontPrintSemicolon
	\uline{$\Progress(\stack,\xi)$:}\;
	$(\tau,m,i,\tried)\gets$ \sv{top}($\stack$) \;
	
\nl \label{action}	
	\If 
	{$m[i]$ is an already triggered action}
		{\textbf{case} \sv{execution-status}($m[i]$): \;
			\T\textsf{running: } \text{return} \stack\;
\nl			\T\textsf{failed: }\T  \text{return} $\Retry(\stack)$ \label{alg:fail1}\;
			\T\textsf{done: }\T    \text{return} ${\Next}(\stack,\xi)$
			} 
			
\nl \label{next}	
	\ElseIf {$m[i]$ is an assignement step}
		{update $\xi$ according to $m[i]$ \;
		\text{return} $\Next(\stack,\xi)$}
	\ElseIf {$m[i]$ is an action $a$}
		{trigger the execution of action $a$ \;
		\text{return} \stack}
		\ElseIf {$m[i]$ is a task $\tau'$}
		{\nl		observe current state $\xi$ \label{obs2}\;
\nl 		$m' \gets \sv{Select}( \xi, \tau',\stack,d_{max}, \uctRollouts)$    \label{plan2}\;
		\lIf {$m' = \emptyset$} {return $\Retry(\stack)$}
		\lElse {return \sv{push}$((\tau',m',1,\emptyset),\stack)$}
		}
	
\caption{\Progress returns an updated stack taking into account  the execution status of the ongoing action, or the type of the next step in method instance $m$.}
\end{algorithm}

To progress a refinement stack \stack, \sv{Progress} (Algorithm \ref{alg:progress}) focuses on the tuple $(\tau,m,i, \tried)$ at the top of \stack.  \phantomsection \label{def:mi}If the current line $m[i]$ is an action already triggered, then the execution status of this action is checked. If the action $m[i]$ is still running, this stack has to wait, but \RAE goes on for other pending stacks in the \set{Agenda}.
If $m[i]$ failed, \sv{Retry} examines alternative method instances. Otherwise the action $m[i]$ is done: \RAE will proceed in the following iteration with the next step in method instance $m$, as defined by the function \Next (Algorithm \ref{alg:next}).

\begin{algorithm}[!ht] \label{alg:next} 	\DontPrintSemicolon
	\uline{\Next($\stack, \xi$):}\;
	
	\Repeat {$i$ is not the last step of $m$}
		{$(\tau,m,i,\tried)\gets$ \sv{top}($\stack$) \;
			\textsf{pop}(\stack)\;
			\lIf {\stack $=\la\ra$}
			{
				return $\la\ra$
			}
	}
	$j \gets \text{step following~} i \text{ in } m $ depending on $\xi$ \;
		\text{return} $\sv{push}((\tau, m, j, \tried), \stack)$
	\caption{Returns the next step in a method instance $m$ for a  stack \stack and updates \stack.}
\end{algorithm}

$\Next(\stack,\xi)$ advances within the body of the topmost method instance $m$ in \stack as well as with respect to \stack. If $i$ is the last step in the body of  $m$, the current tuple is removed from  \stack: method instance $m$ has successfully addressed $\tau$. If $\tau$ is a root task; \Next and \textsf{Progress} return $\emptyset$, meaning that $\tau$ succeeded; its stack \stack is removed from the \textit{Agenda}. If $i$ is not the last step of $m$,  \RAE proceeds to the next step $j$. Normally $j$ is the next line after $i$, but if that line is a control instruction (e.g., an \textit{if} or \textit{while}) then $j$ is the step to which the control instruction directs us (which of course may depend on the current state $\xi$).

Starting from line \ref{next} in \sv{Progress}, $i$ points to the next line of $m$ to be processed.
If $m[i]$ is an assignment, the corresponding update of
$\xi$ if performed; \RAE proceeds with the next step. If $m[i]$ is an action $a$, its execution is triggered; \RAE will wait until $a$ finishes to examine the \Next step of $m$. If $m[i]$ is  a
task $\tau'$, a refinement with a method instance $m'$, returned by \sv{Select}, is performed. The corresponding tuple is pushed on top of  \stack. If there is no applicable method instance to $\tau'$, then the current
method instance $m$ failed to accomplish $\tau$, a \textsf{Retry} with other method instances is performed.

\begin{algorithm}[!h]  \label{alg:retry} \DontPrintSemicolon
     \uline{$\Retry(\stack)$:} \;
	$(\tau,m, \step, \tried)\gets$ \sv{pop}($\stack$)\;
	$\tried \gets \tried \cup \{m\}$ 	\tcp*[F]{$m$ failed}
\nl		observe current state $\xi$  \label{obs3}\;
\nl 		$m' \gets \sv{Select}(\xi, \tau,\stack,d_{max}, \uctRollouts)$    \label{plan3}\;
\nl	\lIf{$m' \neq \emptyset$}
	{return \sv{push}$((\tau,m',1, \tried),\stack)$} \label{obs4}
	{\lElseIf {$\stack\neq \emptyset$}
		{return $\Retry(\stack)$ }
\nl		\lElse  {return \failure} \label{alg:fail2}
	}
\caption{\Retry examines untried alternative method instances, if any, and returns an updated stack.}
\end{algorithm}

\sv{Retry} (Algorithm \ref{alg:retry}) adds the failed method instance $m$ to the set of method instances that have been tried for $\tau$ and failed. It removes the corresponding tuple from  \stack.  It retries refining $\tau$ with another method instance $m'$ returned by \sv{Select} which has not been already tried (line \ref{obs4}). If there is no such $m'$ and if \stack is not empty, \sv{Retry} calls itself recursively on the
topmost stack element, which is the one that generated $\tau$ as a
subtask: retrial is performed one level up in the refinement tree. If  stack \stack is empty, then $\tau$ is the root task or event: \RAE failed to accomplish $\tau$. 

\RAE fails either (i) when there is no method instance applicable to the root task in the current state (line \ref{failure0} of \RAE), or (ii) when all applicable method instances have been tried and failed (line \ref{failure1}). A method instance fails either (i) when one of its actions fails (line \ref{alg:fail1} in \sv{Progress})
or (ii) when all applicable method instances for one of its subtasks have been tried and failed (line \ref{alg:fail2} in \sv{Retry}).

Note that \textsf{Retry} is not a backtracking procedure: it does not go back to a previous \textit{computational node} to pick up another option among the candidates that \textit{were} applicable when that node was
first reached. It finds another method instance among those that are \textit{now} applicable for the \textit{current} state of the world $\xi$.  \RAE interacts with a dynamic world: it cannot rely on the set \sv{Applicable}$(\xi,\tau)$ computed earlier, because $\xi$ has changed, new method instances may be applicable. However, the same method instance that failed at some point may succeed later on and may merit retrials. We discuss this issue in  \autoref{sec:discussion}.

\newcommand{\eval}{\sv{eMEU}\xspace} 
\newcommand{\selectMethod}{\sv{UPOM}\xspace} 
\newcommand{\evalUCT}{\sv{eMEU$_{UCT}$}\xspace} 
\newcommand{\evalUCTHelper}{\sv{eMEU$_{UCT1}$}\xspace} 

\newcommand{\Uf}{U_{\sv{Failure}}}
\newcommand{\Us}{U_{\sv{Success}}}

\newcommand{\PLANH}{\sv{U1}\xspace}
\newcommand{\RPLAN}{\small \sv{RAEplan}\xspace}
\newcommand{\one}{\mathbb{I}} 

\section{Planning for \RAE}
\label{sec:plan}

In \autoref{sec:operational}, we informally defined the deliberative acting problem as the problem of selecting the ``best'' method instance $m \in \mathcal{M}$ to perform $\tau$ in a current state $\xi$ for a domain \mbox{$\Sigma=(\Xi, \mathcal{T, M, A})$}.
A {\em refinement planning domain} is a tuple $\Phi = (S,\mathcal{T},\M,\A)$, where
$S$ is the set of states that are abstractions of states in $\Xi$, 
and $\Tasks$,  $\M$, and $\A$ are the same as in $\Sigma$.

Recall that if \RAE is run purely reactively, \sv{Select} chooses a refinement method instance from a predefined order of refinement methods, without comparing alternative options in the current context.
 In this section, we define a utility function to assess and compare method instances in \sv{Applicable}$(\xi,\tau)$
to select the best one. This utility function might, in principle, be used by an exact optimization procedure for finding the optimal method instance for a task. We propose a more efficient Monte Carlo Tree Search approach for finding an approximately optimal method instance. The planner relies on a procedure, called \PLAN,  inspired from the Upper Confidence bounds search applied to Trees (UCT). \PLAN (\textit{UCT Procedure for Operational Models}) is parameterized for rollout depth $d$ and number of rollouts, $\uctRollouts$. It relies on a heuristic function $h$ for estimating the criterion at the end of the rollouts when $d < \infty$. 

The proposed approach runs multiple simulations using the method instances and  a \textit{generative sampling model} of actions. This model is defined as a function \sv{Sample}: $S \times \mathcal{A} \to S$. \sv{Sample}$(s,a)$ returns a state $s'$ randomly drawn from $\gamma(s,a)$, with  $\gamma: S \times A \to 2^{S} \cup \{ \textsf{failed} \}$. The transition function $\gamma$ is augmented with the token \textsf{failed}
to account for possible failures of $a$. We assume, as usual, that the sampling reflects the probability distribution of the action's real-world outcomes.

A simulation of a method instance $m$ for a task $\tau$ during planning goes successively through the steps of $m$, as required by the control flow for the current context, and generates a sequence of \textit{simulated states} $\langle s_0,\ldots, s_i,\ldots\rangle$, where initially $s_0$ corresponds to an abstraction of the current real world state $\xi$. 
For instance in Example~\ref{ex:ee1} this involves simulating for the method \sv{m1-survey}$(l, r)$ several \sv{moveTo}$(r, l')$ tasks, followed by \textsc{DetectPerson}$(r, \textit{cam})$ action and \sv{rescue}$(r, l')$ tasks depending on the current context.
The utility function is computed along such a sequence, taking into account the \textit{deterministic} refinements of method instances and the \textit{nondeterministic} outcomes of actions (see \autoref{fig:tree}). 
Simulation during planning does not \textsf{Retry}, as in \RAE, but it takes into account possible failures. Further, we do not observe in \UPOM nor consider possible changes in $\xi$ during a simulation. These changes, when leading to events, are dealt with at the acting level through the main loop of \RAE  which remains concurrently active during planning. 

\subsection{Utility criteria and optimal approach}
\label{sec:utilities}

The  appropriate utility function can be application dependent. One may consider a function combining rewards for desirable or undesirable states,  and costs for the time and resources of actions. To keep the formal presentation simple, we assume that there are no rewards in states. We studied two utility functions measuring respectively the actor's efficiency and robustness.  Regarding the former, instead of minimizing costs, the efficiency utility function maximizes values to easily account for failures. For the latter, the actor seeks a method instance that has a good chance to succeed. 

We first define two value functions for actions, $v_e$ and $v_s$, which lead to the two proposed utility functions for method instances.

\paragraph{Efficiency}
 Let $\sv{Cost}: S\times \mathcal{A}\times (S \cup \{\textsf{failed}\}) \to \R^+$ be a cost function. 
 \sv{Cost}$(s,a,s')$ is the cost of performing action $a$ in state $s$ when the outcome is $s'$. Note that the cost of an action $a$ is finite even when $a$ fails. This is the case since in general an actor is able to figure out that an  attempted action failed to limit its cost. However, a failed action $a$ in a method instance $m$ leads to the failure of $m$; its eficiency is simply 0. Hence we define the efficiency value of an action as follows:
\begin{equation}\label{eq:ve}
v_e(s,a,s')= 
\begin{cases}
0
& \text{if }  s'=\text{ ``\textsf{failed}''},\\

1/\sv{Cost}(s,a,s')
&\text{otherwise}.
\end{cases} 
\end{equation}
If we let $v_{e1} \oplus v_{e2}$ denote
the  cumulative efficiency value of two successive actions whose efficiency values are  $v_{e1} = 1/c_1$ and $v_{e2} = 1/c_2$, then
\begin{equation}
\label{eq:eplus}
\textstyle
v_{e1} \oplus v_{e2} = 1/(c_1+c_2)
= 1/\left(\frac{1}{v_{e1}}+\frac{1}{v_{e2}}\right)
 = v_{e1} \times v_{e2}/ (v_{e1} + v_{e2}).
\end{equation}

\paragraph{Success Ratio}
Here, we measure the utility of a method instance as its probability of success over all possible outcomes of its actions. Hence we simply take a value 0 for an action that fails, and 1 if the action succeeds.
\begin{equation}\label{eq:vs}
v_s(s,a,s')= 
\begin{cases}
0
& \text{if }  s'=\text{ ``\textsf{failed}''},\\

1
&\text{otherwise}.
\end{cases} 
\end{equation}
If we let $v_{s1} \oplus v_{s2}$ denote the  cumulative success ratio for two successive actions in a method instance whose success ratios are  $v_{s1}$ and $v_{s2}$,
then
\begin{equation}
\label{eq:splus}
v_{s1} \oplus v_{s2} = v_{s1} \times v_{s2}.
\end{equation}
\phantomsection \label{def:one}For both value functions $v_e$ and $v_s$, the operator $\oplus$ is associative, which is needed for combining successive steps. For both value functions, we let $\one$ denote the identity element for  operation $\oplus$, i.e., the element such that $x\oplus\one=\one\oplus x=x$:
\begin{LIST}
\item
For $v_e$ in \autoref{eq:ve}, we have $\one=\infty$, corresponding to a cost of $1/\one = 0$. If $v_{e1} = \one$, then $v_{e1} \oplus v_{e2} = 1/(0 + \frac{1}{v_{e2}}) = v_{e2}$
for every $v_{e2}$.
\item
For $v_s$ in  \autoref{eq:vs}, we have $\one=1$, corresponding to success (task is already accomplished).  
\end{LIST}
Note that if either of two actions  in a method instance $m$ fails, their combined value is 0,  since $m$ also fails.

\medskip

Let us now define a utility function for method instances using  either $v_e$ or $v_s$. In order to compute the expected utility of 
a method instance $m$ we need to consider possible \textit{traces} of the execution of $m$ for a task $\tau$. In \RAE, an execution trace was conveniently represented though the evolution of  \stack for the task $\tau$. In planning, we similarly use  \stack as a LIFO list of tuples $(\tau,m,i,\tried)$, as defined in \RAE.\footnote{We do not need for the moment to keep track of already \tried method instances, but we'll see in a moment the usefulness of this term} 
For a given simulation of $m$ for $\tau$, \stack is initialized as a copy of the current stack in \RAE.
We progress in the simulation of $m$ step by step using the function \Next (Algorithm \ref{alg:next}), pushing in $\stack$ a new tuple when a step requires a refinement into a subtask.

Let $\textsf{top}(\stack)$ be the stack tuple $(\tau,m,i,\tried)$. The utility of a particular simulation of $i^{th}$ step of $m$ for $\tau$ is given by the following recursive equation:

\begin{equation}\label{eq:um}
U(m,s,\stack)= 
\begin{cases}
U(m,s',\Next(\stack,s))
& \text{if }  m[i] \text{ is an assignment,} \\

v(s,a,s')\oplus U(m,s',\Next(\stack,s))
&\text{if }  m[i] \text{ is an action } a,\\

U(m',s,\textsf{push}((\tau',m',1,\nullset),\Next(\stack,s))
&\text{if }  m[i] \text{ is a subtask } \tau',\\

\one
&$if $ \stack=\nullset,
\end{cases} 
\end{equation}
where $v$ is either $v_e$ or $v_s$. An assignment step changes the state from $s$ to $s'$ but does not change the utility $U$. An action $a$ changes the state nondeterministically to $s'$; the utility is the combined value of $a$ and the utility of the remaining step. A refinement step does not change the state; it is addressed in this particular simulation by refining 
$\tau$ into $\tau'$ with $m'$.
The function \Next moves to the following step, and to the empty stack at the end of every simulated execution. 

From \autoref{eq:um} we derive the \textit{maximal expected utility} of $m$ for $\tau$ by maximizing recursively over all possible refinements in $m$ and averaging over all possible outcomes of actions, including failures:

\begin{equation}\label{eq:um*}
U^*(m,s,\stack)= 
\left\{
\begin{array}{l}

\makebox[5cm][l]{$U^*(m,s',\Next(\stack,s))$} 
\text{if }  m[i] \text{ is an assignment,}\\[1ex]

\sum_{s' \in \gamma(s,a)}\Pr(s'|s,a) [v(s,a,s')\oplus U^*(m,s',\Next(\stack,s))]\\
\makebox[5cm][l]{}
\text{if }  m[i] \text{ is an action } a,\\[1ex]

$max$_{m'\in \sv{Applicable}(s,\tau')}U^*(m',s,\textsf{push}((\tau',m',1),\Next(\stack,s))\\
\makebox[5cm][l]{} 
\text{if }  m[i] \text{ is a  subtask } \tau',\\[1ex]

\makebox[5cm][l]{$\one$}
$if $ \stack=\nullset.
\end{array} 
\right.
\end{equation}

In the above equation, $\gamma(s,a)$ includes the token ``\failed''. We assume as usual that if  $\sv{Applicable}(s,\tau)=\nullset$ then
max$_{m\in \sv{Applicable}(s,\tau)}U^*(m,s,\stack)=0$, 
meaning a refinement failure. Instantiating $v$ as either $v_e$ or $v_s$ gives the two utility functions, the efficiency and the success ratio of method instances, respectively.

The optimal method instance for a task $\tau$ in a state $s$ for the utility $U^*$ is:

\begin{equation}\label{eq:m2*}
m^*_{\tau,s}= \text{argmax}_{m\in \sv{Applicable}(s,\tau)}U^*(m,s,\langle(\tau,m,1,\nullset)\rangle)
\end{equation}

\medskip

It is possible to implement \autoref{eq:um*} directly as a recursive backtracking optimization algorithm and to make the planning algorithm return $m^*_{\tau,s}$, as defined above. However, 
this would be too computationally demanding and not practical for an online planner. We propose instead to seek an approximately optimal method instance with an anytime controllable procedure using a Monte Carlo Tree Search algorithm in the space of operational models. 

\subsection{A planning algorithm based on UCT}

\phantomsection \label{def:depth}To find an approximation $\tilde{m}$ of $m^*$, we propose a progressive deepening Monte Carlo Tree Search procedure with $n_{ro}$ rollouts, down to a depth $d_{max}$ in the refinement tree of a task $\tau$. 
A rollout in MCTS is an exploration of a path along a random branch from each nondeterministic node (i.e., an outcome of an action) down to depth $d_{max}$.
The basic ideas of UPOM are the following:
\begin{LIST}
\item at  an action node of the search tree, we average over the values of the corresponding $n_{ro}$  rollouts;
\item at  a task node, we choose the refinement method instance with the highest expected utility;
\item starting from $d=d_{\max}$,we decrease $d$ for a refinement step and an action step, but not in an assignment step;

\item we take a heuristic estimate of the utility of the remaining refinements  at the tip of a rollout, i.e., at $d=0$;
\item we stop a rollout at a failure of an action or a refinement, and return a value $\Uf=0$; we also stop when the stack is empty and return $\Us= \one$. 

\end{LIST}

\newcommand{\mchosen}{m_c\xspace}

\begin{algorithm}[!ht]   \DontPrintSemicolon
	\uline{$\sv{Select}(\xi,\tau, \stack, d_{max}, \uctRollouts)$:} \;
		$(\tau, m,i,\tried) \gets \sv{top}(\stack)$\;$M\gets \sv{Applicable}(\xi,\tau)\setminus\tried $ \;
	\lIf{$M=\nullset$}{return $\nullset$}
	\lIf{$|M=\{m\}|=1$}{return $m$}
	$s \gets \sv{Abstract}(\xi)$ ; $\sigma \gets$ copy of \stack  ; $d\gets 0$\;
	\nl \label{slct:1} $\tilde{m} \gets \text{argmax}_{m \in M}h(\tau,m, s)$ \tcp*[F]{initialize $\tilde{m}$}
\nl \label{slct:3}
	\Repeat{$d=d_{max}$ or \sv{search time is over}}
	{$d \gets d+1$ \tcp*[F]{progressive deepening}
\nl\label{slct:nro}	\For{ $\uctRollouts$ times}
	{
		\PLAN($s$, $\sv{push}((\tau, nil, 1, \nullset), \stack$), $d$)
	}
	 $\tilde{m} \gets \argmax_{m \in M} {Q_{\stack,s}(m)}$
	}
	return $\tilde{m}$
\caption{A progressive deepening procedure using \PLAN  for finding an approximately optimal method instance.}\label{alg:oracle}
\end{algorithm}
	
\begin{algorithm}[!ht]  \DontPrintSemicolon
	\uline{$\PLAN(s, \stack, d)$:} \;
	\lIf{$\stack=\la \ra$}{return $\Us$} 

	$(\tau, m,i,\tried) \gets \sv{top}(\stack)$\;
	\nl \lIf{$d=0$}{return $h(\tau,m, s)$}  \label{uct:heuristic}

	\If { $m = nil$ or $m[i]$ is a task $\tau'$}
	{
		\lIf {$m = nil$} {$\tau' \leftarrow \tau$ } 
		\If { $N_{\stack,s}(\tau')$ is not initialized yet} 
		{\nl \label{uct:mp}$M' \leftarrow \sv{Applicable}(s, \tau')\setminus\tried$\;
			\lIf {$M' = 0$} {return $\Uf$} 
			$N_{\stack, s}(\tau') \leftarrow 0$ \;
			\For { $m' \in M'$}
			{
				$N_{\stack, s}(m') \leftarrow 0$ ; 
				$Q_{\stack, s}(m') \leftarrow 0$\;
			}
		}
		\textit{Untried} $\leftarrow \{ m' \in M' | N_{\stack, s}(m') = 0 \}$\;
		
		\If {\textit{Untried} $ \neq \nullset$}
		{\nl\label{uct:mc1}
			$\mchosen \leftarrow $ random selection from \textit{Untried}
		}	\nl \label{uct:ucb} 
		\lElse 
		{
			$\mchosen \leftarrow \argmax_{m\in M'} \{Q_{\stack, s}(m) +  C\times [\log{N_{\stack, s}(\tau)}/N_{\stack, s}(m)]^{1/2}\}$ 
		}
		
		\nl \label{uct:efftask} $\lambda \gets  \PLAN(s, \sv{push}((\tau', \mchosen,1,\nullset),\Next(\stack,s)), d-1)$\;
		
		\nl \label{uct:qupdate}
		$Q_{\stack, s}(\mchosen) \leftarrow [N_{\stack, s}(\mchosen) \times Q_{\stack, s}(\mchosen) + \lambda]/[1 + N_{\stack, s}({\mchosen})]$\;
		$N_{\stack, s}(\mchosen) \leftarrow N_{\stack, s}(\mchosen) + 1$\;
		return $\lambda$ \;
	}
	\If {$m[i]$ is an assignment}
	{$s' \gets$ state $s$ updated according to $m[i]$ \;
		return  $\PLAN(s', \Next(\stack,s'), d)$ } 
	
	\If {$m[i]$ is an action $a$}
	{\nl $s'\gets \sv{Sample}(s,a)$ \label{uct:sample}\;
		\lIf{$s'=\failed$}{return $\Uf$}	
		\nl
		\lElse{ 
			{ return $v(s,a,s')\oplus\PLAN(s', \Next(\stack, s'), d-1)$ }
		\label{uct:eff}}
	}
	\caption{Monte Carlo  tree search procedure \PLAN;  performs one rollout recursively down the refinement tree of  a method instance to compute an estimate of its optimal utility.
} 
	\label{alg:upom}
\end{algorithm}

This is detailed in  algorithms \ref{alg:oracle} and \ref{alg:upom}. \sv{Select} 
is called by \RAE with five parameters: $\xi$,  $\tau$, and $\stack$, and the  control parameters $d_{max}$, the maximum rollout depth, and  $n_{ro}$, the number of UCT rollouts. 
Recall that on a new root task $\tau$, \RAE  calls \sv{Select} with $\sigma=\la(\tau, nil, 1, \nullset)\ra$. 
\sv{Select} returns $\tilde{m}$, an approximately optimal method instance for $\tau$, or $\nullset$ if no method instance is found, i.e.,  if there is no applicable method instances for $\tau$ in $\xi$, but of those already tried by \RAE for this task. 
\textsf{Select} uses a copy of  \RAE's current stack $\stack$, and a simulation state $s$, which is an abstraction of the current execution state $\xi$ (e.g., in Example \ref{ex:ee1}, $l$ can be a precise metric location for acting and topological reference for planning).
It initializes $\tilde{m}$ with a heuristic estimates (line \ref{slct:1}).
It performs a succession of simulations at progressively deeper refinement levels using the function \PLAN  to evaluate the utility of a candidate method instance. The progressive deepening loop (line \ref{slct:3}) is pursued until reaching the maximum rollout depth, or until the actor interrupts the search because of time limit or any other reason, at which point the current $\tilde{m}$ is returned and will be tried by \RAE. \sv{Select} is an \textit{anytime} procedure: it returns a solution whenever interrupted. $Q_{\sigma,s}(m)$ is a global data structure that approximates the utility $U^*(m,s,\sigma)$.

\PLAN (\autoref{alg:upom}) takes as arguments a simulation state $s$, a stack $\stack$, and the rollout depth $d$. It performs one rollout over recursive calls for a method instance $m$ and its refinements. On the first call of a rollout, $m=nil$, meaning that no method instance has yet been chosen. A method instance $m_c$ is chosen among untried method instances (line \ref{uct:mc1}). If all method instances have been tried, $m_c$ is chosen (line \ref{uct:ucb}) according to a tradeoff between exploration and exploitation.
\phantomsection \label{def:C}The constant $C>0$ fixes this tradeoff for the exploration
less sampled method instances (high $C$) versus the exploitation or
more promising ones (low $C$).

\phantomsection \label{def:Q} $Q_{\sigma,s}(m)$ is calculated as follows. $Q_{\sigma,s}(m)$ combines the value of a sampled action with the utility of the remaining part of a rollout (line \ref{uct:eff}), and it updates $Q$ by averaging over previous rollouts (line \ref{uct:qupdate}). The value function $v$ (line \ref{uct:eff}) is either $v_e$ or $v_s$ depending on the chosen utility function, efficiency or success ratio. \phantomsection \label{def:usuf}For both function,  $\Us= \one$ and $\Uf= 0$.

\begin{figure}[!htbp]
\begin{center}
\includegraphics[width=.95\textwidth]{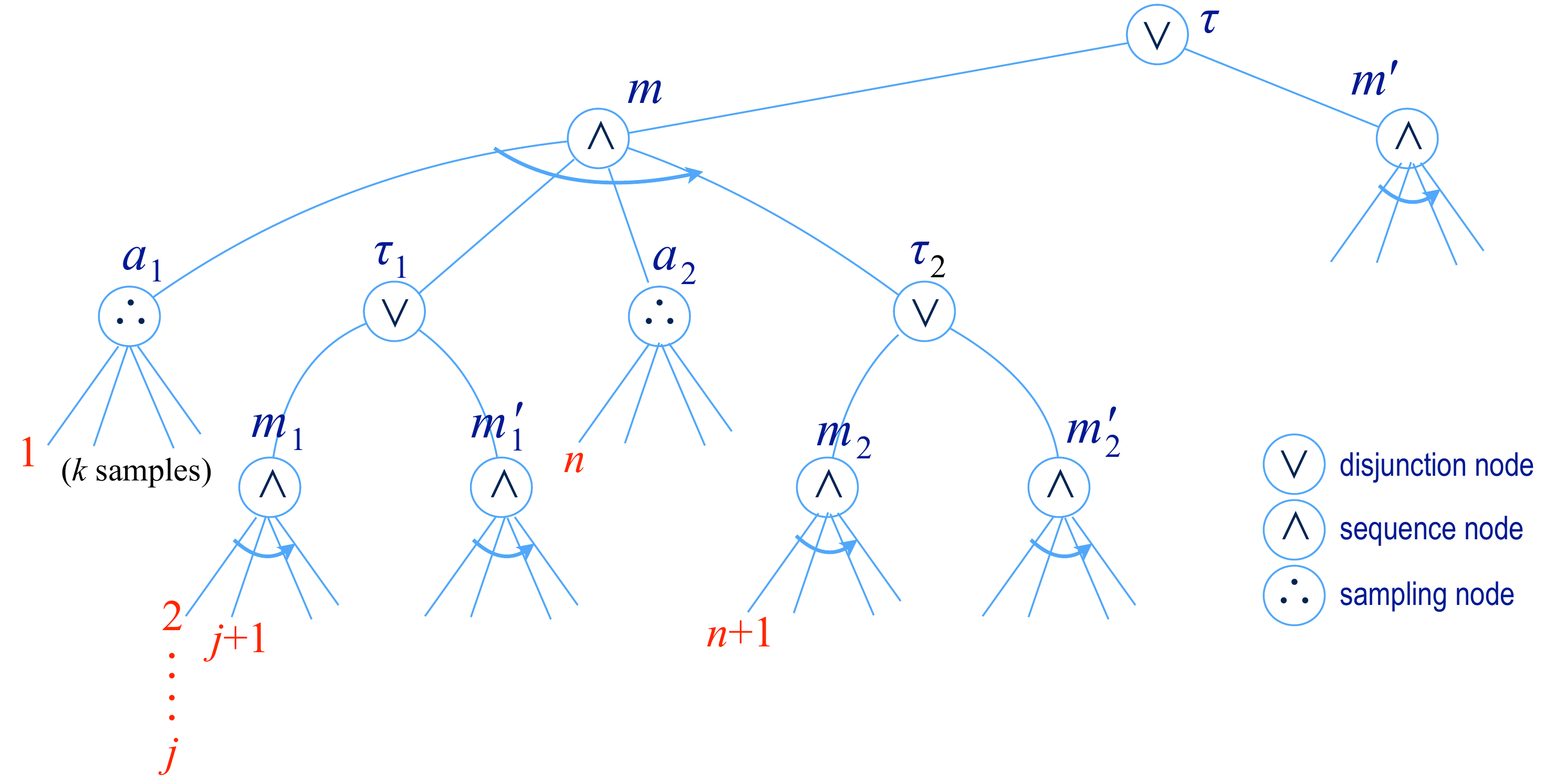}
\caption{A refinement tree, with three types of nodes: \textit{disjunction} for a task over possible method instances, \textit{sequence} for a method instances over all its steps, and \textit{sampling} for an action over its possible outcomes. A rollout can be, for example, the sequence of nodes marked 1 (a sample of $a_1$), 2 (first step of $m_1$), $\ldots, j$ (subsequent refinements), $j+1$ (next step of $m_1$), $\ldots, n$ (a sample of $a_2$), $n+1$ (first step of $m_2$), etc. }
\label{fig:tree}
\end{center}
\end{figure}

\paragraph{Differences from Equations \ref{eq:um*} and \ref{eq:m2*}}
A significant difference between the pseudocode in \autoref{alg:upom} and \autoref{eq:um*} is the restriction of \sv{Applicable} to method instances that have not been tried before by \RAE for the same task. 
This is a conservative strategy, because at this point the actor has no means for distinguishing  failures of tried method instances that require retrials from those that don't. We'll come back to a retrial strategy in \autoref{sec:conclusions}.

Another difference shows up in  the initialization of $\stack$ in \sv{Select}. This is explained by going back to how \sv{Select} is used by \RAE. At a root task $\tau$, when \sv{Select} is called the first time (line \ref{plan1} of \RAE), $\stack=\la(\tau,nil,1,\nullset)\ra$. If \RAE proceeds  for $\tau$ with a method instance $m$ returned by \sv{Select}, at the next refinement call of \RAE, e.g., for $\tau_1$ (see \autoref{fig:tree}), \sv{Select} needs to consider the utility of the method instances for $\tau_1$, but also their impact on the remaining steps in $m$, here on $a_2$ and  $\tau_2$. In other words, the actor requires the best method instance for $\tau_1$ in the context of its current execution state, taking into account the remaining steps of the method instance $m$ it is  executing. This best method instance for $\tau_1$ may be different from that given by \autoref{eq:m2*}. The need to keep track of previously tried method instances and pending tasks explains why $\stack$ is taken as a copy of the current \stack in \RAE for the root task at hand. However, this does not lead to reconsider previously made choices of method instances the actor is currently executing, e.g., in \autoref{fig:tree}, $m'$ is not reassessed.
Note that \PLAN does not pursue a rollout at an internal refinement node with the method instance maximizing the current utility evaluation $Q$, but with the best method instance according to the UCT exploration/exploitation tradeoff (line \ref{uct:ucb}). 

\paragraph{Asymptotic convergence}
In  \ref{app:mapping} we prove the asymptotic convergence of \PLAN towards an optimal method, i.e., as $n_{ro} \to \infty$ (\autoref{th:mapping}). The proof assumes no depth cut-off ($d_{max} = \infty$) and static domains, i.e., domains without exogenous events.\footnote{It should be possible to extend the proof to dynamic domains if there are known probability distributions over the occurrence of exogenous events.} It proceeds by mapping \PLAN's search strategy into UCT, which has been demonstrated to converge on a finite horizon MDP with a probability of not finding the optimal action at the root node that goes to zero at a polynomial rate as the number of rollouts grows to infinity (Theorem 6 of \cite{kocsis2006bandit}). To simplify the mapping, we first consider \PLAN with an additive utility function,
and show how to map \PLAN's search space into an MDP. We then discuss how this can be extended to  the efficiency and success ratio utility functions defined in \ref{sec:plan}, using the fact that the UCT algorithm is not restricted to the additive case; it still converges as long as the utility function is monotonic.

\paragraph{Control parameters}
\phantomsection \label{def:mu}The effects of the two control parameters $d_{max}$ and $n_{ro}$ are not independent.
This is because UCT  exploration examines an untried method instance before pursuing a rollout on an already tried one. 
Exploration would be complete if $n_{ro}>\mu$, where 
\[
\textstyle
\mu=\sum_{\tau_i \text{ is a subtask}}\text{max}_{s}|\sv{Applicable}(s,\tau_i)|,
\]
over all subtasks $\tau_i$ in the refinement tree (see Figure~\ref{fig:tree}), 
down to the refinement depth of the root task. But $\mu$ increases with $d_{max}$. In our experiments, we keep a large constant $n_{ro}$  and increase $d$ in the progressive deepening loop until the max depth $d_{max}$. An alternative  control of \sv{Select} can be the following:
\begin{LIST}
\item for a given $d$, pursue the rollouts (line \ref{slct:nro}) until there are $K$ successive exploitation rollouts, i.e.,  for which $Untried=\nullset$, for some constant $K$;\footnote{The probabilistic roadmap motion planning algorithm uses a similar idea to stop after $K$ configuration samples unsuccessful for augmenting the roadmap.}
\item pursue the progressive deepening loop (line \ref{slct:3}) until no subtask is left unrefined for the $K$ exploitation rollouts or until the search time is over. 
\end{LIST}
This is an adaptive control strategy that requires only two constants $C$ and $K$.

\paragraph{Search depth}
Finally, let us discuss the important issue of the depth cutoff strategy. Two options may be considered: \textit{(i)} $d$ is the number of steps of a rollout (as in MDP algorithms), or \textit{(ii)} $d$ is the refinement depth of a rollout.  The pseudocode in \autoref{alg:upom} takes the former option: $d$ decreases at every recursive call, for an action step as well as for a task refinement step. The advantage is that the cutoff at $d=0$ stops the current evaluation. The difficulty is that the root method instance, and possibly its refinements, are only partially evaluated. For example in  \autoref{fig:tree}, if $j>d_{max}$, steps $a_2$ and $\tau_2$ of $m$ will never be considered; similarly for the remaining steps in $m_1$:  rollouts will go in deep refinements and never assess all the steps of evaluated method instances. The value returned by \PLAN can be arbitrarily far from $U^*$. The other issue of this strategy is that the heuristic estimate  has to take into account remaining refinements lower down the cutoff point as well as remaining steps higher up in the refinement tree, i.e., what remains to be evaluated in $\stack$. 

In the alternative option where $d$ is the refinement depth of a rollout, $d$ decreases at  a task refinement step only, not at an action step. The advantage is to allow each rollout to go through all the steps of every developed method instance. Furthermore, the heuristic estimate at a cutoff is focused in this case on a subtask and its applicable method instances, whose simulation will not be started (nondeveloped method instances). The disadvantage is that one needs an estimate of the state following the achievement of a task with a nondeveloped method instance in order to pursue the sibling steps. In  \autoref{fig:tree} with $d=1$ for example, $\tau_1$ will not be refined; $a_2$ and remaining steps of $m$ will be based on an estimated state following the achievement of $\tau_1$. The definition of a default state change following a task is domain dependent and might not be easily specified in general.

The modifications needed in \PLAN to implement option \textit{(ii)} are as follows:
\begin{LIST}
\item In order to be able to go back to higher levels of $d$ when the simulation is pursued in parent method instances after a cutoff, it is convenient to maintain $d$ as part of the simulation stack: a fifth term $d$ is added in every tuple of $\stack$. 
\item The arguments of \PLAN are modified according to the previous point. 

\item \phantomsection \label{def:g}Line \ref{uct:heuristic} in \PLAN has to pursue the evaluation higher up in $\stack$:
\begin{quote}
\textbf{if} $d=0$ \textbf{then} return $h(\tau,m,s)\oplus\PLAN(g(s,\tau,m),pop(\stack),b,k)$,
\end{quote}
where $g(s,\tau,m)$ is a default state after the achievement of $\tau$ with $m$ in $s$.
\end{LIST}

For our experimental results (see \autoref{sec:implementation}), we have implemented a mixture of the two options: we take $d$ as the refinement steps of a rollout (decreasing $d$ at a task refinement step only), but we stop the evaluation when reaching $d=0$, taking heuristic estimates for the remaining steps of pending method instances. This has the disadvantage of a partial evaluation, but its advantages are to allow easily defined heuristic and not require a following state estimate.

\newcommand{\lm}{\sv{Learn$\pi$}\xspace} 
\newcommand{\lmi}{\sv{Learn$\pi_i$}\xspace} 
\newcommand{\lh}{\sv{LearnH}\xspace} 

\section{Learning for \RAE and \PLAN}
\label{sec:integration}

Purely reactive \RAE chooses a method instance for a task using an {\em a priori} ordering or a heuristic. \RAE with anytime receding-horizon planning uses \PLAN to find an approximately optimal method instance to refine a task or a subtask. At maximum rollout depth, \PLAN needs also heuristic estimates 

The classical techniques for domain independent heuristics in planning do not work for operational refinement models. Specifying by hand efficient domain-specific heuristics is not an acceptable solution. However, it is possible to learn such heuristics automatically by running \PLAN offline in simulation over numerous cases. For this work we relied on a neural network approach, using both linear and rectified linear unit (ReLU) layers.

We developed three learning procedures to guide \RAE and \PLAN:
\begin{LIST}
\item \lm learns a policy which maps a context defined by a task $\tau$, a state $s$, and a stack $\sigma$ to a refinement method $m$ in this context, to be chosen by \RAE when no planning can be performed. In Example~\ref{ex:ee1}, for the task \sv{getSupplies}, \lm is used to choose between \sv{m1-GetSupplies} and \sv{m2-GetSupplies} in the current context. 
\item \lmi learns the values of uninstantiated parameters of refinement method $m$ chosen by \lm.  In Example~\ref{ex:ee1}, for the method \sv{m1-survey}$(l, r)$, \lmi is used to choose the value of $r$. The value of $l$ comes from the argument $l$ in the task \sv{survey}$(l)$.
\item  \lh learns a heuristic evaluation function to be used by \PLAN. 
\end{LIST}

\subsection{Learning to choose methods (\lm)}
In a first approach, \lm learns a mapping from contexts to partially instantiated methods. A parameter of a method instance can inherit its value from the task at hand. However, different instances of a method may be applicable in a given state to the same task. This is illustrated in Example \ref{ex:ee1} by method  \sv{m1-survey}$(l,r)$  where $l$ is inherited from the task, but $r$ can be instantiated as any robot such that \sv{status}$(r)$ = \textit{free}.
\lm simplifies the learning by abstracting all these applicable method instances to a single class. To use the learned policy, \RAE chooses randomly among all applicable instances of the learned method for the context at hand. \lm learning procedure consists of the following four steps, which are schematically depicted in \autoref{fig:lm}. 

\begin{figure}[!h]
	\centering
	\includegraphics[width=0.7\columnwidth]{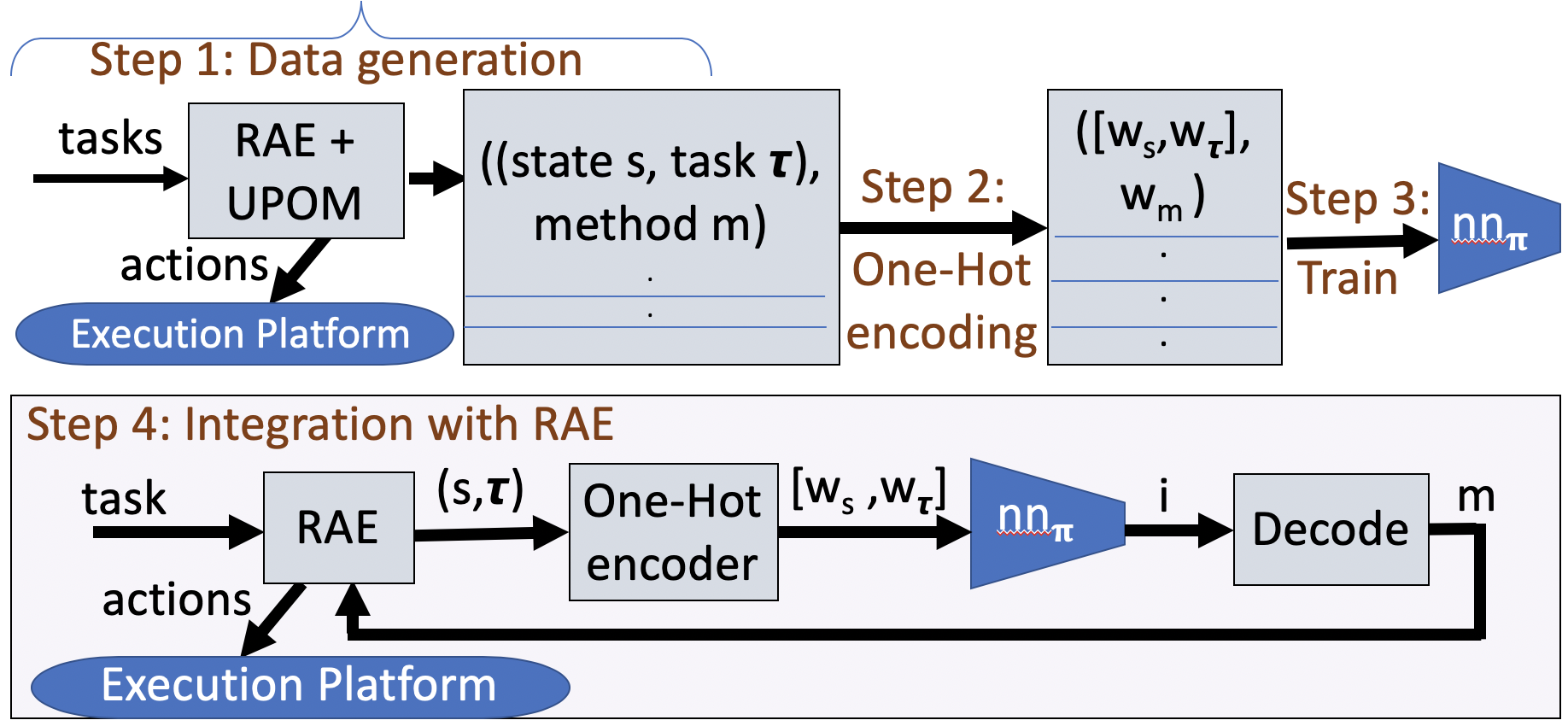}
	\caption{A schematic diagram for the \lm procedure.}
	\label{fig:lm}
\end{figure}

\paragraph{Step 1: Data generation}
\phantomsection \label{def:r}Training is performed on a set of data records of the form $r = ((s, \tau), m)$, where $s$ is a state, $\tau$ is a task to be refined and $m$ is a method for $\tau$. Data records are obtained by making  \RAE call the planner offline with randomly generated tasks.  Each call returns a method instance of the method $m$. We tested two approaches (the results of the tests are in  \autoref{sec:eval}):
\begin{LIST}
	\item $\lm$-1 adds $r = ((s, \tau), m)$ to the training set if \RAE succeeds with $m$ in accomplishing $\tau$ while acting in a dynamic environment.
	\item $\lm$-2 adds $r$ to the training set irrespective of whether $m$ succeeded during acting.
\end{LIST}

\paragraph{Step 2: Encoding}
The data records are encoded according to the usual requirements of neural net approaches. Given a record $r =((s, \tau),m)$, we encode $(s,\tau)$ into an input-feature vector and encode $m$ into an output label, with the refinement stack $\sigma$ omitted from the encoding for the sake of simplicity.\footnote{Technically, the choice of $m$ depends partly on $\sigma$. However, since $\sigma$ is a program execution stack, including it would greatly increase the input feature vector's complexity, and the neural network's size and complexity.}
Thus the encoding is
\begin{equation}
((s, \tau), m) \stackrel{\text{Encoding}}{\longmapsto} 
([w_s,w_\tau], w_m),
\label{eq:encoding}
\end{equation} 
with $w_s$, $w_\tau$ and $w_m$ being One-Hot representations of $s$, $\tau$, and $m$.
The encoding uses an $N$-dimensional One-Hot vector representation of each state variable, with $N$ being the maximum range of any state variable. \phantomsection \label{def:V}Thus if every $s \in S$ 
has $V$ state-variables, then $s$'s representation $w_s$ is $V\times N$ dimensional. 
Note that some information may be lost in this step due to discretization.

\paragraph{Step 3: Training}
\phantomsection \label{def:nnpi}Our multi-layer perceptron (MLP) $nn_\pi$ consists of two linear layers separated by a ReLU layer to account for non-linearity in our training data. 
To learn and classify $[w_s, w_\tau]$ by refinement methods,
we used a SGD (Stochastic Gradient Descent) optimizer and the Cross Entropy loss function.
The output of $nn_\pi$ is a vector of size $|\mathcal{M}|$ where $\mathcal{M}$ is the set of all refinement methods in a domain. 
Each dimension in the output represents the degree to which a specific method is optimal in accomplishing $\tau$.

\paragraph{Step 4: Integration in \RAE}
 \RAE uses the trained network $nn_\pi$ to choose a refinement method whenever a task or sub-task needs to be refined. Instead of calling the planner, \RAE encodes $(s, \tau)$ into $[w_s, w_\tau]$ using Equation~\ref{eq:encoding}. Then $m$ is chosen as
\begin{equation*}
m \leftarrow Decode(\argmax_i(nn_\pi([w_s, w_\tau])[i])),
\end{equation*}
where $Decode$ is a one-one mapping from an integer index to a refinement method.

\subsection{Learning to choose method instances (\lmi)}

Here, we extend the previous approach to learn a mapping from context to fully instantiated methods. The \lmi procedure learns over all the values of uninstantiated parameters  using a multi-layered perceptron (MLP).
We have a separate MLP for each uninstantiated parameter. 

\paragraph{Step 1: Data generation}
\phantomsection \label{def:vun}For each uninstantiated method parameter $v_{un}$, training is performed on a set of data records of the form $r = ((s, v_\tau), b)$, where $s$ is the current state, $v_\tau$ is a list of values of the task parameters, and $b$ is the value of the parameter $v_{un}$. Data records are obtained by making  \RAE call \PLAN offline with randomly generated tasks.  Each call returns a method instance $m$ and the value of its parameters. 

\paragraph{Step 2: Encoding}
Given a record $r =((s, v_\tau), b)$, we encode $(s,v_\tau)$ into an input-feature vector and encode $b$ into an output label.
Thus the encoding is
\begin{equation}
((s, v_\tau), b) \stackrel{\text{Encoding}}{\longmapsto} 
([w_s,w_{v_\tau}], w_b),
\label{eq:encoding_mi}
\end{equation} 
with $w_s$, $w_{v_\tau}$ and $w_b$ being One-Hot representations of $s$, $v_\tau$, and $b$.

\paragraph{Step 3: Training}
\phantomsection \label{def:nnvun}We train a multi-layered perceptron (MLP) for each uninstantiated task parameter $v_{un}$. Each such MLP  $nn_{v_{un}}$ consists of two linear layers separated by a ReLU layer to account for non-linearity in our training data. 
To learn and classify $[w_s, w_{v_\tau}]$ by the values of $v_{un}$,
we used a SGD (Stochastic Gradient Descent) optimizer and the Cross Entropy loss function.
The output of $nn_{v_{un}}$ is a vector of size  $|Range(v_{un})|$. Each dimension in the output represents the degree to which $v_{un}$ takes a specific value.

\paragraph{Step 4: Integration in \RAE}
After \RAE has chosen a refinement method $m$ for task $\tau$, we have \RAE use the trained network $nn_{v_{un}}$ to choose a value for each uninstantiated parameter $v_{un}$. \RAE encodes $(s, v_\tau)$ into $[w_s, w_{v_\tau}]$ using Equation~\ref{eq:encoding_mi}. Then, the value for $v_{un}$, $b$ is chosen as
\begin{equation*}
b \leftarrow Decode(\argmax_j(nn_{v_{un}}([w_s, w_{v_{un}}])[j])),
\end{equation*}
where $Decode$ is a one-one mapping from integer indices to $Range(v_{un})$. 

\smallskip
With \lmi, we choose the value of each uninstantiated parameter independently of the others, while remaining among currently applicable values (i.e., valid method instances). Certainly, parameters are not independent, and are not processed as such by RAE nor UPOM. This choice is simply \textit{relaxation assumption} in \lmi, which is common in the design of heuristics. Despite this relaxation, we found the guidance of \lm and \lmi to be quite effective when compared to reactive \RAE (see \cite[Figures 6 and 7]{patra2020integrating}). Furthermore, recall that \lmi is needed only when methods have more parameters than the task they address, and it is used only as first choice in progressive deepening, when there is no time for planning. 

\subsection{Learning a heuristic function (\lh)} 
\label{sec:lh}
The \lh procedure tries to learn an estimate of the utility $u$ 
of accomplishing a task $\tau$ with a method instance $m$ in state $s$.
One difficulty with this is that $u$ is a real number.
In principle, an MLP could learn the $u$ values using either regression or classification.
We chose to use classification.\footnote{
We chose classification because it outperformed regression in our preliminary experiments. 
It is possible that regression could be made to perform better, but that is beyond the scope of this paper.
}
We divided the range of utility values into $K$ intervals. By studying the range and distribution of utility values, we chose $K$ and the range of each interval such that the intervals contained approximately equal numbers of data records.
\lh learns to predict $interval(u)$, i.e., the interval in which $u$ lies.
The steps of \lh are the following (see \autoref{fig:lh}):

\begin{figure}[!h]
	\centering
	\includegraphics[width=0.7\columnwidth]{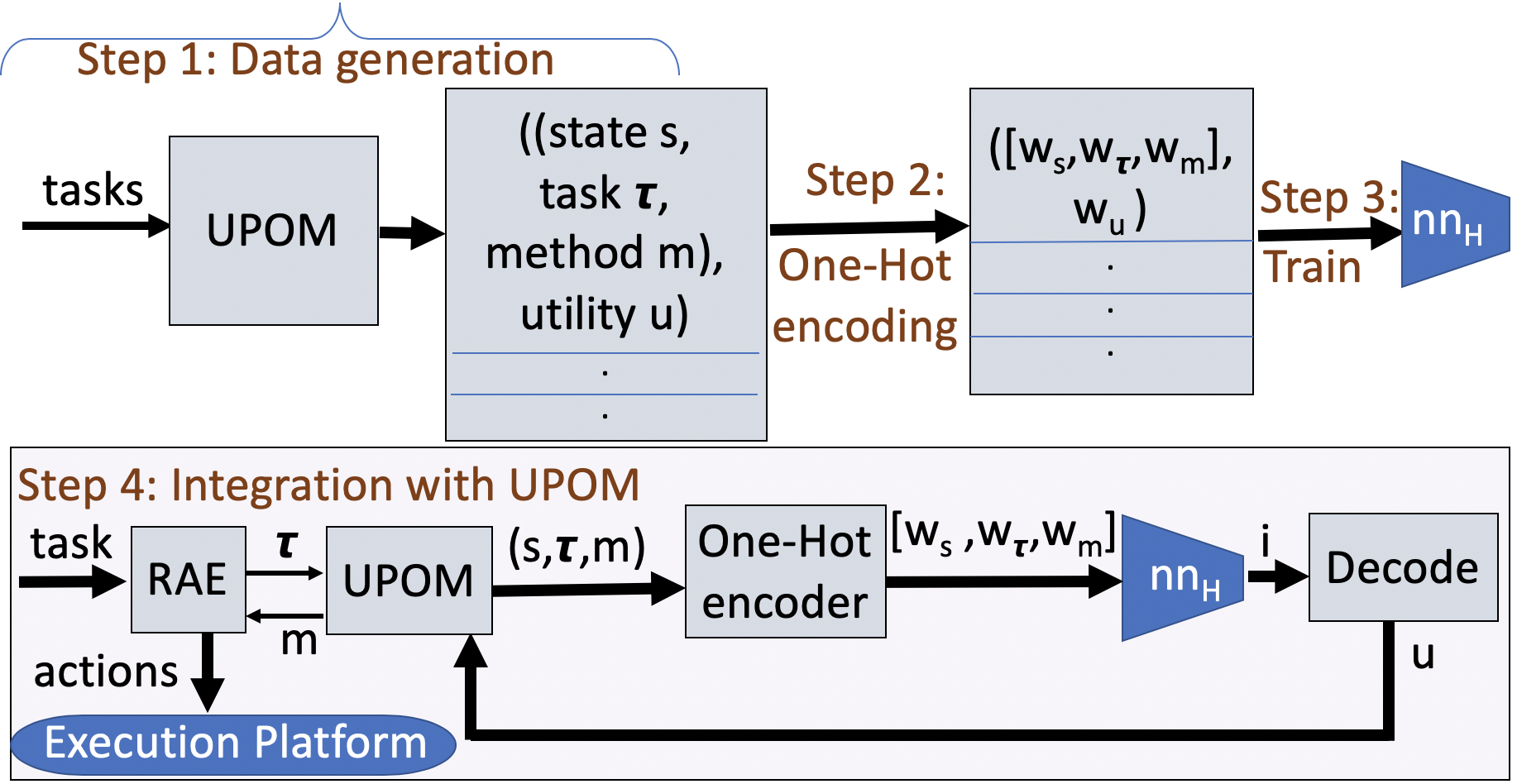}
	\caption{A schematic diagram for the \lh procedure.}
	\label{fig:lh}
\end{figure}

\paragraph{Step 1: Data generation}
We generate data records in a similar way as in the \lm procedure, with the difference that each record $r$ is of the form $((s, \tau, m), u)$ where $u$ is the estimated utility value calculated by \PLAN. 

\paragraph{Step 2: Encoding}
In a record $r =((s, \tau, m), u)$, we encode
$(s, \tau, m)$ into an input-feature vector using $N$-dimensional One-Hot vector representation, omitting $\sigma$ for the same reasons as before. 
If $interval(u)$ is as described above, then the encoding is
\begin{align}
((s, \tau, m), interval(u))
&\stackrel{\text{Encoding}}{\longmapsto} ([w_s,w_\tau, w_m], w_u)
\label{eq:encodinglh}
\end{align} 
with $w_s$, $w_\tau$, $w_m$ and $w_u$ being One-Hot representations of $s$, $\tau$, $m$ and $interval(u)$.

\paragraph{Step 3: Training}
\lh's MLP $nn_H$ is the same as \lm's, except for the output layer. $nn_H$ has a vector of size $K$ as output where $K$ is the number of intervals into which the utility values are split. Each dimension in the output of $nn_H$ represents the degree to which the estimated utility lies in that interval.

\paragraph{Step 4: Integration in \RAE}
\RAE  calls the planner with a limited rollout length $d$, giving \PLAN the following heuristic function to estimate a rollout's remaining utility:
\begin{equation}
\nonumber
h(\tau, m, s) \leftarrow Decode(\argmax_{i}(nn_H([w_s, w_\tau, w_m])[i])),
\end{equation}
where $[w_s,w_\tau, w_m]$ 
is the encoding of
$(\tau, m, s)$ using Equation \ref{eq:encodinglh}, and $Decode$ is a one-one mapping from a utility interval to its mid-point.
Before the progressive deepening loop over calls to \PLAN, \sv{Select} initializes $\tilde{m}$ in line~\ref{slct:1} according to this heuristic $h$.

\subsection{Incremental online learning}

The previous learning procedures rely on synthetic training data obtained from a collection of states and tasks. However, even if one is careful, the training data may not reflect an actor's specific working conditions. This is a well known important issue in machine learning. It can be addressed by continual incremental online learning. Here is a procedure to do so in our framework.
\begin{LIST}

\item  Initialization: either (i) without a heuristic by running \RAE{+}\PLAN online with $d_{max}=\infty$ , or (ii) with an initial heuristic obtained from offline learning on simulated data. 

\item Online acting, planning and incremental learning:
\begin{LIST}
	\item \phantomsection \label{def:Z} Augment the training set by recording successful methods (with the values of uninstantiated parameters) and $U$ values; train the models using \lm and \lh with $Z$ records, and then switch  \RAE to use either \lm alone when no search time is available, or \PLAN with current heuristic $h$ and finite $d_{max}$ when planning time available. 
	\item Repeat the above steps every $X$ runs (or on idle periods) using the most recent $Z$ training records (for $Z$ about a few thousands) to improve the learning on both \lh and \lm.
	
\end{LIST}
\end{LIST}

The issues for deploying this procedure in a practical application are discussed in \autoref{sec:discussion}.

\newcommand{\poweroften}[2]{$#1\times10^{#2}$}
\section{Experimental Evaluation}
\label{sec:implementation}\label{sec:eval}

\subsection{Domains}
We have implemented and tested our framework on five domains which illustrate  service and exploration robotics scenarios with aerial and ground robots. All the agents are under a centralized control. In these domains, acting is performed in a simulated environment. Consequently, we did not need to abstract the acting state $\xi$ into a planning state $s$; both are identical in our tests.

The \SR domain extends the search and rescue setting of Example~\ref{ex:ee1} with several UAVs surveying a partially mapped area and finding injured people in need of help. UGVs gather supplies, such as medicines, and go to rescue the localized persons. Exogenous events are weather conditions and debris in  paths. The complete set of tasks, commands, and refinement methods for the \SR domain is described in \ref{app:rescue}.

In \EE, several chargeable UGVs and UAVs explore a partially known terrain and gather information by surveying, screening, monitoring, e.g., for ecological studies. They need to go back to the base regularly to deposit data or to collect a specific equipment. Exogenous events are appearance of  animals in motion.

In \CR domain, several robots are collecting objects of interest. The robots are rechargeable and may carry the charger with them.
They can't know where objects are, unless they do a sensing action at the object's location. They must search for an object before collecting it. A task reaches a dead end if a robot is far away from the charger and runs out of charge. While collecting objects, robots may have to attend to some emergency events happening in certain locations.

The \SD domain has several robots trying to move objects from one room to another in an environment with a mixture of spring doors (which close unless they're held open) and ordinary doors. A robot can't simultaneously carry an object and hold a spring door open, so it must ask for help from another robot. A free robot can be the helper. The type of each door isn't known to the robots in advance. 

The \OF domain has several robots in a shipping warehouse that must co-operatively package incoming orders, i.e., lists of items of different types and weights to deliver to customers. Items for a single order have be placed in a machine, which packs them together; packages have to be placed in the shipping doc. To process multiple orders concurrently, items can be moved to a pallet before transfer to a machine.  Robots have limited capacities. 

\SR, \EE, \SD and \CR have sensing actions.
\SR , \EE, \CR and \OF can have dead-ends. The features of these domains are summarized in \autoref{fig:dom_prop},  \autoref{fig:dom_numbers}  and  \autoref{fig:searchSpaceSize}. Recall from Section~\ref{sec:operational} that $\mathcal{M}$  is the set of all refinement methods, $\overline{\mathcal{M}}$ is the set of all refinement method instances, and $\overline{\mathcal{T}}$ and $\overline{\mathcal{A}}$ are the sets of instances of tasks and actions. 
In  \autoref{fig:dom_numbers}, the upper bound on the size of the state space is $|S| < \prod_{x \in X}{|Range(x)|}$, since state variables are generally not independent.

\begin{table}[!ht]
	\centering
	\small
	\begin{tabular}{|@{~}c@{~~}|@{~~}c@{~~}c@{~~}c@{~~}c@{~~}c@{~}|}
		\hline
			&   Dynamic & Dead 	& Sensing	& Robot 		& Concurrent\\
		Domain  & events	& ends	&   	& collaboration & tasks	\\
		\hline
		\\[-2ex]
\SR & \checkmark & \checkmark & \checkmark & \checkmark & \checkmark \\			
\EE &  \checkmark & \checkmark & \checkmark & \checkmark & \checkmark\\
\CR & \checkmark & \checkmark & \checkmark & -- & \checkmark\\
\SD &  \checkmark & -- & \checkmark & \checkmark & \checkmark\\
\OF &  \checkmark & \checkmark  & -- & \checkmark & \checkmark\\
		\hline
	\end{tabular}\vspace{-1ex}
	
	\caption{Features of the test domains.}
	\label{fig:dom_prop}
\end{table}

\begin{table}[!ht]
	\centering
	\small
	\begin{tabular}{|c|c|c|c|c|c|c|c|}
		\hline
		\\[-2ex]
		Domain & Upper bound on $|\mathcal{S}|$  & $|\mathcal{T}|$ & $|\overline{\mathcal{T}}|$ & $|\mathcal{M}|$  & $|\overline{\mathcal{M}}|$ & $|\mathcal{A}|$	& $|\overline{\mathcal{A}}|$\\
		\hline
		\\[-2ex]
\SR & \poweroften{9.6}{16} & 8 & 2508 & 16 & 8766 & 14 & \poweroften{2.7}{6} \\			
\EE & \poweroften{1.93}{21} & 9 & 358 & 17 & 756 &  14 & 843 \\
\CR &  \poweroften{2.4}{17} & 6 & 270 & 10 & 282 & 9  & 463 \\
\SD & \poweroften{1.5}{11} & 6 & 192 & 9 & 651 & 10 & 490 \\
\OF &$\infty$& 6 & 64 & 6 & 318 & 9 & 4442 \\
		\hline
	\end{tabular}\vspace{-1ex}
	
	\caption{Sizes of the test domains. 
	$\mathcal{M}$  is the set of all refinement methods, $\overline{\mathcal{M}}$ is the set of all refinement method instances, and $\overline{\mathcal{T}}$ and $\overline{\mathcal{A}}$ are the sets of instances of tasks and actions. In the \OF domain, the number of states is infinite because
	some of the state variables may have real values, but the other numbers are finite because the
	task, methods, and instances do not have real-valued arguments.
	}
	\label{fig:dom_numbers}
\end{table}

\begin{table}[!ht]
	\centering
	\small
	\begin{tabular}{|@{~}c@{~~}|@{~~}c|c|c|c|c|c|}
		\hline
		Domain &  \multicolumn{2}{c|}{Rollout length} &  \multicolumn{4}{c|}{ Branching factor}   	\\\cline{1-7}
		 &  & &   \multicolumn{2}{c|}{Task nodes ($\vee$)} & \multicolumn{2}{c|}{Action nodes ($\therefore$)}    	\\\cline{4-7}
	& Avg & Max & Avg & Max & Avg & Max  \\

		\hline
\SR &  19 & 27 &  2.4 & 4 & 1.1 & 2  \\			
\EE &  20 &  112 &  1.5 & 3  & 1.2 & 2\\
\CR & 26 & 56 & 1.3 & 3  & 1.4 & 4\\
\SD &  30 & 78 &  2 & 6 & 1.2 & 2\\
\OF & 42  &  52 & 1.1  & 64 & 1.1  & 2\\
		\hline
	\end{tabular}\vspace{-1ex}
	
	\caption{Estimates of the search space parameters of the test domains.}
	\label{fig:searchSpaceSize}
\end{table}

\subsection{Planning parameters}
\label{sec:experiments}

Here we analyze the effect of the two planning parameters, $n_{ro}$ and $d_{max}$, on the two utility functions we considered, the efficiency, and the success ratio, as well as on the retry ratio of \RAE. We tested $n_{ro} \in [0, 1000 ] $ and $d_{max} \in [0,30]$. The case $n_{ro}=0$ rollout corresponds to purely reactive  \APE, without  planning. We only report for $n_{ro} \in [0, 250]$ since no significant additional effect was observed beyond $n_{ro}> 250$. We tested each domain on 50 randomly generated problems. A problem consists of one or two root tasks that arrive at  random time points in \APE's input stream, together with other randomly generated exogenous events. For each problem we recorded 50 runs to account for the nondeterministic effects of actions.
We measured the following:
\begin{LIST}
\item  the efficiency of \RAE for a task, i.e.,  the reciprocal of the sum of the costs of the actions executed by \RAE for accomplishing that task;
\item the success ratio of \RAE for a run, i.e., the number of successful tasks over the total of tasks for that run; and
\item the retry ratio of \RAE for a run, i.e., the number of call to \sv{Retry} over the total of  tasks for that run.
\end{LIST}

Since we are more concerned with the relative values than the absolute values of the efficiency, the success ratio, and the retry ratio, we rescaled in the following plots the $Y$ axis with respect to the base case for $n_{ro}=0$. 
Note that the measured efficiency takes into account the execution context with concurrent tasks and exogenous events; hence it is different for the corresponding utility function optimized in \PLAN (i.e., the expected efficiency of \autoref{eq:um*}); similarly for the success ratio.
We used a 2.8 GHz Intel Ivy Bridge processor.
The cut-off time for a run was set to 30 minutes.

\begin{figure}[!ht]
\centering 
	\includegraphics[width=0.9\columnwidth]{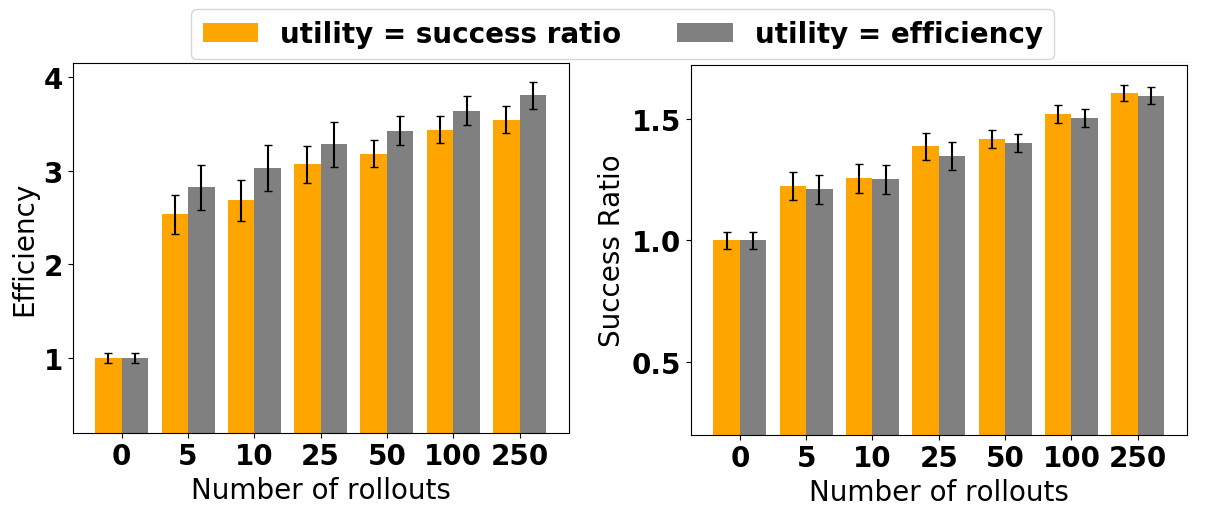}
	\caption{Efficiency and success ratio for two different utility functions (orange is expected success ratio and gray is expected efficiency) averaged over all five domains, with $d_{max} = \infty$. The $Y$ axis is rescaled with respect to the base case of $U$ for $n_{ro}=0$.}
	\label{fig:nu_sr_UCT_max_depth}
\end{figure}

\paragraph{Comparison of the two utility functions}
We studied two utility functions that are not totally independent but assess different criteria. The success ratio is useful as a measure of robustness. Suppose method instance $m_1$ is always successful but has a large cost, whereas $m_2$ sometimes fails but costs very little when it works: $m_1$ has a higher success ratio, but $m_2$ has higher expected efficiency.

\autoref{fig:nu_sr_UCT_max_depth} shows the measured efficiency and success ratio of \RAE for the two utility functions, averaged over all domains. Each  data point is the average of $12,500$ runs, with the error bars showing 95\% confidence interval; we plot relative values with respect the base case of $U$ for $n_{ro}=0$. As expected, the measured efficiency is higher when the optimized utility function  of \PLAN is the expected efficiency. Similarly for the success ratio. However, optimizing one criteria has also a good effect on the other one, since the two are not independent.
We also observe that 5 rollouts have already a significant effect on the efficiency, with slight improvements as \PLAN does more rollouts. In contrast, the success-ratio increases smoothly from no planning to planning with 250 rollouts. This can be due to the difference between the two criteria: a task that succeeds in its first attempt and a task that succeeds after several retries of \RAE have both a success-ratio of 1, but the efficiency in the latter case is lower. This point is analyzed next.

\paragraph{Retry ratio}
\autoref{fig:rr_de_UCT_max_depth} shows the {\em retry ratio}, i.e., the number of calls to \Retry, divided by the total number of tasks.  Recall that calling \Retry is how \RAE faces the failure of chosen methods; \RAE may succeed but after numerous retrials, which is not desirable. Although the retry ratio criteria is not independent from the combined two utility functions, this ratio depicts very clearly how effective the guidance of \RAE is. We observe that the retry ratio drops sharply from purely reactive \RAE to calling \PLAN with 5 rollouts. From then onwards, until 250 rollouts, the retry ratio continues to decrease gradually. The behavior is similar in all domains, so we have combined the results together to show the rescaled average values in a single plot.

\begin{figure}[!ht]
	\centering 
	\includegraphics[width=0.5\columnwidth]{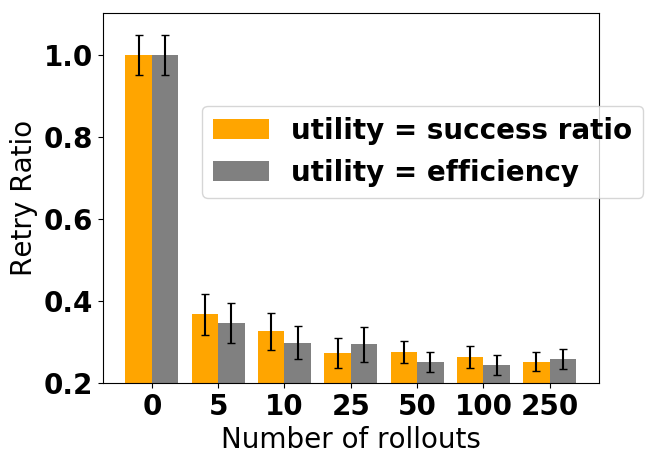}
	\caption{Retry ratio (\# of retries / total \# of incoming tasks)  averaged over all five domains, for \UPOM with $d_{max} = \infty$; $Y$ axis  rescaled with respect to the base case for $n_{ro}=0$.}
	\label{fig:rr_de_UCT_max_depth}
\end{figure}

\paragraph{Efficiency across domains}
In \autoref{fig:nu_de_UCT_max_depth} we detail for each domain the measured efficiency of \RAE when the utility of \PLAN was set to expected efficiency, for varying $n_{ro}$ and $d_{max} = \infty$. Each data point is the average of $2500$ runs. We observe that the efficiency generally improves with  the number of rollouts. However, there is not much improvement with increase in $n_{ro}$ in the \CR domain, and in the \OF domain, the efficiency drops slightly when $n_{ro} = 250$. We conjectured that this can be due to concurrent interfering tasks. Hence, we measured for  \CR and \OF domains the efficiency for test cases with only one root task; the results in \autoref{fig:nu_de_UCT_max_depth_one_task} confirmed this conjecture.

\begin{figure}[!ht]
\centering 
	\includegraphics[width=0.33\columnwidth]{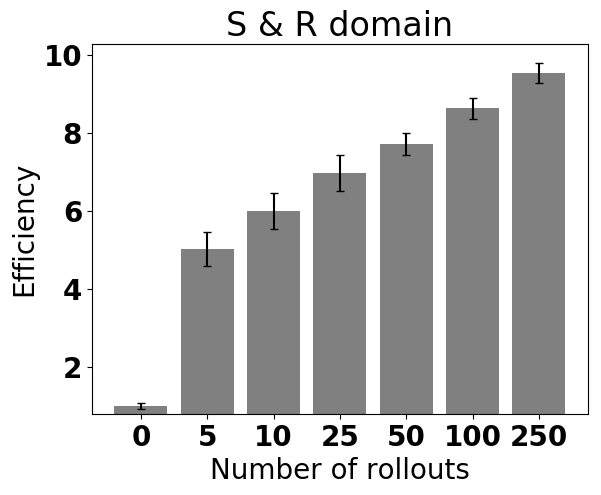}%
\hspace{\fill}\includegraphics[width=0.33\columnwidth]{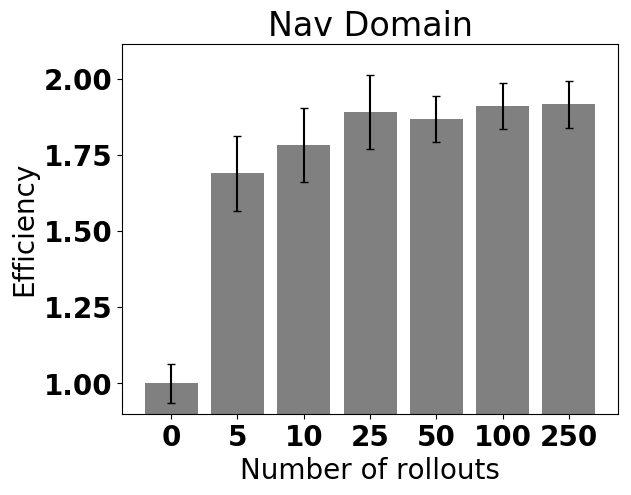}%
\hspace{\fill}\includegraphics[width=0.33\columnwidth]{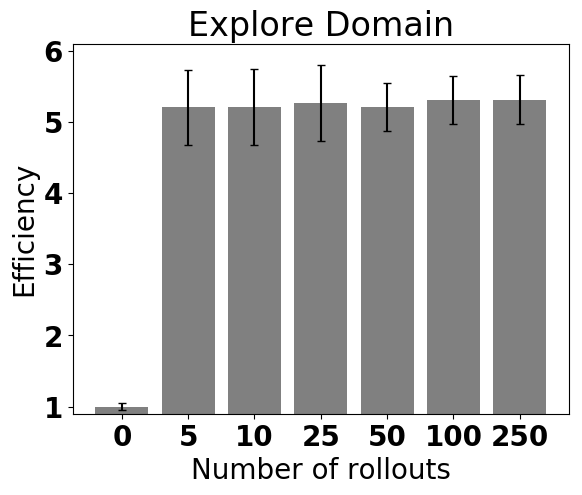}
	
		\includegraphics[width=0.33\columnwidth]{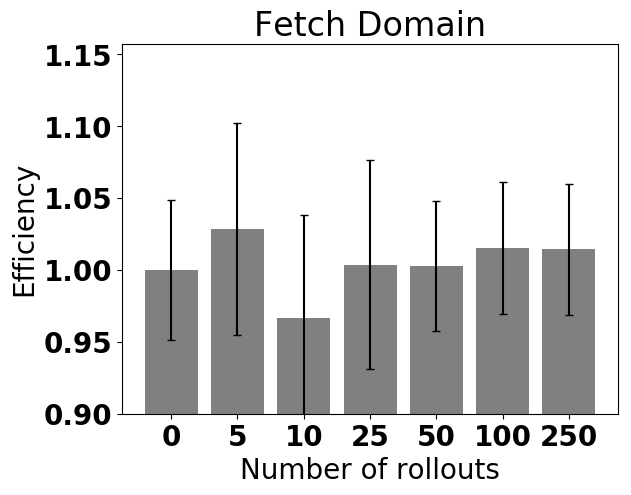}%
\hspace{\fill}\includegraphics[width=0.33\columnwidth]{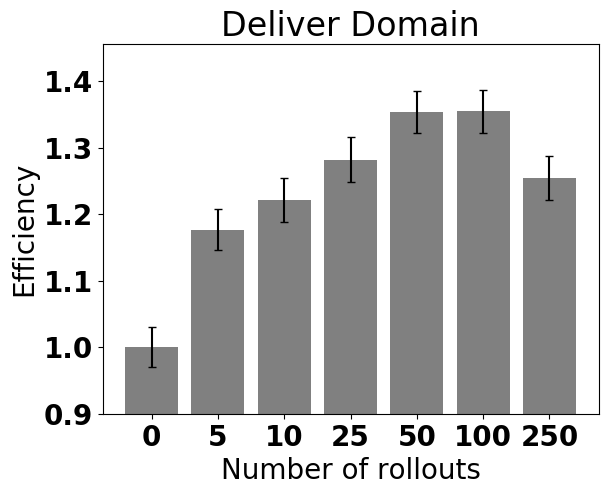}%
\hspace{\fill}\includegraphics[width=0.33\columnwidth]{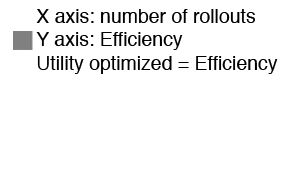}
	\caption{Measured efficiency of \RAE for  $n_{ro} \in [0, 250]$ and $d_{max} = \infty$; $Y$ axis  rescaled with respect to the base case for $n_{ro}=0$).}
	\label{fig:nu_de_UCT_max_depth}
\end{figure}

\begin{figure}[!ht]
\centering 
	\includegraphics[width=0.33\columnwidth]{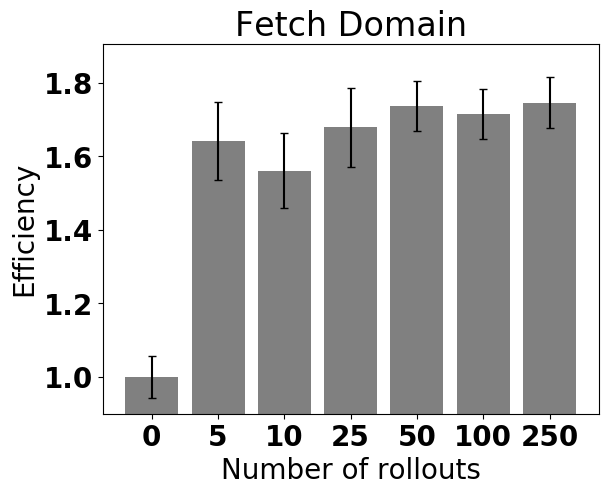}
	\includegraphics[width=0.33\columnwidth]{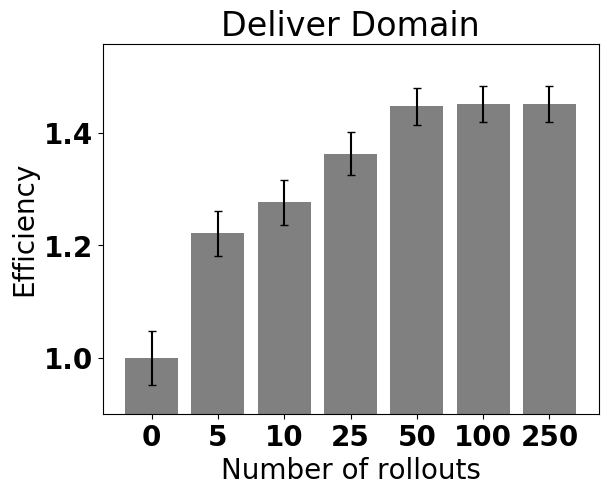}

	\caption{Measured efficiency averaged over only test cases with one root task, in \CR and \OF domains with $d_{max} = \infty$; $Y$ axis  rescaled with respect to the base case for $n_{ro}=0$.}
	\label{fig:nu_de_UCT_max_depth_one_task}
\end{figure}

\paragraph{Success ratio across domains} 
\autoref{fig:sr_de_UCT_max_depth} shows for each domain the measured success ratio of \RAE when the utility of \PLAN was set to expected success ratio, for varying $n_{ro}$ and $d_{max} = \infty$.
The success-ratio generally increases with increase in the number of rollouts. Again, a slight drop is observed in the \OF domain. \autoref{fig:sr_one_task} shows that for test cases with only one root task the success-ratio improves in the \CR domain, and remains constant in the \OF domain. 
The success ratio remains 1 in the \OF domain because all test cases with one root task succeed eventually, with or without retries. In the domains with dead ends, the improvement in success ratio is more substantial than domains without dead ends because planning is more critical for cases where one bad choice of refinement method instance can lead to permanent failure.

\begin{figure}[t]
	\centering 

	\includegraphics[width=0.33\columnwidth]{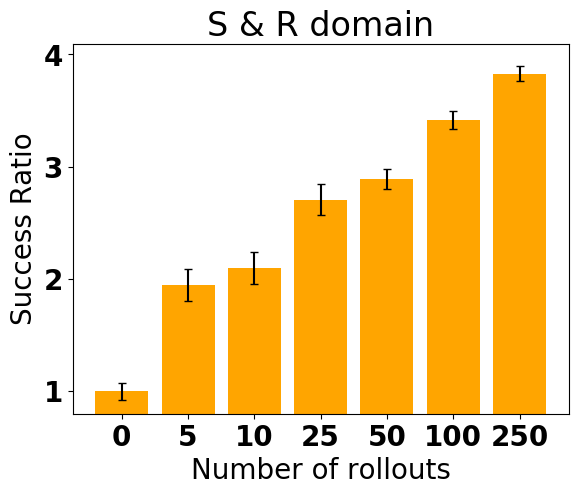}%
\hspace{\fill}\includegraphics[width=0.33\columnwidth]{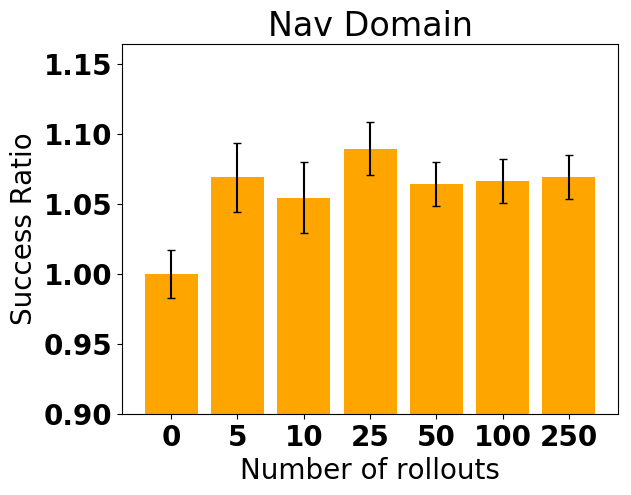}%
\hspace{\fill}\includegraphics[width=0.33\columnwidth]{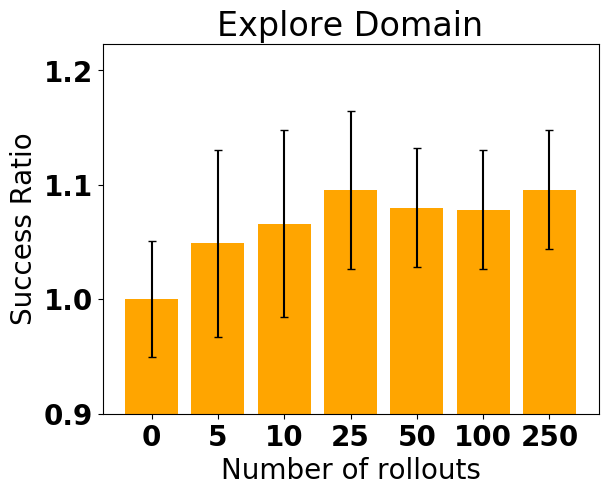}
	
		\includegraphics[width=0.33\columnwidth]{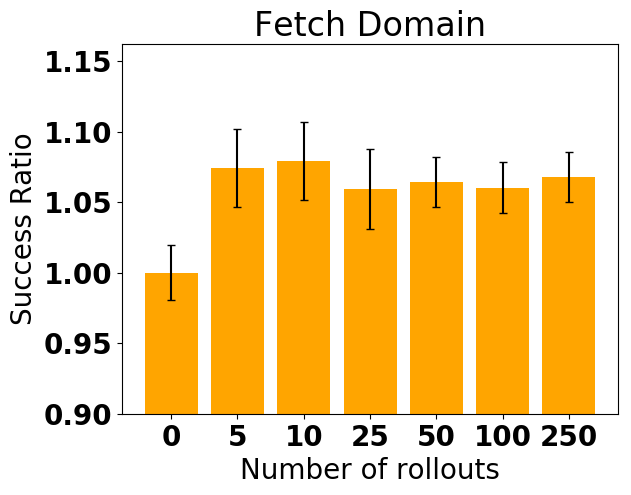}%
\hspace{\fill}\includegraphics[width=0.33\columnwidth]{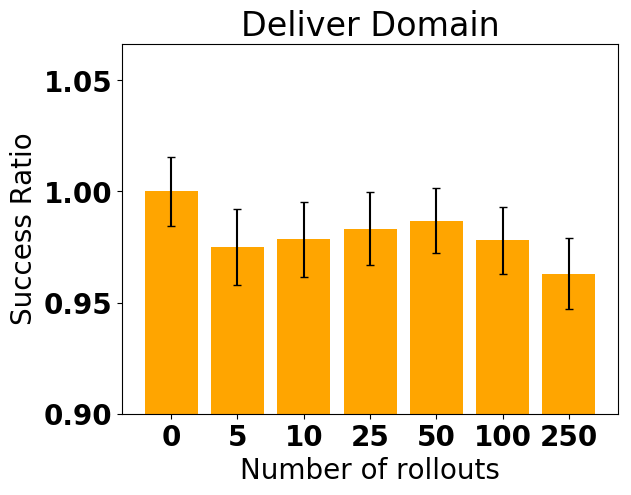}
\hspace{\fill}\includegraphics[width=0.33\columnwidth]{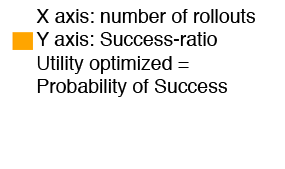}
	\caption{Measured success ratio (\# of successful incoming tasks/ total \# of incoming tasks) for  $n_{ro} \in [0, 250]$ and  $d_{max} = \infty$; $Y$ axis  rescaled with respect to the base case for $n_{ro}=0$.}
	\label{fig:sr_de_UCT_max_depth}
\end{figure}

\begin{figure}[t]
\centering 
	\includegraphics[width=0.33\columnwidth]{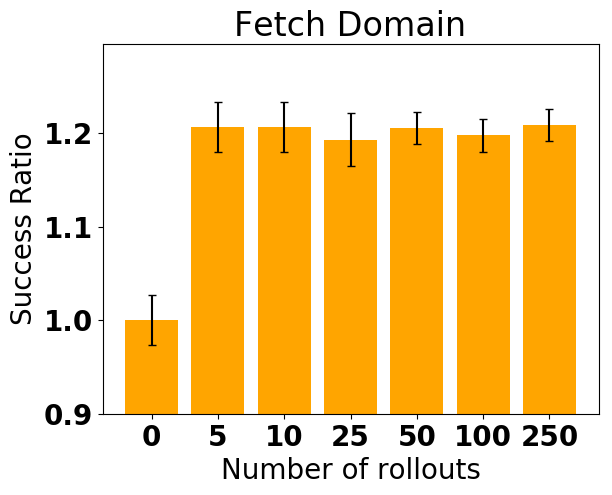}
	\includegraphics[width=0.33\columnwidth]{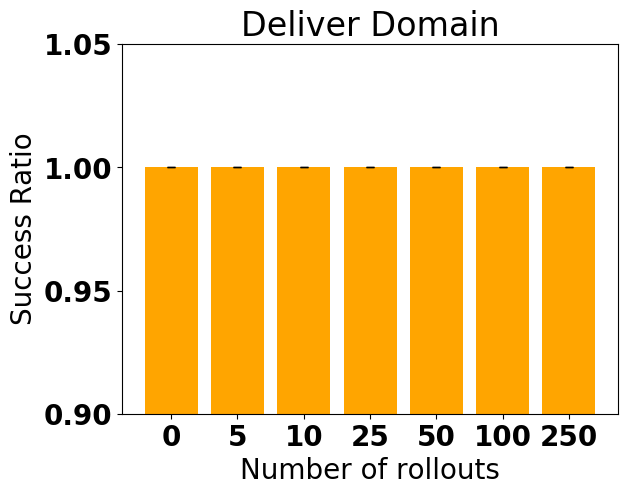}
	
	\caption{Measured success ratio averaged over only test cases with one root task, in \CR and \OF domains with $d_{max} = \infty$; $Y$ axis  rescaled with respect to the base case for $n_{ro}=0$.}
	\label{fig:sr_one_task}
\end{figure}

\paragraph{Depth and Heuristics}
We ran \UPOM at different values of $d_{max} \in [0, 30]$, without progressive deepening in \sv{Select}. At the depth limit, \PLAN estimates the remaining efficiency using one of the following heuristic functions:
\begin{LIST}
	\item \phantomsection \label{def:heur}$h_0$ always returns $\infty$;
	\item $h_D$ is a hand written domain specific heuristic;
	\item $h_\lh$ is the heuristic function learned by the \lh procedure (\autoref{sec:lh}).
\end{LIST}

The results, in \autoref{fig:nu_de_UCT_lim_depth}, show that the efficiency generally increases with depth across all domains. In the \SD domain, the $h_\lh$ performs better than $h_0$ and $h_D$ with 95\% confidence at depths 2 and 3. In the \EE domain, $h_\lh$ performs better than $h_0$ and $h_D$ at depth 1 with 95\% confidence. The same is true for \CR at depth 2. In the \OF domain, the learned heuristic performs better than the others with 95\% confidence for all depths $>=$ 1. The performance difference between the three different heuristics are due to the properties of the domain, how the refinement methods are designed and how much of it is learnable by the \lh procedure.

\begin{figure}[!ht]
\centering 

	\includegraphics[width=0.33\columnwidth]{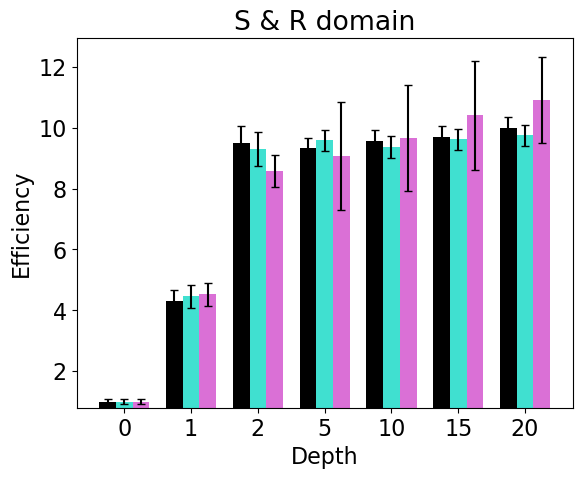}%
\hspace{\fill}\includegraphics[width=0.33\columnwidth]{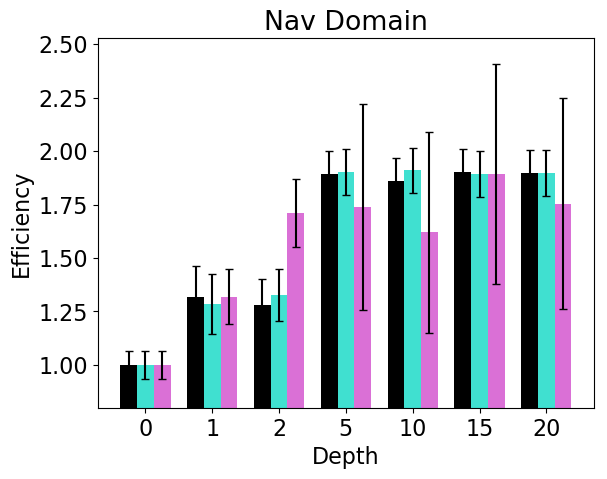}%
\hspace{\fill}\includegraphics[width=0.33\columnwidth]{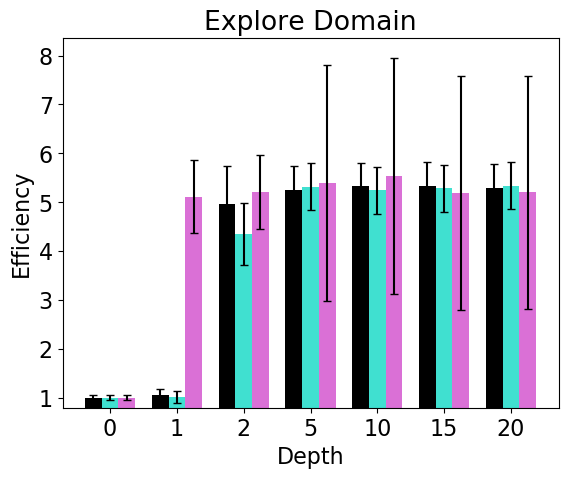}
	
		\includegraphics[width=0.33\columnwidth]{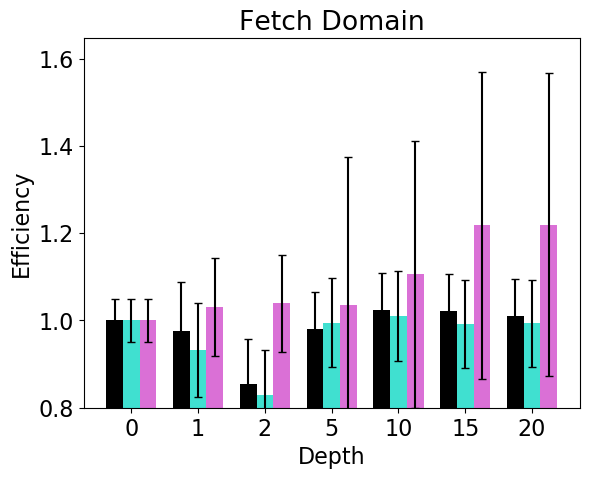}%
\hspace{\fill}\includegraphics[width=0.33\columnwidth]{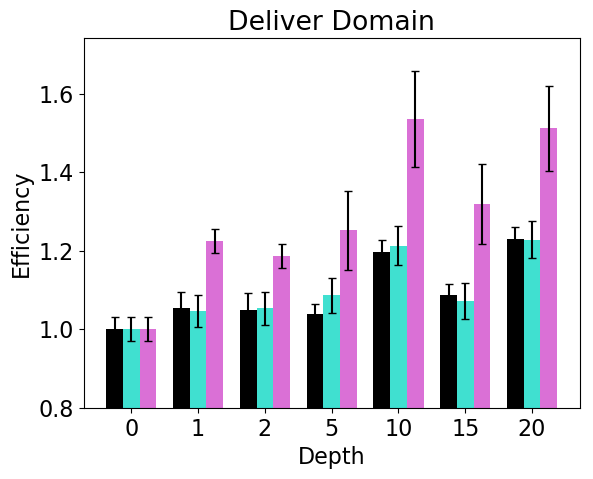}%
\hspace{\fill}\includegraphics[width=0.3\columnwidth]{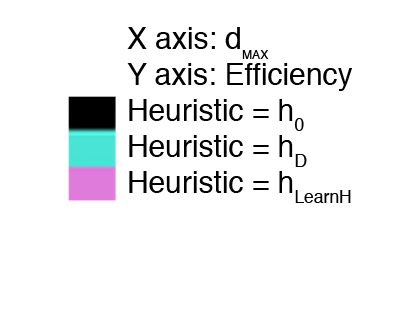}
	
	\caption{Measured efficiency with limited depth and  three different heuristic functions. The utility function optimized  is expected efficiency; $Y$ axis  rescaled with respect to the base case for $n_{ro}=0$.}
	\label{fig:nu_de_UCT_lim_depth}
\end{figure}

\paragraph{Measured \textit{vs} expected efficiency}
Each time \RAE calls \UPOM, \UPOM uses its rollouts to make a prediction of expected efficiency.
This predicted efficiency may differ from the measured efficiency that \RAE achieves by time that it finishes. We will use the term {\em relative error} to denote the absolute value of this difference divided by the error in \UPOM's initial prediction with a constant number of rollouts \footnote{We choose to plot the error relative to 5 \UPOM rollouts.}  (i.e., \UPOM's prediction before \RAE has performed any actions).
\autoref{fig:error_eff} shows the average relative error of \UPOM's predictions, averaged over all $12,500$ runs of our test problems,
as a function of \RAE's {\em progress}, i.e.,
how many actions \RAE has performed since it began.
Note that \UPOM's relative error generally decreases as \RAE's progress increases, and that it is quite small after just one or two actions.

\begin{figure}[!ht]
	\centering 
	
	\includegraphics[width=0.6\columnwidth]{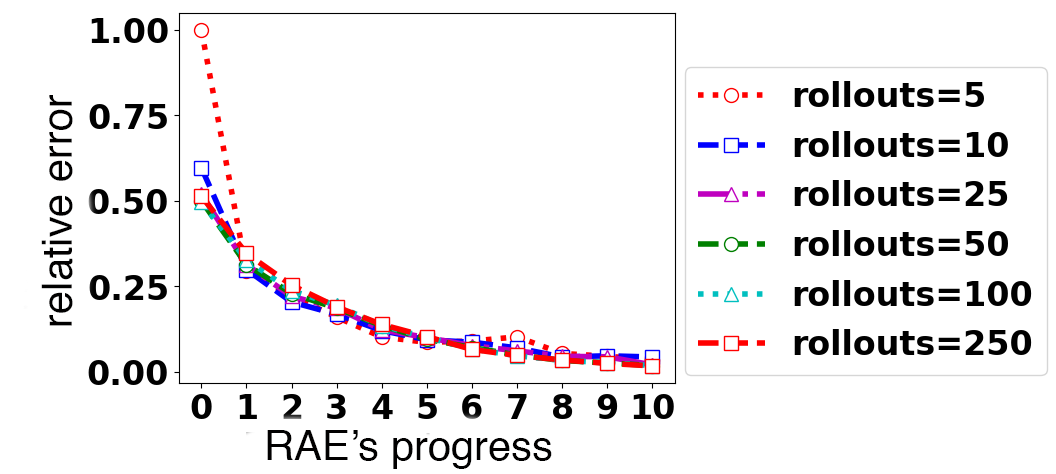}
	
	\caption{Average relative error of \UPOM's predictions, for various numbers of rollouts by \UPOM, shown as a function of \RAE's progress (i.e., how many actions \RAE has performed so far). Each data point is an average over all $12,500$ runs of our test problems.
}
	
	\label{fig:error_eff}
\end{figure}

\subsection{Assessment of  \PLAN}

We are not aware of any comparable planner for operational models, but of  \RPLAN \cite{patra2019acting}, a Monte Carlo Tree Search procedure we developed earlier. We discuss here \RAE with \PLAN \textit{vs} \RPLAN.\footnote{We didn't compare \PLAN with any non-hierarchical planning algorithms because it would be very difficult to perform a fair comparison, as discussed in \cite{kambhampati2003are}.
}
We configured \PLAN to optimize the expected efficiency as its utility function, the same as \RPLAN. 
In order not to favor the UCT strategy of \PLAN with respect to the tree branching strategy of \RPLAN, we set $n_{ro}=1000$, with $d_{max} = \infty$ in each rollout.

\autoref{fig:totalTime} shows the computation time for a single run of a problem (one or two root tasks), averaged across all domains and problems, i.e., over $12500$ runs. \RAE with \PLAN runs more than twice as fast as \RAE with \RPLAN. Note that the computation  time of \RAE alone is negligible, since it is designed to be a fast reactive system, without search.  However, in physical experiments, the total time includes sensing and actuation time, hence the planning overhead would not appear as significant as it is here.

\begin{figure}[!h]
	\centering
	\includegraphics[width=0.6\columnwidth]{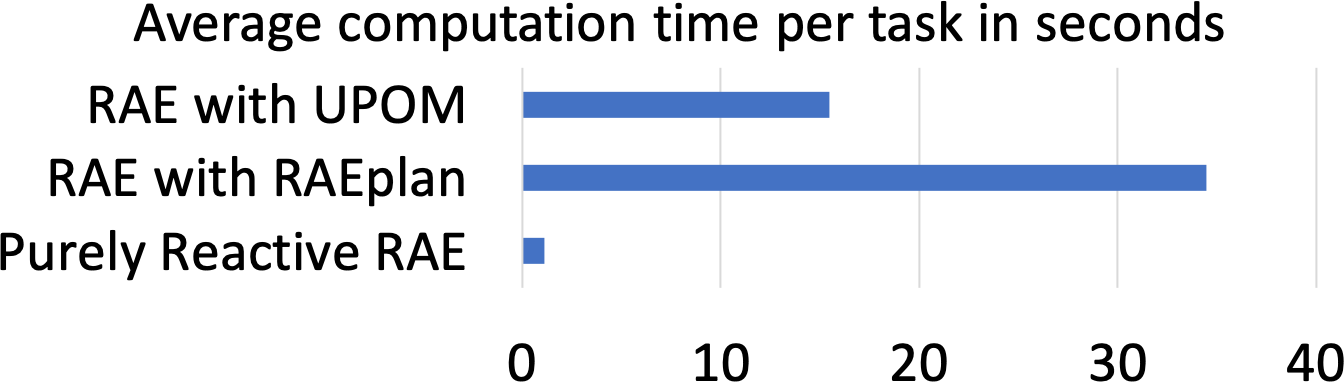}
	\caption{Average computation time in seconds for a single run of a problem, for \RAE with and without the planners.}
	\label{fig:totalTime}
\end{figure}

\paragraph{Efficiency} 
\autoref{fig:nu} gives the measured efficiency for the five domains, with the 95\% confidence intervals.
It shows in all domains that \RAE with \PLAN is more efficient than purely reactive \RAE and \RAE with \RPLAN.

\begin{figure}[!ht]
	\centering
	\includegraphics[width=0.99\columnwidth]{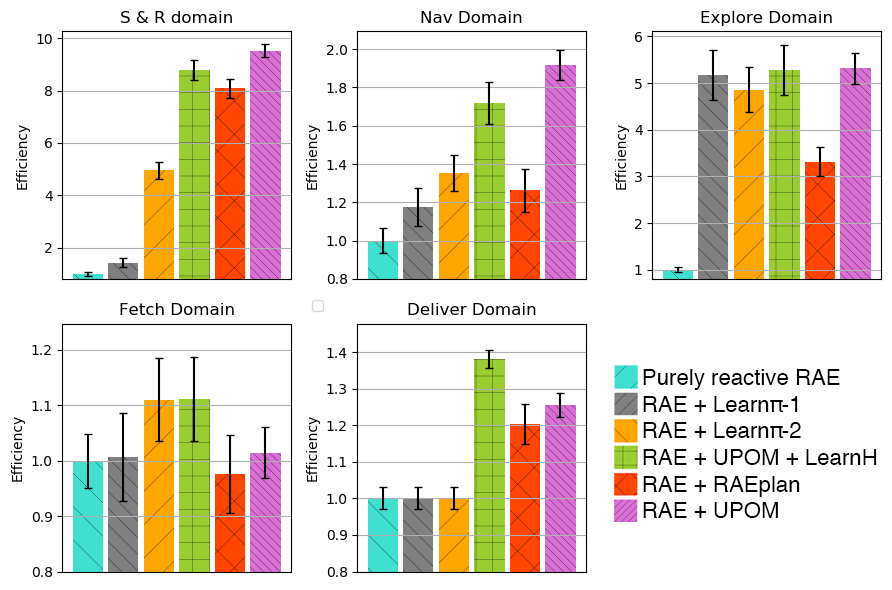}
	\caption{Measured efficiency for each domain with purely reactive \RAE, \RAE with \RPLAN, \RAE with the policies learned by \lm without planning, \RAE with \PLAN, the heuristic learned by \lh and $d_{max}=5$, and \RAE with \PLAN and $d_{max}=\infty$; $Y$ axis  rescaled with respect to the base case for $n_{ro}=0$.} 
	\label{fig:nu}
\end{figure}

\paragraph{Success ratio} 
\autoref{fig:sr} shows \RAE's success ratio both with and without the planners.
We observe that planning with \PLAN outperforms purely reactive \RAE in \SR and \CR with 95\% confidence, and \EE and \SD with 85\% confidence. Also,  \PLAN outperforms \RPLAN in \CR and \SD domains with a 95\% confidence, and \EE domain with 85\% confidence. In the \SR domain, the success ratio is similar for \RPLAN and \PLAN.

Asymptotically, \PLAN and \RPLAN should have near-equivalent efficiency and success ratio metrics. They differ because neither are able to traverse the entire search space due to computational constraints. Our experiments on simulated environments suggest that \PLAN is more effective than \RPLAN when called online with real-time constraints.

\begin{figure}[t]
	\centering
	\includegraphics[width=0.99\columnwidth]{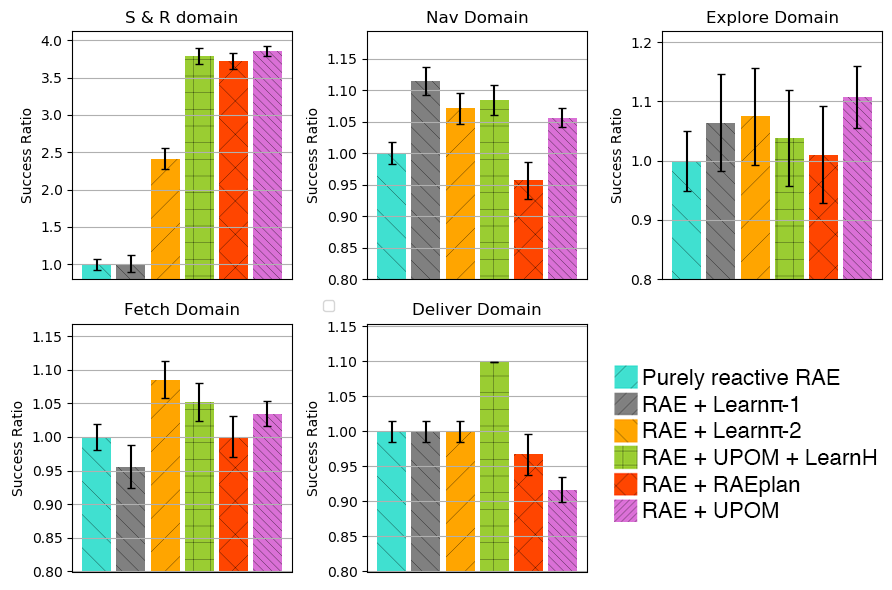}
	\caption{Measured success ratio for each domain with purely reactive \RAE, \RAE with \RPLAN, \RAE with the policies learned by \lm without planning, \RAE with \PLAN, the heuristic learned by \lh and $d_{max}=5$, and \RAE with \PLAN and $d_{max}=\infty$; $Y$ axis  rescaled with respect to the base case for $n_{ro}=0$.}
	\label{fig:sr}
\end{figure}

\subsection{Assessment of learning procedures}

For training purposes, we synthesized data records for each domain by randomly generating root tasks and then running \RAE with \PLAN. The number of randomly generated tasks in \SR, \SD, \EE, \CR,  and \OF domains are 96, 132, 189, 123, and 100 respectively.  We save the data records according to the \lm-1, \lm-2, \lmi and \lh procedures, and encode them using the One-Hot schema. We divide the training set randomly into two parts: 80\% for training and 20\% for validation to avoid overfitting on the training data. 

The training and validation losses decrease and the accuracy increases with increase in the number of training epochs
(see \autoref{fig:training}).

\begin{figure}[h]
	\centering
	\includegraphics[width=0.7\columnwidth]{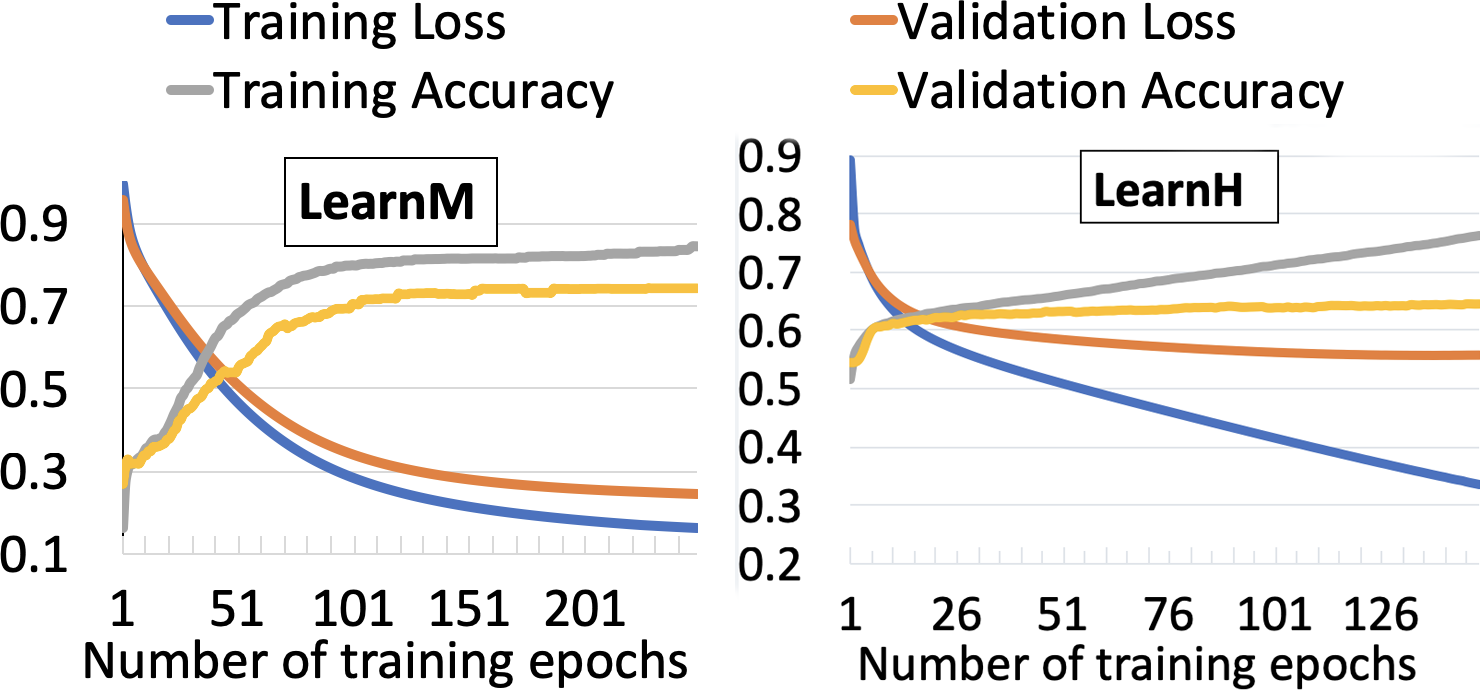}
	\caption{Training and validation results for \lm and \lh, averaged over all domains.}
	\label{fig:training}
\end{figure}

The accuracy of \lm is measured by checking whether the refinement method instance returned by \PLAN matches the template predicted by the MLP $nn_\pi$, whereas the accuracy of \lh is measured by checking whether the efficiency estimated by \PLAN lies in the interval predicted by $nn_{H}$. We chose the learning rate to be in the range $[10^{-3}, 10^{-1}]$. Learning rate is a scaling factor that controls how weights are updated in each training epoch via backpropagation.

\begin{table*}[!h]
	\begin{center}
		\resizebox{1\textwidth}{!}{
			\begin{tabular}{| c | c | c | c | c | c  | c | c  |  c | c | l}
				\cline{1-10}
				\text{Domain} & \multicolumn{3}{ c |}{Training Set Size}  & \multicolumn{2}{ c |}{\#(input features)}  &  \multicolumn{2}{ c |}{Training epochs}  & \multicolumn{2}{ c |}{\#(outputs)}  \\ 
				\cline{2-10}
				& \text{LM-1} & \text{LM-2} & \text{LH} &  \text{LM-1 and -2} & \text{LH} & \text{LM-1 and -2} & \text{LH} & LM-1 and -2 &  LH \\
				\cline{1-10}

				\SR  & 250 &  634 &  3542 & 330  & 401 & 225 & 250 & 16 & 10 \\ 
				\SD  & 1686 &  5331 &  16251& 126 & 144 & 750  & 150 & 9 & 75 \\ 
				\EE  & 2391 &  6883 &  10503 & 182 & 204 & 1000  & 250 & 17 & 200 \\ 
				\CR  & 262 & 508 &  1084 & 97  &  104 & 430  & 250 & 10 &  100 \\ 
				\OF  & - & - &  2001 & -  &  627 & -  & 250 & - &  10 \\
				\cline{1-10}
				
			\end{tabular}
		}
	\end{center}
	\caption{The size of the training set, number of input features and outputs, and the number of training epochs for three different learning procedures: \lm-1, \lm-2, and \lh. We note LM-1 = \lm-1, LM-2 = \lm-2, and LH = \lh.}
	\label{fig:learningInfo}
\end{table*}

\autoref{fig:learningInfo} summarizes the training set size, the number of input features and outputs after data records are encoded using the One-Hot schema, number of training epochs for the three different learning procedures. In the \lh learning procedure, we define the number of output intervals $K$ from the training data such that each interval has an approximately equal number of data records. The final validation accuracies for \lm are 65\%, 91\%, 66\% and 78\% in the domains \CR, \EE, \SR and \SD respectively. The final validation accuracies for \lh are similar but slightly lower. The accuracy values may possibly improve with more training data and encoding the refinement stacks as part of the input feature vectors.

To test the learning procedures we measured the efficiency and success ratio of \RAE with the policies learned by \lm-1 and \lm-2 without planning, and \RAE with \PLAN and the heuristic learned by \lh. We use the same test suite as in our experiments with \RAE using \RPLAN and \PLAN, and do 20 runs for each test problem. When using \PLAN with \lh, we set $d_{max}$ to 5 and $n_{ro}$ to 50, which has  about 88\% less computation time compared to using \PLAN  with infinite $d_{max}$ and $n_{ro} = 1000$. Since the learning happens offline, there is almost no computational overhead when \RAE uses the learned models for online acting.

\paragraph{Efficiency}
\autoref{fig:nu} shows that \RAE with \PLAN+ \lh is more efficient than both purely reactive \RAE and \RAE with \RPLAN in three domains (\EE, \SR and \SD)  with 95\% confidence, and in the \CR domain with 90\% confidence. The efficiency of \RAE with
\lm-1 and \lm-2 lies in between \RAE with \RPLAN and \RAE with \PLAN + \lh, except in the \SR domain, where they perform worse than \RAE with \RPLAN but better than purely reactive \RAE. This is possibly because the refinement stack plays a major role in the resulting efficiency  in the \SR domain. 

\paragraph{Success ratio}
In these last experiments, \PLAN optimizes for the efficiency, not the success ratio. It is however interesting to see how we perform for this criteria even when it is not the chosen utility function. 
In \autoref{fig:sr}, we observe that \RAE with \PLAN + \lh outperforms purely reactive \RAE and \RAE with \RPLAN in three domains (\CR, \SD and \SR) with 95\% confidence in terms of success ratio. In \EE, there is only slight improvement in success-ratio possibly because of high level of nondeterminism in the domain's design. 

In a majority of the domains, the best efficiency and success ratio is achieved by either \RAE with \UPOM, or \RAE with \UPOM + \lh. However, their computation times are quite different. Note that when \UPOM is run with \lh, the rollout lengths are quite shallow and the computation time is similar to purely reactive \RAE (Figure \ref{fig:totalTime}). This makes it highy scalable. In contrast, \RAE with \UPOM explores each rollout to its maximum possible depth. This increases the computation time, making it less suitable for online usage with strict time constraints.

In most cases, we observe that \RAE does better with \lm-2 than with \lm-1. Recall that the training set for \lm-2 is created with all methods returned by \PLAN regardless of whether they succeed while acting or not, whereas \lm-1 leaves out the methods that don't. This makes \lm-1's training set much smaller. In our simulated environments, the acting failures due to random exogenous events don't have a learnable pattern, and a  smaller training set makes \lm-1's performance worse.

\begin{table*}[!h]
	\begin{center}
		\resizebox{1\textwidth}{!}{
			\begin{tabular}{| c | c | c | c | c | c | l}
				\cline{0-5}
				\text{Domain} & Method & Parameter & Training Set Size & \#(input features)  & \#(outputs)  \\ 
				\cline{0-5}
				\SD  & MoveThroughDoorway\_M2 & $robot$ & 404 &   150 & 4 \\ 
				\cline{2-6}
				  & Recover\_M1 & $robot$ & 337 &  128 & 4 \\ 
				\cline{0-5}
				\OF  & Order\_M1 & $machine$ & 296 &  613 & 5  \\
				\cline{3-6}
				  &  & $objList$ & 297 &  613 & 2 \\
				  \cline{2-6}
				  & Order\_M2 & $machine$ & 95 &  613 & 5  \\
				  \cline{3-6}
				  & & $objList$ & 95 &  613 & 2 \\
				  \cline{3-6}
				  &  & $pallet$ & 95 &  613 & 4  \\
				  \cline{2-6}
				  & PickupAndLoad\_M1 & $robot$ & 244 &  637 & 7  \\
				  \cline{2-6}
				  & UnloadAndDeliver\_M1 & $robot$ & 219 &  625 & 7  \\
				  \cline{2-6}
				  & MoveToPallet\_M1 & $robot$ & 7 &  633 & 7  \\
				\cline{0-5}
			\end{tabular}
		}
	\end{center}
	\caption{The size of the training set, number of input features and outputs for learning method parameters in \lmi.}
	\label{table:mi}
\end{table*}

\begin{figure}[!h]
	\centering 
	\includegraphics[width=0.99\columnwidth]{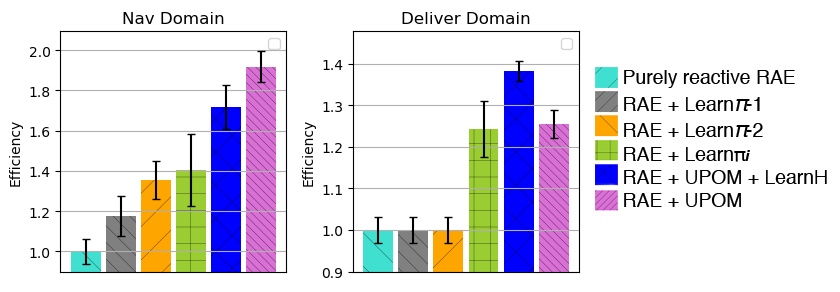}
	\caption{The cross-hatched blue bars show the performance of \RAE with \lmi (learning method instances) for the two domains, \SD and \OF, which have methods with parameters not in tasks ($Y$ axis rescaled with respect to the base case for $n_{ro}=0$).}
\label{fig:learnMI}
\end{figure}

\paragraph{Learning Method Instances} 

Two of our simulated domains, \SD and \OF, have refinement methods with parameters that are not inherited from the task at hand. For these domains, \lm-1 and \lm-2 give only partially instantiated methods, while \lmi is more discriminate. To test its benefit, we trained a MLP for each  parameter not specified in the task. The size of the training set, number of input features and number of outputs are summarized in \autoref{table:mi}.

Figure~\ref{fig:learnMI} compares the efficiency of \RAE with \lmi vs purely reactive \RAE and \RAE with \RPLAN, \lm-1, \lm-2, \lh, and \UPOM. In the \OF domain, \RAE with \lmi is better than purely reactive \RAE as well as  \RAE with \lm-1 or \lm-2 with 95\% confidence. In the \SD domain, \RAE with \lmi  also outperforms \lm-1 and purely reactive \RAE with 95\% confidence, but  not \lm-2.  The performance benefit is important
 in the \OF domain because the refinement methods have several parameters that are not arguments of the task that the method is used for, hence there are combinatorially many different ways to instantiate these parameters when creating method instances.

In summary, for all the domains, planning with \PLAN and learning clearly outperforms purely reactive \APE.

\section{Discussion}
\label{sec:discussion}

We now discuss several issues that readers may find helpful for using and deploying \RAE and \UPOM in practical applications. We also point out several limitations and topics for future work.

\subsection{Retrial in \RAE}
As mentioned earlier,  \textsf{Retry} is not a backtracking procedure. Since \RAE interacts with a dynamic world, \textsf{Retry} cannot go back to a previous state. It selects a method instance among those applicable in the \textit{current} world state, except for those that have been tried before and failed. This latter restriction may not always be necessary, since the same method instance that failed at some point may succeed later on. It can be complicated to analyze the conditions responsible for failures and ascertain whether they still hold. However, \RAE can be adapted to retrial of method instances if they are vulnerable to noisy sensing, or if the execution context is one in which they should be retried. For example, one may give the methods additional parameters that are not needed for the logic of the method instance but that characterize the context (e.g., the pose of a sensor that may have changed between trials), while bounding the number of retrials.

Retrial can be applied more easily for actions. In \RAE, a method instance fails when one of its actions fails. But if an action has nondeterministic outcomes, it may be worthwhile retrying the action as assessed by its expected utility. This may be implemented after a full analysis and the computation of an optimal MDP policy,\footnote{
\newcommand{\fail}{\textit{fail}}
This can be done with a sequence of dummy states $s_{\fail_i}$ such that the effects of action $a$ in $s$ include $s^\fail_1 \in \gamma(s,a),\, s^\fail_2 \in \gamma(s^\fail_1,a),\, \ldots$, $s^\fail_{i+1} \in \gamma(s^\fail_i,a)$, with two actions being applicable to each $s^\fail_i$: $a$ and \sv{stop-with-failure}.
}
or simply with an ad-hoc loop on the \sv{execution-status} of the actions that merit retrials. Furthermore, the body of a method is a procedure in which one can specify complex retrial loops. For example, a difficult \sv{grasp} action in robotics may need several sequences of $\langle\sv{move}, \sv{sense}, \sv{grasp}\rangle$ before succeeding or giving up.

In the \OF domain, we implemented a retrial procedure by having the actions set a status describing if they should be retried after failing. This status depends on the nature of the failure, with likely unrectifiable failures such as unsatisfied preconditions not being retried. These actions were called from a wrapper task that can repeat the command up to two additional times.

\subsection{Concurrency}
\RAE's \textit{Agenda} contains several refinement stacks, one for each top-level task, and the purpose of \RAE's main loop is to progress these stacks concurrently.
Each of the experimental domains in Section \ref{sec:implementation} used this facility and involved concurrent tasks. However, our current implementation provides no built-in way to manage possible conflicts and needed synchronizations; instead they must be managed by the refinement methods. 
We handled conflict in our \SR, \SD, and \OF domains 
by having state variables to specify the status of the resources. The task waiting for the resource executes no action until that resource is available again.
We use semaphores to avoid race conditions. 
Such refinement methods could be made easier to write by extending \RAE to provide synchronization constructs such as those used in TCA and TDL \cite{Simmons:1994vr, Simmons:1998wf}, and rely on the \sv{execution-status} of actions to handle waits. Since both \PLAN and learning rely on the simulation of the method instances, they could support such extensions as long as \sv{Sample} can simulate the duration of actions. More  research is needed on how to extend \RAE to permit formal verification of concurrency properties (liveness, deadlocks), e.g., as in the Petri-net based reactive system ASPiC \cite{lesire2018aspic}.

Note that it is possible to extend \RAE to allow the body of a method to specify concurrent subtasks (see \cite[Sect. 3.2.4]{ghallab2016automated}). However, our current version of \RAE does not include such a facility.

\subsection{Learning operational models}
The \lm and \lh procedures are designed to improve the decision making of \RAE, with or without planning. They are also of help to a domain author, who does not need to design a minimal set of methods associated to a preference ordering. However, assistance in acquiring operational models would be highly desirable. Let us raise few remarks about this important issue of future work. 

Actions and methods, the two main components of operational models, demand different learning techniques. Execution models of actions are often domain dependent. For example, in robotics several approaches have been proposed for learning primitive actions, e.g., \cite{Peters:2008tx, kober:2012vh,ToddHester:2012tn,Deisenroth:2013hv}, usually relying on \textit{Reinforcement Learning} (RL), possibly supervised and/or with inverse RL  (see survey \cite{Kober:2013fh}). Other techniques for learning actions as low-level skills may also be relevant, e.g., \cite{Colledanchise:2015um,Zhang:2018dh}. These and similar techniques would provide operational models of primitive actions needed by \RAE,  as well as a domain simulator needed by \UPOM for the function \sv{Sample} (see line \ref{uct:sample} of \autoref{alg:upom}).
However, since \UPOM may call \sv{Sample} many times during its Monte Carlo rollouts, a detailed domain simulator may have a high computational overhead. A learned domain independent but shallow model of primitive actions, e.g., \cite{pasula2007learning, lang2012exploration}, would provide an efficient \sv{Sample} function.

Regarding refinement methods, several techniques have been developed for learning HTN methods, e.g., \cite{hogg2008HTN, hogg2009learning, hogg2010learning, zhuo2009learning, Zhang:2018dh}. However, our refinement methods for operational models can be significantly more complex. Possible avenues of investigation for synthesizing these methods include program synthesis techniques \cite{walkinshaw2016inferring,gulwani2017program}, partial programming and RL \cite{Andre:2002vo,Marthi:2005tg,Simpkins:2008vh}, as well as learning from the demonstrations of a tutor \cite{Argall:2009hh}.

\subsection{Potential and limitations}

The work presented in this paper includes a reactive system (\RAE), extended with capabilities for planning (\UPOM) and learning (\lm, \lmi, \lh). It is intended to empower an autonomous system facing a diversity of tasks in nondeterministic dynamic environments with robust and efficient deliberation competences. It makes use of a hierarchical representation of tasks and refinement methods, and abstraction of the reactive system's states for use in the planning and learning algorithms.

Beyond the aspects stressed earlier, there are several possible ways to extend the work:
\begin{LIST}
\item In addition to planning and acting, a deliberative system needs also \textit{monitoring}. This function fits naturally in our framework through methods for handling alarms and other surveilled events. Many fault detection, identification and recovery systems, e.g., in space and critical applications \cite{hwang2010survey, Pettersson:2005jq}, can use approaches like \RAE. Furthermore, \RAE and \UPOM have been used in a prototype cybersecurity monitoring and recovery system (see Section \ref{sec:deploy}) in which monitoring functions invoke \RAE and \UPOM when they are needed to plan recovery from cyber-attacks. 
\item Although \textit{time} and temporal primitives have not been introduced in our relatively simple representation, these constructs are widely used in several reactive langages, such as TCA/TDL or Petri Net based systems. Offering similar facilities in \RAE appears quite feasible, but the extensions in \UPOM will require more work.
\item Similarly, we have not covered space and motion planning.  But these can be added in \RAE, e.g., as external functions conditioning particular actions. They can be part of the simulations performed in \UPOM. 
\end{LIST}
There also are several limitations, two of which need to be underlined:
\begin{LIST}
\item The approach developed in this paper is not suitable for addressing intensive combinatorial  search and optimization problems. For example, while the approach is adequate for a video game such as Starcraft; for board games such as chess or go, it is better to use game-tree search and learning techniques 
\cite{newborn2013deep,silver2018general}. Similarly, the organization of  numerous resources over a long horizon in a well-modeled predictable environment are better addressed with temporal planning and scheduling tools, e.g., \cite{estlin2007increased}. However, it is conceivable to connect such a tool for assigning tasks to actors that are driven by our approach (see \cite[Section 4.5]{ghallab2016automated}).
\item Our approach, as presented, does not integrate general inference-making mechanisms, e.g., in order to deduce from common sense or general knowledge that in some contexts, additional care is necessary for an activity to succeed. Right now, such a situation would require refinement methods that are adapted to that context; and in their absence, \UPOM would find that all of the available methods fail. Possible avenues of response to this limitation include online continual learning of methods, and extending \RAE to include an inference engine that uses general and domain-specific axioms \cite{nau2003shop2,thiebaux2005in}.
\end{LIST}

\subsection{Deployment in new applications}\label{sec:deploy}
To deploy the proposed approach in a new domain, it is necessary to have an execution platform or the equivalent collection of sensing and primitive actions, as well as a set of methods. If the tasks and events are well-defined and the techniques for handling them are known, human experts may be able to write the methods without too much difficulty. 

 \RAE executes methods and triggers the execution of primitive actions by the platform. \UPOM has to simulate both. Simulating refinement methods is not an issue, but more effort may be needed to simulate the execution platform and its environment. Clearly, the reliability of such a simulator affects the quality of \UPOM's steering. It is not  straightforward to develop a reliable simulator that reproduces the dynamics of a nondeterministic uncertain environment.  In some cases, available simulation tools can be useful, e.g., physics-based simulations \cite{boeing2007evaluation, faure2012sofa,bridson2015fluid}, robotic simulations \cite{michel2004cyberbotics, leon2010opengrasp,shah2018airsim}, automated manufacturing simulations \cite{jahangirian2010simulation,mourtzis2014simulation} and so forth. Some of these tools are often used for the specification and design of an execution platform, which may simplify their use for the development of a simulator. A fallback option, easily applicable in most cases, would be to define the  procedure \sv{Sample} by sampling the possible outcomes of every action from probability distributions, which are initialized by a human expert, then refined by learning and experiments. The results in  \autoref{sec:implementation} show that this shallow simulation  already provides substantial improvements in decision-making.
Note that it is possible to combine detailed simulations for critical actions, for which tools might be available, and shallow simulations for the remaining actions.

The deployment of \RAE and \UPOM in a prototype application for security monitoring and recovery from attacks on Software-Defined Networks (SDN) is described in \cite{patra2021using}. It defines an abstract representation of the SDN's state used by \RAE and \UPOM, hand-programmed primitive actions that could be executed on the SDN,  tasks and events regarding attacks on an SDN, and a set of
refinement methods giving possible recovery procedures for each of the attacks. Several refinement method are applicable to the same task or event; the human expert does not need to specify which applicable method is preferable in which contexts, since the purpose of \UPOM is to make such choices online.  The final system has been evaluated by SDN experts successfully.

The development effort required about three person-months of work by the human expert,
but the human expert's ongoing experience with \RAE and \UPOM revealed software features that we needed to add so that his development effort could proceed.
Given \RAE and \UPOM in their current form, we think such a development effort might take much less time.

\section{Conclusion}
\label{sec:conclusions}

We have presented a novel system for integrating acting and planning using hierarchical refinement operational models. The refinement acting engine, \RAE, can either run purely reactively or it can get advice from an online planner to choose efficient method instances for performing a task. The planning procedure, \PLAN, uses an anytime search strategy inspired by UCT, extended to operate in a more complicated search space. 
\PLAN provides near-optimal method instances with respect to utility functions that may be quite general, and it converges asymptotically. We have proposed two distinct utility functions that favor efficiency and robustness,  respectively.
\PLAN is integrated with \RAE using a receding-horizon, anytime progressive-deepening procedure. 

We have also presented three learning strategies:
\lm, to learn a mapping from a task in a given context to a good method, \lmi, to learn values of uninstantiated method parameters, and \lh, to learn a domain specific heuristic function for our hierarchical refinement framework. We have shown how incremental learning can be integrated online with acting and planning.

We have presented empirical results over five domains that have challenging features such as dynamicity,  dead-ends,  exogenous events, sensing and information gathering actions, collaborative and concurrent tasks. Rather than just evaluating the system's planning functionality, we have devised simulations and measurements that assess its overall acting performance, with and without planning and learning, taking into account exogenous events and failure cases. 

We have measured the actor's efficiency, success ratio and retry ratio, and discussed their relationships with respect to the planner's utility function, maximizing either the expected efficiency or the expected success ratio.
Our results show that \lm improves the performance of reactive \RAE with respect to the three performance measures, and they are improved even further when \RAE is used either with \PLAN and \lh or with \PLAN at unbounded search depth.
Thanks to learning, the computational overhead remains acceptable for an online procedure, since in this case
a small number of rollouts already bring a good benefit.

In summary, acting purely reactively in dynamic domains with dead ends can be costly and risky. Our proposed integration of acting, planning and learning can be of great benefit, reflected by a higher efficiency or robustness. 
An open source repository of all of the algorithms and test domains is available at \cite{upomrepo}.

\paragraph{Acknowledgements} We are very grateful to
F\'elix Ingrand, LAAS-CNRS, for very fruitful discussions regarding the initial specification of \RAE's algorithms, and
to Amit Kumar, U. Maryland, for advices regarding the multi-layer perceptron design and training. Many thanks to the anonymous reviewers for their very positive and insightful remarks.
This work has been supported in part by NRL grants N00173191G001 
and N0017320P0399, 
ONR grant N000142012257,
and AFOSR grant 1010GWA357.  
The information in this paper does not necessarily reflect the position or policy of the funders, and no official endorsement should be inferred.

\bibliographystyle{elsarticle-harv}

\bibliography{main}

\appendix

\newcommand{\state}{s\xspace}
\newcommand{\States}{S\xspace}

\newcommand{\success}{\textsf{success}\xspace}
\newcommand{\stacks}{\text{stacks}\xspace}
\newcommand{\reward}{R\xspace}
\newcommand{\cost}{C\xspace}
\newcommand{\trans}{P}

\section{Asymptotic Convergence of \UPOM}
\label{app:mapping}

\sloppy  

\newcommand{\height}{\textrm{height}}

In this section we demonstrate the asymptotic convergence of \PLAN towards an optimal method, as $n_{ro} \to \infty$. The proof assumes no depth cut-off ($d_{max} = \infty$) and static domains, i.e., domains without exogenous events. We believe it would be straightforward to extend the proof to dynamic domains if there are known probability distributions over the occurrence of exogenous events. However, we have not attempted to extend the proof to that case.

The proof proceeds by mapping \PLAN's search strategy into UCT, which is demonstrated to converge on a finite horizon MDP with a probability of not finding the optimal action at the root node that goes to zero at a polynomial rate as the number of rollouts grows to infinity (Theorem 6 of \cite{kocsis2006bandit}). 

To simplify the mapping, we first consider \PLAN with an additive utility function,
and show how to map \PLAN's search space into an MDP. We then discuss how this can be extended to  the efficiency and success ratio utility functions
defined in \ref{sec:plan}, using the fact that the UCT algorithm is not restricted to the additive case; it still converges as long as the utility function is monotonic.

\subsection*{A.1. Search Space for Refinement Planning}

Let $\Sigma = (\Xi,\Tasks,\M,\A)$ be an acting domain, as
specified at the end of
\autoref{sec:operational}. Throughout this appendix, we will assume that $\Sigma$ is static.

Recall
from \autoref{sec:plan}
that the space searched by \UPOM is a simulated version of $\Sigma$.
To talk about this formally, recall that a {\em refinement planning domain} is a tuple $\Phi = (\States,\mathcal{T},\M,\A)$, where
$\States$ is the set of states (recall that these are abstractions of states in $\Xi$), 
and $\Tasks$,  $\M$, and $\A$ are the same as in $\Sigma$.
Recall from 
\autoref{sec:operational} that  $S$, $\Tasks$, $\M$, and $\A$ are all finite, and that every sequence of steps generated by the methods in $\M$ is finite.

\phantomsection \label{def:P}For $s\in S$ and $a\in \A$, we let 
$\gamma(s,a) \subseteq S$ be the set of all states that may be produced by simulating $a$'s execution in $s$.
For each $s' \in \gamma(s,a)$, we let $\trans(s,a,s')$ be the probability that state $s'$ will be produced if we simulate $a$'s execution in state $s$.

Recall from \autoref{sec:rae} that a {\em refinement stack} is a LIFO stack in which each element is a tuple $(\tau,m,i,\tried)$, where $\tau$ is a task, $m$ is a method, $i$ is an instruction pointer that points to the $i$'th line of $m$'s body (which is a computer program), and $\tried$ is the set of methods previously tried for $\tau$. We will call the tuple $(\tau,m,i,\tried)$ a {\em stack frame}, and we will let $m[i]$ denote the $i$'th line of the body of $m$.

\phantomsection \label{def:Pi}We now can define a {\em refinement planning problem} to be a tuple $\Pi = (\Phi,s_0,\sigma_0,U)$, where $s_0$ is the initial state, $\sigma_0$ is the initial refinement stack, and $U$ is a utility function.

\paragraph{Rollouts}
A {\em rollout} in $\Phi$ is a sequence of pairs
\begin{equation}
\rho = \langle (\sigma_0,s_0), (\sigma_1,s_1), \ldots, (\sigma_n,s_n)\rangle
\end{equation}
satisfying the following properties:
\begin{LIST}
\item
each $s_i$ is a state, and each $\sigma_i$ is a refinement stack;
\item
for each $i > 0$ there is a nonzero probability that $s_j$ and $\sigma_j$ are the next state and refinement stack after $s_{i-1}$ and $\sigma_{i-1}$;
\item
$(\sigma_n,s_n)$ is a termination point for $\UPOM$.
\end{LIST}
If the final refinement stack is $\sigma_n = \langle \rangle$, i.e., the empty stack,
then the rollout $\rho$ is successful. 
Otherwise $\rho$ fails.

In a top-level call to $\UPOM$, the initial refinement stack $\sigma_0$ would normally be
\begin{equation}
\label{eq:sigma0}
\sigma_0 = \langle (\tau_0, m_0, 1,\nullset)\rangle,
\end{equation} 
where $\tau_0$ is a task, and $m_0$ is a method that is relevant for $\tau_0$ and applicable in $s_0$.
In all subsequent refinement stacks produced by $\UPOM$.

\phantomsection \label{def:reach}We will say that a refinement stack $\sigma$ is {\em reachable} in $\Phi$ (i.e., reachable from a top-level call to $\UPOM$) if 
there exists a rollout
\[
\rho = \langle (\sigma_0,s_0), (\sigma_1,s_1), \ldots, (\sigma_n,s_n)\rangle
\]
such that $\sigma_0$ satisfies \autoref{eq:sigma0} and
$\sigma\in \{\sigma_0,\ldots,\sigma_n\}$. We let $\mathcal{R}(\Phi)$ be the set of all refinement stacks that are reachable in $\Phi$.
Since 
every sequence of steps generated by the methods in $\M$ is finite, it follows that
$\mathcal{R}(\Phi)$ is also finite.

\paragraph{Additive utility functions}
\phantomsection \label{def:reward}
The utility function $U$ is {\em additive} if there is either a reward function $\reward(s)$ or a cost function $\cost(s,a,s')$ (where $(s,a,s')$ is a transition from $s$ to $s'$ caused by action $a$)
such that $U$ is the sum of the rewards or costs associated with the state transitions in $\rho$.  These state transitions are the points in $\rho$ where $\UPOM$ simulates the execution of an action.

For each pair  $(\sigma_j,s_j)$ in $\rho$,
let  $(\tau_j, m_j, i_j,\tried_j)$ be the top element of $\sigma_j$.
If $m_j[i_j]$ is an action, then the next element of $\rho$ is a pair $(\sigma_{j+1},s_{j+1})$ in which $s_{j+1}$ is the 
state produced by executing the action $m_j[i_j]$.
In $\Phi$ this corresponds to the state transition $(s_j, m_j[i_j], s_{j+1})$.
Thus the set of state transitions in $\rho$ is
\begin{align}
\nonumber
t_\rho = \{ (s_j,m_j[i_j],s_{j+1}) \mid\mbox{}& \text{$(\sigma_j,s_j)$ and
 $(\sigma_{j+1},s_{j+1})$ are members of $\rho$,} \\
& \text{$(\tau_j, m_j, i_j,\tried_j) = \sv{top}(\sigma_j)$,
 and
$m_j[i_j]$ is an action} \}.
\end{align}

Thus if $U$ is additive, then
\begin{equation}
U(\rho) =
\begin{cases}
\sum_{(s,a,s') \in t_\rho} \reward(s'),
& \text{if $U$ is the sum of rewards},\\
\sum_{(s,a,s')\in t_\rho} \cost(s,a,s'),
& \text{if $U$ is the sum of costs}.
\end{cases}
\end{equation}

\subsection*{A.2. Defining the MDP} \label{sec:mdp}

We want to define an MDP $\Psi$ such that choosing among methods in $\Phi$ corresponds to choosing among actions in $\Psi$. The easiest way to do this is to let 
all of $\Phi$'s actions and methods be actions in $\Psi$.
Based loosely on the notation in \cite{mausam2012planning},
we will write $\Psi$ as
\begin{equation}
\label{eq:mdp}
\Psi = (\States^\Psi,\A^\Psi, s^\Psi_0, S^\Psi_g, \gamma^\Psi, \trans^\Psi, U^\Psi)
\end{equation}
where
\begin{align*}
S^\Psi &= \stacks(\Phi) \times \States \text{ is the set of states},\\
\A^\Psi &= \M \cup \A \text{ is the set of actions},\\
s^\Psi_0 &= (\sigma_0,s_0) \text{ is the initial state},\\
S^\Psi_g &= \{(\langle \rangle,s) \mid s \in S\} \text{ is the set of goal states},
\end{align*}
and the state-transition function $\gamma^\Psi$, state-transition probability function $\trans^\Psi$, and utility function $U^\Psi$ are defined as follows.

\paragraph{State transitions}

To define $\gamma^\Psi$ and $\trans^\Psi$, we must first define which actions are applicable in each state.
Let $(\sigma,s) \in S^\Psi$, and
$(\tau, m, i,t) = \sv{top}(\sigma)$.
Then the set of actions that are applicable to $(\sigma,s)$ in $\Psi$ is
\begin{equation}
\label{eq:applicable}
\sv{Applicable}^\Psi((\sigma,s)) = 
\begin{cases}
\sv{Instances}(\mathcal{M},m[i],s), &\text{if $m[i]$ is a task},\\
\{m[i]\}, &\text{if $m[i]$ is an action}.
\end{cases}
\end{equation}

Thus if $a\in \sv{Applicable}^\Psi((\sigma,s))$, then there are two cases for what 
$\gamma^\Psi(s,a)$ and
$\trans^\Psi(s,a,s')$ might be:

\begin{LIST}
\item
Case 1:
$m[i]$ is a task in $\M$, and $a \in \sv{Instances}(\mathcal{M},m[i],s)$.
In this case, the next refinement stack will be produced by pushing
a new stack frame $\phi = (m[i], a, 1,\nullset)$ onto $\sigma$.
The state $s$ will remain unchanged.
Thus the next state in $\Psi$ will be $(\phi+\sigma, s)$,
where `+' denotes concatenation. 
Thus
\begin{align*}
\gamma((\sigma,s),a) &= \{(\phi+\sigma, s)\};\\
\trans^\Psi[(\sigma,s), a, (\phi+\sigma, s)] &= 1;\\
\trans^\Psi[(\sigma,s), a, (\sigma',s')] &= 0, \quad
\text{if } (\sigma',s') \neq (\phi+\sigma, s).
\end{align*}
\item
Case 2:
$m[i]$ is an action in $\A$, and $a=m[i]$.
Then $a$'s possible outcomes in $\Psi$ correspond one-to-one to its possible outcomes in $\Phi$. More specifically, if $\gamma$ is the state-transition function for $\Phi$ (see
\autoref{sec:operational}), then
\[
\gamma^\Psi((\sigma,s),a) = \{( \Next(\sigma,s'),s') \mid s' \in 
\gamma(s,a)\}
\]
and
\[
\trans^\Psi((\sigma,s), a, (\sigma',s')))  = 
\begin{cases}
\trans(s, a, s'),& \text{if } (\sigma',s') \in \gamma^\Psi((\sigma,s),a),\\
0, & \text{otherwise.}
\end{cases}
\]
\end{LIST}

\paragraph{Rollouts and utility} \label{def:rollout}

A {\em rollout} of $\Pi^\Psi$ is any sequence of states and actions of $\Psi$,
\[
\rho^\Psi = \langle (\sigma_0,s_0), a_1,  (\sigma_1,s_1), a_2, \ldots, 
(\sigma_{n-1},s_{n-1}), a_n, (\sigma_n,s_n)\rangle,
\]
such that for $i=1,\ldots,n$,
$a_i \in \sv{Applicable}(\sigma_{i-1},s_{i-1})$
and
\[\trans^\Psi((\sigma_{i-1},s_{i-1},(\sigma_i,s_i)), a_i) > 0.
\]
The rollout is {\em successful} if $(\sigma_n,s_n) \in S^\Phi_g$, and unsuccessful otherwise.\medskip

We can define $U^\Psi$ directly from $U$. 
If $\rho^\Psi$ is the rollout given above, then
the corresponding rollout in $\Phi$ is $\rho = \langle (\sigma_0,s_0),  (\sigma_1,s_1), \ldots, 
(\sigma_{n-1},s_{n-1}),  (\sigma_n,s_n)\rangle$, and
\[
U^\Psi(\rho^\Psi)
= U(\rho).
\]
If $U$ is additive, then so is $U^\Psi$. In this case, $\Psi$ satisfies the definition of an MDP with initial state (see \cite{mausam2012planning}).

\subsection*{A.3. Mapping $\UPOM$'s Search to an Equivalent UCT Search}

Let
\begin{equation}
\label{eq:pi}
\Pi = (\Phi,s_0,\sigma_0,U)
\end{equation}
be a refinement planning problem, where
\begin{equation}
\label{eq:phi}
\Phi = (\States,\mathcal{T},\M,\A).
\end{equation}
Suppose $\UPOM(s_0,\sigma_0,\infty)$ generates the rollout
\begin{equation}
\label{eq:rollout}
\rho = \langle (\sigma_0,s_0), (\sigma_1,s_1), \ldots, (\sigma_n,s_n)\rangle,
\end{equation}
where $\sigma_j = (\tau_j, m_j, i_j, \tried_j)$, 
for $j=1,\ldots,n$.
$\UPOM$ generates $\rho$ by choosing $m_1$ and then recursively calling
$\UPOM(s_j,\sigma_j,\infty)$.
Consequently,
$\UPOM$'s probability of generating $\rho$ is
\begin{equation}
\label{eq:rolloutp}
p = p_1 \times \ldots \times p_n,
\end{equation}
 where each $p_j$ is the probability that $\UPOM(s_j,\sigma_j,\infty)$ will choose $m_j$ before making its recursive call.
The value of $p_j$ will depend on $\UPOM$'s \emph{metadata} for $\Pi$, e.g., the number of times each method for a task $\tau$ has been tried in each state $s$, and the average utility obtained over those tries.

We want to show that $\UPOM$'s search of $\Pi$ corresponds to an equivalent UCT search of $\Psi$. Below,
\autoref{th:mapping} accomplishes this in the case where the utility function $U$ is additive. After the proof of the theorem, we discuss how to generalize the theorem to cases where $U$ is not additive.

\begin{theorem}
\label{th:mapping}
Let $\Pi$, $\Phi$, $\rho$ and $p$ be as in Equations \ref{eq:pi}--\ref{eq:rolloutp}, let $U$ be additive, let $\UPOM$'s metadata for $\Pi$ be as described above, and let
$\Psi=(\States^\Psi,\A^\Psi, \gamma^\Psi, \trans^\Psi, U^\Psi)$ be the MDP corresponding to $\Pi$. If UCT searches $\Psi$ using the same metadata that $\UPOM$ used, then the probability that UCT generates the rollout
\[
\rho^\Psi = \langle (\sigma_0,s_0), m_1,  (\sigma_1,s_1), m_2, \ldots, 
(\sigma_{n-1},s_{n-1}), m_n, (\sigma_n,s_n)\rangle
\]
is the same probability $p = p_1 \times \ldots \times p_n$ as in \autoref{eq:rolloutp}.
\end{theorem}
\begin{proof}[Sketch of proof.]
The proof is by induction on $n$, the length of $\rho$.
The base case is when $n=0$, i.e., $\rho = \langle (\sigma_0,s_0) \rangle$. If $n=0$ then it must be that $\sv{Applicable}(s_0) = \nullset$. Thus
$\sv{Applicable}^\Psi((\sigma_0,s_0)) = \nullset$, so in this case the theorem is vacuously true.

For the induction step, suppose $n>0$, and
consider $\UPOM$'s recursive call to $\UPOM(s_1,\sigma_1,\infty)$.
In this case, the refinement planning problem is $\Pi_1=(\Phi,s_1,\sigma_1,U)$, and we let $\Psi_1$ be the corresponding MDP.

Given the same metadata as above, $\UPOM(s_1,\sigma_1,\infty)$ will generate the rollout
$\rho_1 = \langle (\sigma_1,s_1), \ldots, (\sigma_n,s_n)\rangle$
with probability $p_2 \times \ldots \times p_n$.
The induction assumption is that
with that same probability,
a UCT search of $\Psi_1$ will generate the rollout
\begin{align*}
\rho^{\Psi}_1 &= \langle (\sigma_1,s_1), m_2, \ldots, 
(\sigma_{n-1},s_{n-1}), m_n, (\sigma_n,s_n)\rangle.
\end{align*}

Before applying the induction assumption, we first need to show that if $p_1$ is the probability that $\UPOM(s_0,\sigma_0,U)$ chooses $m_1$ before making its recursive call, then a UCT search of $\Psi_1$ will choose $m_1$ with the same probability $p_1$. 
There are two cases:
\begin{LIST}
\item Case 1:
$m_1$ is a method in $\Phi$.
As shown in \autoref{alg:upom},
$\UPOM(s_0,\sigma_0,U)$ 
chooses $m_1$ using the same UCB-style computation that a UCT search in $\Psi$ would use at $(\sigma_0,s_0)$. Thus, omitting the details about how to compute $p_1$ from the metadata, it follows that if $\UPOM(s_0,\sigma_0,U)$ chooses $m_1$ with probability $p_1$, then so does the UCT search.
\item Case 2:
$m_1$ is an action in $\Phi$. Then \UPOM's computation
(in lines \autoref{uct:eff} through the end of \autoref{alg:upom}) is {\em not} a UCT-style computation, but this does not matter, because there is only one possible choice, namely $m_1$. In this case, \UPOM's probability of choosing $m_1$ is $p_1 = 1$, and the same is true for the UCT search.
\end{LIST}
In both cases, it follows from the induction assumption that in $\Pi$, $\UPOM$'s probability of generating $\rho$ is $p_1 \times p_2 \times \ldots \times p_n$, and in $\Pi^\Psi$, UCT's probability of generating $\rho^\Psi$ is also $p_1 \times p_2 \times \ldots \times p_n$.
This concludes the sketch of the proof.
\end{proof}

\paragraph{Generalizing beyond MDPs}
If the utility function $U$ is not additive, \autoref{eq:mdp} produces a probabilistic planning problem that looks  similar to an MDP, the only difference being that the utility function $U^\Psi$ is not additive. 
Furthermore, \autoref{th:mapping} still holds even when $U$ is not additive, if we modify the proof to remove the claim that $\Psi$ is an MDP.

We note that the UCT algorithm \cite{kocsis2006bandit} is not restricted to the case where $U^\Psi$ is additive; it will still converge as long as $U^\Psi$ is monotonic. 
If $U$ is monotonic, then so is $U^\Psi$. In this case it follows that UCT---and thus $\UPOM$---will converge to an optimal solution. In particular, $\UPOM$ will converge to an optimal solution when using the {\em efficiency} and {\em success ratio} utility functions described in 
\autoref{sec:utilities}.

\section{Operational model for the \SR domain}
\label{app:rescue}

\begin{lstlisting}[language=Python,numbers=none]
declare_commands([
    moveEuclidean, # UGV moves from a location to another following a Euclidean path
    moveCurved, # UGV moves from a location to another following a curved path
    moveManhattan, # UGV moves from a location to another following a Manhattan path
    fly, # UAV flies from one location to another
    giveSupportToPerson, # UGV helps one person
    clearLocation, # UGV removes debri from a location
    inspectLocation, # UAV surveys a location, searching for injured people
    inspectPerson, # UGV checks whether a person is injured or not
    transfer, # one UGV transfers medical supplies to another
    replenishSupplies, # UGV replenishes medical supplies at the base
    captureImage, # UAV captures an image using on of its cameras
    changeAltitude, # UAV changes its flying altitude
    deadEnd, # the agent reaches a dead end from which it can't recover from
    fail # the command always fails
    ])

declare_task('moveTo', 'r', 'l') # robot r moves to location l
declare_task('rescue', 'r', 'p') # robot r rescues person p
declare_task('helpPerson', 'r', 'p') # robot r helps a person p 
declare_task('getSupplies', 'r') # UGV r refills its medical supplies from the base
declare_task('survey', 'r', 'l') # UAV r surveys the location l
declare_task('getRobot') # assigns a free robot to help a person
declare_task('adjustAltitude', 'r') # UAV r adjusts its flying altitude depending on the weather conditions



declare_methods('moveTo',  # Four possible methods for the \
    MoveTo_Method4, # task moveTo(r, l)
    MoveTo_Method3, 
    MoveTo_Method2, 
    MoveTo_Method1,
    )

declare_methods('rescue', # Two methods for the task rescue(r, p) 
    Rescue_Method1,
    Rescue_Method2,
    )

declare_methods('helpPerson', # Two methods for the \
    HelpPerson_Method2, # task helpPerson(r, p)
    HelpPerson_Method1, 
    )

declare_methods('getSupplies', # Two methods for the \
    GetSupplies_Method2, #  task getSupplies(r)
    GetSupplies_Method1,
    )

declare_methods('survey', # Two methods for the \
    Survey_Method1, # task survey(r, l)
    Survey_Method2
    )

declare_methods('getRobot', # Two methods for the \
    GetRobot_Method1, # task getRobot
    GetRobot_Method2,
    )

declare_methods('adjustAltitude', # Two methods for the \
    AdjustAltitude_Method1,  # task adjustAltitude(r)
    AdjustAltitude_Method2,
    )
    
# Full descriptions of the commands

def moveEuclidean(r, l1, l2, dist):
	''' UGV r moves from a location l1 to location l2  following a Euclidean path. '''
    (x1, y1) = l1
    (x2, y2) = l2
    xlow = min(x1, x2)
    xhigh = max(x1, x2)
    ylow = min(y1, y2)
    yhigh = max(y1, y2)
    # r checks whether there are any obstacles in the path
    for o in rv.OBSTACLES:    
        (ox, oy) = o
        if ox >= xlow and ox <= xhigh and oy >= ylow and oy <= yhigh:
            if ox == x1 or x2 == x1:
                Simulate("%s cannot move in Euclidean path because of obstacle\n" %r)
                return FAILURE
            elif abs((oy - y1)/(ox - x1) - (y2 - y1)/(x2 - x1)) <= 0.0001:
                Simulate("%s cannot move in Euclidean path because of obstacle\n" %r)
                return FAILURE

    state.loc.AcquireLock(r)
    if l1 == l2:
        Simulate("Robot %s is already at location %s\n" %(r, l2))
        res = SUCCESS
    elif state.loc[r] == l1:
        start = globalTimer.GetTime()
        while(globalTimer.IsCommandExecutionOver('moveEuclidean', start, r, l1, l2, dist) == False):
           pass
        res = Sense('moveEuclidean')
        if res == SUCCESS:
            Simulate("Robot %s has moved from %s to %s\n" %(r, str(l1), str(l2)))
            state.loc[r] = l2
        else:
            Simulate("Robot %s failed to move due to some internal failure.\n" %r)
    else:
        Simulate("Robot %s is not in location %d.\n" %(r, l1))
        res = FAILURE
    state.loc.ReleaseLock(r)
    return res

def moveCurved(r, l1, l2, dist):
	''' UGV r moves from a location l1 to l2 following a curved path'''
    (x1, y1) = l1
    (x2, y2) = l2
    centrex = (x1 + x2)/2
    centrey = (y1 + y2)/2
     # r checks whether there are any obstacles in the path
    for o in rv.OBSTACLES:
        (ox, oy) = o
        r2 = (x2 - centrex)*(x2 - centrex) + (y2 - centrey)*(y2 - centrey)
        ro = (ox - centrex)*(ox - centrex) + (oy - centrey)*(oy - centrey)  
        if abs(r2 - ro) <= 0.0001:
            Simulate("%s cannot move in curved path because of obstacle\n" %r)
            return FAILURE
    
    state.loc.AcquireLock(r)
    if l1 == l2:
        Simulate("Robot %s is already at location %s\n" %(r, l2))
        res = SUCCESS
    elif state.loc[r] == l1:
        start = globalTimer.GetTime()
        while(globalTimer.IsCommandExecutionOver('moveCurved', start, r, l1, l2, dist) == False):
           pass
        res = Sense('moveCurved')
        if res == SUCCESS:
            Simulate("Robot %s has moved from %s to %s\n" %(r, str(l1), str(l2)))
            state.loc[r] = l2
        else:
            Simulate("Robot %s failed to move due to some internal failure.\n" %r)
    else:
        Simulate("Robot %s is not in location %d.\n" %(r, l1))
        res = FAILURE
    state.loc.ReleaseLock(r)
    return res

def moveManhattan(r, l1, l2, dist):
	''' UGV r moves from a location l1 to l2 following a Manhattan path '''
    (x1, y1) = l1
    (x2, y2) = l2
    xlow = min(x1, x2)
    xhigh = max(x1, x2)
    ylow = min(y1, y2)
    yhigh = max(y1, y2)
    # r checks whether there are any obstacles in the path
    for o in rv.OBSTACLES:
        (ox, oy) = o
        if abs(oy - y1) <= 0.0001 and ox >= xlow and ox <= xhigh:
            Simulate("%s cannot move in Manhattan path because of obstacle\n" %r)
            return FAILURE

        if abs(ox - x2) <= 0.0001 and oy >= ylow and oy <= yhigh:
            Simulate("%s cannot move in Manhattan path because of obstacle\n" %r)
            return FAILURE

    state.loc.AcquireLock(r)
    if l1 == l2:
        Simulate("Robot %s is already at location %s\n" %(r, l2))
        res = SUCCESS
    elif state.loc[r] == l1:
        start = globalTimer.GetTime()
        while(globalTimer.IsCommandExecutionOver('moveManhattan', start, r, l1, l2, dist) == False):
           pass
        res = Sense('moveManhattan')
        if res == SUCCESS:
            Simulate("Robot %s has moved from %s to %s\n" %(r, str(l1), str(l2)))
            state.loc[r] = l2
        else:
            Simulate("Robot %s failed to move due to some internal failure.\n" %r)
    else:
        Simulate("Robot %s is not in location %d.\n" %(r, l1))
        res = FAILURE
    state.loc.ReleaseLock(r)
    return res

def fly(r, l1, l2):
	''' UAV r flies from one location l1 to another location l2'''
    state.loc.AcquireLock(r)
    if l1 == l2:
        Simulate("Robot %s is already at location %s\n" %(r, l2))
        res = SUCCESS
    elif state.loc[r] == l1:
        start = globalTimer.GetTime()
        while(globalTimer.IsCommandExecutionOver('fly', start) == False):
           pass
        res = Sense('fly')
        if res == SUCCESS:
            Simulate("Robot %s has flied from %s to %s\n" %(r, str(l1), str(l2)))
            state.loc[r] = l2
        else:
            Simulate("Robot %s failed to fly due to some internal failure.\n" %r)
    else:
        Simulate("Robot %s is not in location %d.\n" %(r, l1))
        res = FAILURE
    state.loc.ReleaseLock(r)
    return res

def inspectPerson(r, p):
	''' UGV r helps one person p'''
    Simulate("Robot %s is inspecting person %s \n" %(r, p))
    state.status[p] = env.realStatus[p]
    return SUCCESS

def giveSupportToPerson(r, p):
    if state.status[p] != 'dead':
        Simulate("Robot %s has saved person %s \n" %(r, p))
        state.status[p] = 'OK'
        env.realStatus[p] = 'OK'
        res = SUCCESS
    else:
        Simulate("Person %s is already dead \n" %(p))
        res = FAILURE
    return res

def inspectLocation(r, l):
 	''' UAV r surveys a location, searching for injured people'''
    Simulate("Robot %s is inspecting location %s \n" %(r, str(l)))
    state.status[l] = env.realStatus[l]
    return SUCCESS

def clearLocation(r, l):
	''' UGV r removes debri from a location l'''
    Simulate("Robot %s has cleared location %s \n" %(r, str(l)))
    state.status[l] = 'clear'
    env.realStatus[l] = 'clear'
    return SUCCESS

def replenishSupplies(r):
	''' UGV r replenishes medical supplies at the base'''
    state.hasMedicine.AcquireLock(r)
    if state.loc[r] == (1,1):
        state.hasMedicine[r] = 5
        Simulate("Robot %s has replenished supplies at the base.\n" %r)
        res = SUCCESS
    else:
        Simulate("Robot %s is not at the base.\n" %r)
        res = FAILURE

    state.hasMedicine.ReleaseLock(r)
    return res

def transfer(r1, r2):
	''' One UGV r1 transfers medical supplies to another UGV r2'''
    state.hasMedicine.AcquireLock(r1)
    state.hasMedicine.AcquireLock(r2)
    if state.loc[r1] == state.loc[r2]:
        if state.hasMedicine[r1] > 0:
            state.hasMedicine[r2] += 1
            state.hasMedicine[r1] -= 1
            Simulate("Robot %s has transferred medicine to %s.\n" %(r1, r2))
            res = SUCCESS
        else:
            Simulate("Robot %s does not have medicines.\n" %r1)
            res = FAILURE
    else:
        Simulate("Robots %s and %s are in different locations.\n" %(r1, r2))
        res = FAILURE
    state.hasMedicine.ReleaseLock(r2)
    state.hasMedicine.ReleaseLock(r1)
    return res

def captureImage(r, camera, l):
	''' UAV r captures an image using on of its cameras at location l'''
    img = Sense('captureImage', r, camera, l)

    state.currentImage.AcquireLock(r)
    state.currentImage[r] = img
    Simulate("UAV %s has captured image in location %s using %s\n" %(r, l, camera))
    state.currentImage.ReleaseLock(r)
    return SUCCESS

def changeAltitude(r, newAltitude):
	''' UAV r changes its flying altitude to newAltitude '''
    state.altitude.AcquireLock(r)
    if state.altitude[r] != newAltitude:
        res = Sense('changeAltitude')
        if res == SUCCESS:
            state.altitude[r] = newAltitude
            Simulate("UAV %s has changed altitude to %s\n" %(r, newAltitude))
        else:
            Simulate("UAV %s was not able to change altitude to %s\n" %(r, newAltitude))
    else:
        res = SUCCESS
        Simulate("UAV %s is already in %s altitude.\n" %(r, newAltitude))
    state.altitude.ReleaseLock(r)
    return res

def SR_GETDISTANCE_Euclidean(l0, l1):
'''Calculates the euclidean distance between two 2D points, l0 and l1'''
    (x0, y0) = l0
    (x1, y1) = l1
    return math.sqrt((x1 - x0)*(x1 - x0) + (y1 - y0)*(y1-y0))

def MoveTo_Method1(r, l): 
	# A wheeled UGV robot r takes the straight path to reach l1
    x = state.loc[r]
    if x == l:
        Simulate("Robot %s is already in location %s\n." %(r, l))
    elif state.robotType[r] == 'wheeled':
        dist = SR_GETDISTANCE_Euclidean(x, l)
        Simulate("Euclidean distance = %d " %dist)
        do_command(moveEuclidean, r, x, l, dist)
    else:
        do_command(fail)

def SR_GETDISTANCE_Manhattan(l0, l1):
 ''' Calculates the Manhattan distance between two 2D points, l0 and l1. '''
    (x1, y1) = l0
    (x2, y2) = l1
    return abs(x2 - x1) + abs(y2 - y1)

def MoveTo_Method2(r, l): 
	''' UGV r  takes a Manhattan path to location l '''
    x = state.loc[r]
    if x == l:
        Simulate("Robot %s is already in location %s\n." %(r, l))
    elif state.robotType[r] == 'wheeled':
        dist = SR_GETDISTANCE_Manhattan(x, l)
        Simulate("Manhattan distance = %d " %dist)
        do_command(moveManhattan, r, x, l, dist) 
    else:
        do_command(fail)

def SR_GETDISTANCE_Curved(l0, l1):
 ''' Calculates the curved distance between two 2D points, l0 and l1. '''
    diameter = SR_GETDISTANCE_Euclidean(l0, l1)
    return math.pi * diameter / 2

def MoveTo_Method3(r, l): 
	# UGV r  takes a curved path to reach location l
    x = state.loc[r]
    if x == l:
        Simulate("Robot %s is already in location %s\n." %(r, l))
    elif state.robotType[r] == 'wheeled':
        dist = SR_GETDISTANCE_Curved(x, l)
        Simulate("Curved distance = %d " %dist)
        do_command(moveCurved, r, x, l, dist) 
    else:
        do_command(fail)

def MoveTo_Method4(r, l):
	''' UAV r moves to location l '''
    x = state.loc[r]
    if x == l:
        Simulate("Robot %s is already in location %s\n." %(r, l))
    elif state.robotType[r] == 'uav':
        do_command(fly, r, x, l)
    else:
        do_command(fail)

def Rescue_Method1(r, p):
	''' A UGV r helps a person after procuring medical supplies. '''
    if state.robotType[r] != 'uav':
        if state.hasMedicine[r] == 0:
            do_task('getSupplies', r)
        do_task('helpPerson', r, p)
    else:
        do_command(fail)

def Rescue_Method2(r, p):
	''' A UAV r delegates the rescuing task to a free ground robot. '''
    if state.robotType[r] == 'uav':
        do_task('getRobot')
    r2 = state.newRobot[1]
    if r2 != None:
        if state.hasMedicine[r2] == 0:
            do_task('getSupplies', r2)
        do_task('helpPerson', r2, p)
        state.status[r2] = 'free'
    else:
        Simulate("No robot is free to help person %s\n" %p)
        do_command(fail)

def HelpPerson_Method1(r, p):
    # Robot r helps an injured person p
    do_task('moveTo', r, state.loc[p])
    do_command(inspectPerson, r, p)
    if state.status[p] == 'injured':
        do_command(giveSupportToPerson, r, p)
    else:
        do_command(fail)

def HelpPerson_Method2(r, p):
    # Robot r helps a person p trapped inside some debri but not injured
    do_task('moveTo', r, state.loc[p])
    do_command(inspectLocation, r, state.loc[r])
    if state.status[state.loc[r]] == 'hasDebri':
        do_command(clearLocation, r, state.loc[r]) 
    else:
        CheckResult(state.loc[p])
        do_command(fail)
        
def GetSupplies_Method1(r):
    # UGV r gets medical supplies from nearby robots
    r2 = None
    nearestDist = float("inf")
    for r1 in rv.WHEELEDROBOTS:
        if state.hasMedicine[r1] > 0:
            dist = SR_GETDISTANCE_Euclidean(state.loc[r], state.loc[r1])
            if dist < nearestDist:
                nearestDist = dist
                r2 = r1
    if r2 != None:
        do_task('moveTo', r, state.loc[r2])
        do_command(transfer, r2, r)

    else:
        do_command(fail)

def GetSupplies_Method2(r):
    # UGV r gets medical supplies from the base
    do_task('moveTo', r, (1,1))
    do_command(replenishSupplies, r)

def CheckResult(l):
	''' Function to check whether a person is saved or not after performing the rescue operations.'''
    p = env.realPerson[l]
    if p != None:
        if env.realStatus[p] == 'injured' or env.realStatus[p] == 'dead' or env.realStatus[l] == 'hasDebri':
            Simulate("Person in location %s failed to be saved.\n" %str(l))
            do_command(deadEnd, p)
            do_command(fail)

def Survey_Method1(r, l):
	''' UAV r surveys location l with the help of the front camera.'''
    if state.robotType[r] != 'uav':
        do_command(fail)

    do_task('adjustAltitude', r)

    do_command(captureImage, r, 'frontCamera', l)
    
    img = state.currentImage[r]
    position = img['loc']
    person = img['person']
    
    if person != None:
        do_task('rescue', r, person)

    CheckResult(l)

def Survey_Method2(r, l):
	''' UAV r surveys location l with the help of the bottom camera.'''
    if state.robotType[r] != 'uav':
        do_command(fail)
    
    do_task('adjustAltitude', r)
    do_command(captureImage, r, 'bottomCamera', l)
    
    img = state.currentImage[r]
    position = img['loc']
    person = img['person']
    if person != None:
        do_task('rescue', r, person)

    CheckResultl(l)

def GetRobot_Method1():
	''' Finds a free robot to do some rescue task '''
    dist = float("inf")
    robot = None
    for r in rv.WHEELEDROBOTS:
        if state.status[r] == 'free':
            if SR_GETDISTANCE_Euclidean(state.loc[r], (1,1)) < dist:
                robot = r
                dist = SR_GETDISTANCE_Euclidean(state.loc[r], (1,1))
    if robot == None:
        do_command(fail)
    else:
        state.status[robot] = 'busy'
        state.newRobot[1] = robot   

def GetRobot_Method2():
	''' Assigns the first robot from the list of UGVs to do a resue task '''
    state.newRobot[1] = rv.WHEELEDROBOTS[0]
    state.status[rv.WHEELEDROBOTS[0]] = 'busy'

def AdjustAltitude_Method1(r):
	''' Changes the altitude of an UAV r from high to low.'''
    if state.altitude[r] == 'high':
        do_command(changeAltitude, r, 'low')

def AdjustAltitude_Method2(r):
	''' Changes the altitude of an UAV r from low to high.'''
    if state.altitude[r] == 'low':
        do_command(changeAltitude, r, 'high')
        
# ONE PROBLEM INSTANCE INCLUDES THE FOLLOWING
# Rigid variables  
rv.WHEELEDROBOTS = ['w1', 'w2']
rv.DRONES = ['a1', 'a2']
rv.OBSTACLES = { (100, 100)}
    
# initial values state variables 
	state.loc = {'w1': (15,15), 'w2': (29,29), 'p1': (28,30), 'p2': (10,30), 'a1': (9,19), 'a2': (4,5)}
	state.hasMedicine = {'a1': 0, 'a2': 0, 'w1': 0, 'w2': 0}
	state.robotType = {'w1': 'wheeled', 'a1': 'uav', 'a2': 'uav', 'w2': 'wheeled'}
	state.status = {'w1': 'free', 'w2': 'free', 'a1': UNK, 'a2': UNK, 'p1': UNK, 'p2': UNK, (28,30): UNK, (15, 15): UNK, (10, 30): UNK}
	state.altitude = {'a1': 'high', 'a2': 'low'}
	state.currentImage = {'a1': None, 'a2': None}
	state.newRobot = {1: None}

	# properties of the environment 
	env.realStatus = {'w1': 'OK', 'p1': 'OK', 'p2': 'injured', 'w2': 'OK', 'a1': 'OK', 'a2': 'OK', (28, 30): 'hasDebri', (15, 15): 'clear', (10, 30): 'hasDebri'}
	env.realPerson = {(28,30): 'p1', (15, 15): None, (10, 30): 'p2'}
	env.weather = {(28,30): "foggy", (15, 15): "rainy", (10, 30): "dustStorm"}

# tasks to accomplish 
tasks = {
    8: [['survey', 'a1', (15,15)], ['survey', 'a2', (28,30)]],
    20: [['survey', 'a1', (10,30)]]
}
        
\end{lstlisting}

\newpage
\section{Table of Notation}
\label{app:notation}

\begin{center}
\begin{tabular}{  l l @{} c} 
		&		& Page\\
 Notation & Meaning & defined\\
 \hline
 $\Sigma = (\Xi, \mathcal{T, M, A})$ & an acting domain & \pageref{def:Sigma}\\
 $s$, $S$ &  predicted state, set of states for the planner & \pageref{def:xi}\\ 
 $X$ & set of state variables & \pageref{def:X}\\
 $\xi$, $\Xi$ &  actual state, set of world states for the actor & \pageref{def:xi}\\
 $\tau$, $\mathcal{T}$ & task or event, set of tasks and events & \pageref{def:tau}\\
 $m$, $\mathcal{M}$ & method/method instance, set of methods for $\mathcal{T}$ & \pageref{def:m}\\
 $\overline{\mathcal{M}}$ & set of method instances of $\mathcal{M}$ & \pageref{def:m}\\ 
 $m[i]$ & the $i$th step of $m$ & \pageref{def:mi}\\
 $\sv{Applicable}(\xi, \tau)$ & set of method instances applicable to $\tau$ in state $\xi$ & \pageref{def:applicable}\\
 $a$, $\mathcal{A}$ & action, set of actions & \pageref{def:a}\\ 
 $\gamma(\xi,a)$ & possible states after performing $a$ in $\xi$ & \pageref{def:gamma}\\
 $\sigma$ & refinement stack with tuples of the form $(\tau, m, i, tried)$ & \pageref{sec:rae}\\
 $v_e$, $v_s$ & value functions for efficiency and success ratio  & \pageref{eq:ve}\\ 
 $v_{e1} \oplus v_{e2}$ & cumulative efficiency value of two successive actions & \pageref{eq:eplus}\\
 $v_{s1} \oplus v_{s2}$ & cumulative success ratio of two successive actions & \pageref{eq:splus}\\
 $\one$ & the identify element for $\oplus$, i.e. $x \oplus \one = x$ & \pageref{def:one}\\
 $U(m, s, \sigma)$ & the utility of  $m$ for $\tau$ and $\sigma$ & \pageref{eq:um}\\
 $U^*(m, s, \sigma)$ & the maximal expected utility of $m$ for $\tau$ & \pageref{eq:um*}\\
 $\Us$, $\Uf$ & the utility of a success, the utility of a failure & \pageref{def:usuf}\\
 $m^*_{\tau,s}$ & the optimal method instance for $\tau$ in $s$ for utility $U^*$ & \pageref{eq:m2*}\\
 $d$, $d_{max}$, $n_{ro}$ & depth, max depth, number of rollouts & \pageref{def:depth}\\
 $h(\tau, m, s)$ & heuristic estimate to solve $\tau$ with $m$ in $s$ & \pageref{alg:oracle}\\
 $h_0$, $h_D$ & always returns $\infty$, hand written heuristic & \pageref{def:heur}\\
 $h_\lh$ & learned heuristic & \pageref{def:heur}\\
 $Q_{\sigma, s}(m)$ & approximation of $U^*(m, s, \sigma)$ & \pageref{def:Q}\\
 $C$ &  tradeoff between exploration and exploitation & \pageref{def:C}\\
 $\mu$, $K$ & suggested control parameters for $n_{ro}$ & \pageref{def:mu}\\
 $g(s,\tau, m)$ & default state after accomplishing $\tau$ with $m$ in $s$ & \pageref{def:g}\\
 $r$ & a data record of the form $((s, \tau),m)$ & \pageref{def:r}\\
 $w_s$, $w_\tau$, $w_m$, $w_u$ & One-Hot representations of $s$, $\tau$, $m$, and $interval(u)$ & \pageref{eq:encoding}\\
 $v_{un}$ & uninstantiated method parameter & \pageref{def:vun}\\ 
 $v_\tau$ & list of values of task parameters & \pageref{def:vun}\\
 $b$ & value of the parameter $v_{ui}$ & \pageref{def:vun}\\ 
 $V$ & number of state variables & \pageref{def:V}\\
 $nn_\pi$ & MLP for \lm & \pageref{def:nnpi}\\ 
 $nn_{v_{un}}$ & MLP for each $v_{un}$ & \pageref{def:nnvun}\\
 $Z$ & Number of training records & \pageref{def:Z}\\
 $\Phi = (S, \mathcal{T, M, A})$ & a refinement planning domain & \pageref{sec:plan}\\
 $P(s,a,s')$ & probability that $a$'s execution in $s$ returns $s'$ & \pageref{def:P}\\
 $\Pi = (\Phi,s_0,\sigma_0,U)$ & a refinement planning problem & \pageref{def:Pi}\\
 $\rho$ & a rollout in $\Phi$ & \pageref{def:rollout}\\
 $\mathcal{R}(\Phi)$ & the set of all refinement stacks that are reachable in $\Phi$ & \pageref{def:reach}\\
 $\reward(s)$, $\cost(s,a,s')$ & reward function, cost function & \pageref{def:reward}\\
 $\Psi$ & an MDP used in the appendix& \pageref{sec:mdp}\\

 \hline 
\end{tabular}
\end{center}

\end{document}